\documentclass{article}

\newif\ifarxiv
\arxivtrue
\PassOptionsToPackage{numbers,sort&compress}{natbib} %
\ifarxiv
  \usepackage[preprint]{neurips_2026}
\else
  \usepackage[main, final]{neurips_2026}
\fi
\makeatletter
\ifarxiv
  \renewcommand{\@noticestring}{Extended version of a paper accepted at NeurIPS 2026.}
\fi
\makeatother

\usepackage[utf8]{inputenc} %
\usepackage[T1]{fontenc}    %
\usepackage{hyperref}       %
\usepackage{url}            %
\usepackage{booktabs}       %
\usepackage{amsfonts}       %
\usepackage{nicefrac}       %
\usepackage{microtype}      %
\usepackage{xcolor}         %

\usepackage{balance}

\usepackage{xcolor}
\definecolor{myblue}{HTML}{4278f5}
\definecolor{colorcite}{HTML}{3b4670}
\usepackage{hyperref}
\usepackage{etoc} %
\hypersetup{
	colorlinks=true,
	linkcolor=myblue,
	citecolor=colorcite,
	urlcolor=colorcite
}
\usepackage{url}
\usepackage{mathtools}
\usepackage{graphicx}
\usepackage{paralist}
\usepackage{xspace}
\usepackage{caption}
\usepackage{setspace}
\usepackage{booktabs}
\usepackage{tabularx}

\usepackage{tikz}
\usetikzlibrary{arrows,decorations.pathmorphing,backgrounds,positioning,fit,petri,calc,positioning,shapes.geometric}
\usetikzlibrary{shapes.geometric}
\usetikzlibrary{shapes}
\usepackage{comment}
\usetikzlibrary{positioning,chains,fit,shapes,calc}
\usepackage{tkz-berge}
\definecolor{myblue}{RGB}{80,80,160}
\definecolor{mygreen}{RGB}{80,160,80}
\definecolor{mygrey}{HTML}{706968}
\definecolor{mydarkgrey}{HTML}{43110a}
\definecolor{mydarkred}{HTML}{e12800}
\ifarxiv
  
\else
  
\fi

\usepackage{amsmath}
\allowdisplaybreaks
\usepackage{amssymb}
\usepackage{mathtools}
\usepackage{amsthm}
\usepackage{thmtools}
\usepackage{thm-restate}
\usepackage{stmaryrd}
\usepackage{tabularray}
\usepackage{multirow}
\usepackage{subcaption}
\usepackage{microtype}

\usepackage{colortbl}
\usepackage{tabulary}
\usepackage{etoolbox}

\usepackage{titlesec}
\titlespacing{\section}{0pt}{*1}{*0.5}

\usepackage{dsfont}
\usepackage[ruled, noline, noend, linesnumbered, algo2e]{algorithm2e}
\PrintSemicolon

\usepackage[capitalize,nameinlink]{cleveref}

\newcommand{\para}[1]{\noindent\textbf{#1}}

\crefname{section}{Sec.}{Secs.}
\Crefname{section}{Sec.}{Secs.}
\Crefname{definition}{Def.}{Defs.}
\Crefname{appendix}{App.}{Apps.}
\crefname{equation}{eqs}{eqs}   %
\Crefname{equation}{Eq.}{Eqs}   %
\crefname{problem}{problem}{problems}
\Crefname{problem}{Problem}{Problems}

\newcommand{\charBin}{\ensuremath{v_{\mathsf{B}}}\xspace}
\newcommand{\charBinS}[1]{\ensuremath{\charBin(#1)}\xspace}

\newcommand{\modelSet}{\ensuremath{\mathcal{C}}\xspace}
\newcommand{\model}{\ensuremath{C}\xspace}
\newcommand{\coal}{\ensuremath{S}\xspace}
\newcommand{\labelSet}{\ensuremath{\mathcal{Y}}\xspace}

\newcommand{\tieSet}{\ensuremath{\mathcal{M}}\xspace}
\newcommand{\cardM}{\ensuremath{n}\xspace}
\newcommand{\domainX}{\ensuremath{\mathcal{X}}\xspace}
\newcommand{\totLabs}{\ensuremath{\ell}\xspace}
\newcommand{\xp}{\ensuremath{x}\xspace}
\newcommand{\predModel}{\ensuremath{\model(\xp)}\xspace}
\newcommand{\game}{\ensuremath{G}\xspace}
\newcommand{\pSet}{\ensuremath{N}\xspace}
\newcommand{\chF}{\ensuremath{v}\xspace}
\newcommand{\pCard}{\ensuremath{n}\xspace}
\newcommand{\shap}{\ensuremath{\phi}\xspace}
\newcommand{\banz}{\ensuremath{\beta}\xspace}

\newcommand{\VotVec}{\ensuremath{\bm{p}}\xspace}
\newcommand{\VotVecCoord}{\ensuremath{p}\xspace}
\newcommand{\ind}{\ensuremath{\mathds{1}}\xspace}

\newcommand{\y}{\ensuremath{y}\xspace}

\newcommand{\dataSet}{\ensuremath{\mathcal{D}}\xspace}
\newcommand{\dataSetVal}{\ensuremath{\mathcal{D}_{\text{val}}}\xspace}
\newcommand{\dataSetTra}{\ensuremath{\mathcal{D}_{\text{test}}}\xspace}

\newcommand{\sumGame}{\ensuremath{{G}_\dataSet}\xspace}
\newcommand{\sumGameMC}{\ensuremath{{G}_\dataSet}\xspace}

\newcommand{\charMC}{\ensuremath{v_{\mathsf{MC}}}\xspace}
\newcommand{\gameMC}{\ensuremath{G}\xspace}

\newcommand{\unique}{\ensuremath{U}\xspace}

\newcommand{\utilSuiteRaw}{enpower}
\newcommand{\suite}{{\textsc{\utilSuiteRaw}}\xspace}
\newcommand{\suiteBin}{\textsc{\utilSuiteRaw}\ensuremath{_{\mathsf{B}}}\xspace}
\newcommand{\suiteMC}{\textsc{\utilSuiteRaw}\ensuremath{_{\mathsf{MC}}}\xspace}
\newcommand{\suiteShapImp}{\textsc{\utilSuiteRaw}\ensuremath{_{\mathsf{MC}}\text{-}\shap}++\xspace}
\newcommand{\codeURL}{\url{https://github.com/iliesarpe/enpower}\xspace}

\newcommand{\computeSHAP}{\ensuremath{\mathtt{Get}\text{-}\shap}\xspace}
\newcommand{\computeBANZ}{\ensuremath{\mathtt{Get}\text{-}\banz}\xspace}

\newcommand{\getpsi}{\ensuremath{\text{\texttt{Get\_}}\powerIdx}\xspace}

\newcommand{\ensRF}{\texttt{RF}\xspace}
\newcommand{\ensAB}{\texttt{AB}\xspace}
\newcommand{\ensHET}{\texttt{HET}\xspace}
\DeclareMathOperator{\Bin}{Bin}

\usepackage{algorithm}
\usepackage{xspace}

\usepackage{amssymb,amsmath,amsfonts,amsthm,bm}

\def\eqref#1{equation~\ref{#1}}

\def\floor#1{\lfloor #1 \rfloor}
\def\1{\bm{1}}

\def\evn{{n}}

\def\evz{\ensuremath{z}\xspace}
\def\gfvar{{t}}

\DeclareMathAlphabet{\mathsfit}{\encodingdefault}{\sfdefault}{m}{sl}
\SetMathAlphabet{\mathsfit}{bold}{\encodingdefault}{\sfdefault}{bx}{n}

\newcommand{\E}{\mathbb{E}}

\newcommand{\R}{\mathbb{R}}

\newcommand{\corrPred}{\ensuremath{\evn_{+}}\xspace}
\newcommand{\wrongPred}{\ensuremath{\evn_{-}}\xspace}
\newcommand{\bigO}{\ensuremath{\mathcal{O}}\xspace}

\newcommand{\coalset}{\ensuremath{\mathcal{S}}\xspace}
\newcommand{\powerIdx}{\ensuremath{\psi}\xspace}

\newcommand{\CorW}{\ensuremath{\sigma}\xspace}
\newcommand{\WorC}{\ensuremath{\tau}\xspace}
\newcommand{\shapTermCorr}{\ensuremath{T_{+}}\xspace}
\newcommand{\shapTermWrong}{\ensuremath{T_{-}}\xspace}
\newcommand{\shapTerm}{\ensuremath{T}\xspace}

\theoremstyle{plain}
\newtheorem{theorem}{Theorem}[section]

\newtheorem{lemma}[theorem]{Lemma}
\newtheorem{corollary}[theorem]{Corollary}
\theoremstyle{definition}
\newtheorem{definition}[theorem]{Definition}

\newtheorem{property}[theorem]{Property}

\theoremstyle{remark}

\title{Exact power indices for plurality-voting ensembles}

\author{%
  Ilie Sarpe$^{*1\dagger}$ \qquad
  Theofanis Georgakopoulos$^{*2}$\qquad
  Aristides Gionis$^{1,3}$ \\[2pt]
  \normalfont
  $^{1}$\textit{KTH Royal Institute of Technology}, Stockholm, Sweden \\
  $^{2}$\textit{Ericsson BCSS}, Stockholm, Sweden \\
   $^{3}$\textit{Digital Futures}, Stockholm, Sweden \\
}

\begin{document}
\etocdepthtag.toc{mtmain}

\maketitle
{\renewcommand{\thefootnote}{\fnsymbol{footnote}}\footnotetext[1]{Equal contribution. $^\dagger$Contact author.}}

\begin{abstract}
Plurality-voting ensembles aggregate the 
predictions of multiple base classifiers and 
output the most voted class, 
with ties typically broken at random. 
Despite the simple aggregation rule, 
attributing the impact of each base classifier
to the ensemble's predictions is far from straightforward.
That is, a single base classifier's impact depends on  how it interacts with the votes of all other classifiers in the ensemble, for example, in resolving ties, and not only on its own predictions.
We address the attribution problem using established tools from cooperative game theory, 
studying the Banzhaf and Shapley power indices for new games modeling plurality voting and 
random tie-breaking, for binary and multi\-class classification settings.
Computing power indices is, in general, intractable. 
Surprisingly, our methods, collectively named~\suite, 
enable the \emph{exact} computation of Banzhaf and Shapley  values in 
\emph{linear time} for the binary classification setting and in
\emph{low-degree polynomial time} for the multiclass setting. 
The algorithms in~\suite rely on new combinatorial formulations of the Banzhaf and Shapley  indices, yielding closed-form solutions for the binary case and ordinary generating function representations for the multiclass case.
We validate empirically the methods in \suite demonstrating their high applicability, providing attribution values not captured by existing methods.
We apply the attribution of \suite to: 
ensemble pruning and prompt-based ensembling of vision-language models. %
\end{abstract}

\section{Introduction}

Ensemble methods, which combine the predictions of multiple \emph{base classifiers} 
through a voting mechanism, are widely used in modern machine learning, as they
achieve higher accuracy and robustness 
than individual models \citep{Sagi2018EnsSurv,Zhou2025EnsembleMethods}.
Among the variety of aggregation rules \citep{Mohammed2023RevEnsDL}, 
\emph{plurality voting}---where the prediction corresponds to the most voted class, 
with ties broken uniformly at random---is a well-established scheme
that applies directly to binary and multi\-class classification and is
used by many ensembles and popular software packages~\citep{Zhou2025EnsembleMethods, RpackageRF}. 

\emph{Attribution} is a central question for an ensemble, that is, 
quantifying how much each base classifier contributes to the ensemble's predictions.
Simple heuristics, such as leave-one-out accuracy or per-model performance,
ignore the interactions among base models:
such heuristics measure the contribution of a base classifier to a single sub-ensemble, i.e., the full or empty one, respectively, among the $\Theta(2^{\cardM})$ sub-ensembles (where \cardM corresponds to the number of base classifiers).
For example, a single vote of a base classifier can modify the prediction of any sub-ensemble with a tie for the most voted class.
Cooperative game theory addresses precisely this problem through \emph{power indices}, 
in particular the 
\emph{Banzhaf value} \citep{banzhaf1964weighted} and the 
\emph{Shapley value} \citep{shapley1953value}, uniquely characterized by natural axioms,
faithfully capture the contribution of a player to the outcome of the game.

Computing power indices is intractable in general, 
involving a summation over an exponential number of coalitions, 
which we also refer to as sub-ensembles.
Prior work has therefore resorted either to approximations
\citep{bachrach2010approximating,Rozemberczki2021ShapEnsemble} or to restricted game classes, 
most prominently weighted majority games (WMGs)
\citep{Uno2012WMG,Ykhlef2017BanzWeighted}.
However, these approaches do not capture plurality voting with random tie-breaking and 
cannot easily be applied to multi\-class settings.

We first introduce a new game capturing plurality-voting ensembles with random tie-breaking.
We then design \suite, a collection of algorithms computing the \emph{exact} 
Banzhaf and Shapley {power indices} of every base classifier 
in our game, 
for binary and multi\-class classification.
We consider base classifiers casting hard votes.
Since our analysis considers only the votes of the base classifiers, our results are agnostic to the underlying models and apply to heterogeneous ensembles.
The power indices we obtain can be used for many important applications such, ensemble pruning \citep{Kulkarni2012Pruning}, data validation, and feature selection \citep{ghorbani2019data}.
We first address the \emph{binary} case, 
which differently from a WMG is not a simple game (i.e., a monotone game with values in $\{0,1\}$) due to random tie-breaking.
We obtain closed-forms for both power indices, computable 
in time $\bigO(|\dataSet|\cardM)$, 
where
$|\dataSet|$ is the number of test points.
We then provide a new algorithmic contribution for the \emph{multi\-class} case with $\totLabs$ classes, where existing games do not capture faithfully ensemble aggregation. 
We show how to compute exactly the Banzhaf values in time $\bigO(|\dataSet| \cardM \totLabs^2)$
and the Shapley values in time $\bigO(|\dataSet| \cardM^3 \totLabs)$, 
that is, polynomial in $\cardM$ and $\totLabs$.
We consider unweighted votes; extending \suite to weighted plurality voting 
is non-trivial. 

Our technical results %
exploit \emph{structural symmetries}: 
base classifiers with identical votes yield the same contributions to the ensemble.
In the binary case, this yields closed-form solutions for both power indices.
In the multi\-class case, 
we encode coalition counts via \emph{ordinary generating functions}~\citep{Flajolet2013BookComb} with low degree, computed efficiently with dynamic programming. 
Our techniques build on previous methods using ordinary generating functions to compute power indices in WMGs \citep{Bilbao2000GenPI}. 
\ifarxiv
We also discuss a variant of our dynamic programming approach for the Shapley values using fast polynomial multiplication, with a faster asymptotic time bound.
\fi

In our experimental evaluation we use the Banzhaf and Shapley values computed with \suite on two applications: \emph{ensemble pruning}~\citep{Lu2010PruningScoreDiv,khan2021optimal}, and \emph{prompt-based ensembling of vision-language models}~\citep{radford2021learningCLIP}, on a zero-shot image classification task.
Our empirical study shows that \suite is highly efficient in practice computing \emph{exact} power indices in a few seconds on most instances, and produces attribution values that are not captured by existing methods, directly characterizing plurality-voting ensembles with random tie-breaking.

\section{Related work}\label{sec:related}

Ensembles combine predictions of \emph{base models} 
using a \emph{voting mechanism}~\citep{Sagi2018EnsSurv}. %
Prior work has considered
learning algorithms to obtain homogeneous or heterogeneous base models, 
e.g., through boosting, bagging, or stacking~\citep{Mohammed2023RevEnsDL}, 
and different aggregation rules to combine predictions of base models. 
Widely-adopted rules include averaging, majority voting, and \emph{plurality voting}~\citep{Zhou2025EnsembleMethods}, which is the focus of our work. 

Attribution is a key challenge in explainable AI~\citep{Minh2021XAI}, 
and power indices have been used extensively to assign values to individual models or features~\citep{lundberg2017unified,wang2023data, strumbelj2010efficient}. 
To the best of our knowledge, our work is the first to consider the attribution problem in plurality-voting ensembles with random tie-breaking, for both binary and multi\-class settings.

Earlier work has mostly focused on computing power indices for ensembles under \emph{binary} classification, 
building on results from classical cooperative game theory~\citep{Chalkiadakis2012ComputationalBook,Deng1994WMG} and weighted majority games (WMGs)~\citep{Georgiou2013Weight,Dogan2019Weight,Georgiou2006SVMWMG}.
Under unitary weights, power indices for WMGs can be computed in polynomial time~\citep{Uno2012WMG}, 
while recent contributions use techniques similar to ours~\citep{kakimura2026computing}.
However, as we discuss in~\Cref{sec:RWExtended},  WMGs cannot be used effectively in our setting.

Close to our work, 
\citet{Rozemberczki2021ShapEnsemble} study a game for ensembles of soft binary classifiers aggregated through a \emph{majority} rule. 
They design an approximation algorithm using multilinear extensions 
to estimate Shapley values.
\citet{Ykhlef2017GraphShapEnsamble} use a classical result by~\citet{Deng1994WMG} 
to model ensemble diversity through Shapley values, 
capturing, however, only pairwise relationships among base models and not matching the 
ensemble's aggregation rule.
\citet{Ykhlef2017BanzWeighted} apply their diversity scoring to Banzhaf values computed over WMGs.
In \Cref{sec:exps} we show that our exact methods are much faster than prior work \citep{Uno2012WMG, Ykhlef2017GraphShapEnsamble}, and are comparable with the fastest approximate method by \citet{Rozemberczki2021ShapEnsemble}.
Moreover, our algorithms yield new values not captured by existing methods.
Our methods are specifically designed for multi\-class settings, 
where prior work cannot be %
adapted; see~\Cref{sec:RWExtended} for details on additional related work.

\section{Preliminaries}
\label{sec:prelim}

Let $\modelSet = \{\model_1, \dots, \model_\cardM\}$ be an \emph{ensemble} of $\cardM$ classifiers, 
where each classifier $\model\in\modelSet$ corresponds to a function $\model: \domainX \mapsto \labelSet$ 
from a domain $\domainX$ to a target label set $\labelSet =\{\y_1,\dots,\y_\totLabs\}$. 
We use $\predModel \in \labelSet$ to denote the class assigned by $\model$ to data point $\xp\in \domainX$. 
We do not assume that classifiers come from the same model family, 
neither we require access to training data.

Our goal is to characterize the contribution of each classifier $\model\in\modelSet$. 
Inspired by \emph{attribution methods} in machine learning 
\citep{ghorbani2019data,lundberg2017unified,vstrumbelj2014explaining},  
we employ ideas from \emph{cooperative game theory}.

\begin{definition}[\citep{Chalkiadakis2012ComputationalBook}]
\label{def:charGame}
A \emph{characteristic function game} (or simply \emph{game}) $\game$ is given by a pair $(\pSet, \chF)$ where $\pSet = [\pCard]$ is a finite non-empty set of players, and $\chF: 2^\pSet \to \R$ is a characteristic function mapping each \emph{coalition} $\coal\subseteq\pSet$ to a \emph{value} $\chF(\coal)$.
\end{definition}
\emph{Power indices} are used to capture the contribution of players in games
by considering changes in the value of coalitions that the players participate.
Two of the most popular indices, widely adopted in the literature and considered in this paper, 
are the \emph{Banzhaf} and \emph{Shapley} values, defined below.

\begin{definition}\label{def:shapAndBanz}
	Given a game $\game= (\pSet, \chF)$ the \emph{Banzhaf} value $\banz_i(\game)$ 
  (or simply $\banz_i$ when the game $\game$ is clear) 
  of player $i\in \pSet$ and the \emph{Shapley} value $\shap_i(\game)$ 
  (or simply $\shap_i$) are defined respectively as,%
  \footnote{%
  For a set $\coal \subseteq \pSet$ and a player $i\in \pSet$, 
  we define $\coal+i \doteq \coal \cup \{i\}$ and $\coal-i \doteq \coal \setminus \{i\}$.}
      \begin{equation}
    \banz_i \doteq \frac{1}{2^{\pCard -1}}\!\!\sum_{i\in \coal \subseteq \pSet}\!\! \left( \chF(\coal) - \chF(\coal - i) \right),
    ~~~
    \shap_i\doteq \!\!\sum_{i\in \coal \subseteq \pSet}\!\! \frac{(\pCard- |\coal|)!(|\coal|-1)!}{\pCard!}\left(\chF(\coal) - \chF(\coal - i)\right).\label{eq:shap}
  \end{equation}
\end{definition}
The Shapley values are the \emph{unique} values that satisfy the following desirable properties: 
\emph{dummy player},  \emph{symmetry},  \emph{linearity}, and  \emph{efficiency}
(see~\cref{appendix:subsection:CGT}).
The Banzhaf values satisfy all the above properties \emph{except} efficiency.
While Shapley values are the natural choice when the values must sum to $\chF(\pSet)$, 
the Banzhaf values are more robust to noise and easier to estimate, 
making them preferable in stochastic or large-scale regimes~\citep{wang2023data}.

Next, we define the characteristic function $\chF$
for our setting.
Given a data point $\xp \in \domainX$ and a coalition $\coal \subseteq \modelSet$, 
we let $\VotVec(\xp, \coal) \in \mathbb{Z}^\totLabs$ be a \emph{vote profile}, 
where the $j$-th component corresponds to the number of classifiers from $\coal$ \emph{voting} for $\y_j$. 
We simplify $\VotVec(\xp, \coal)$ to $\VotVec(\coal)$ or $\VotVec$ when $\xp, \coal$ are clear from the context. 
We define the \emph{majority} set of labels 
$\tieSet(\xp, \coal) = \{\y_j \in \labelSet : \VotVecCoord_j \ge \VotVecCoord_i, i\in[\totLabs]\}$, 
containing the labels receiving the most votes from the classifiers in the coalition 
$\coal \subseteq \modelSet$. 

We now introduce characteristic function games that model how plurality voting ensembles with random tie-breaking are used in practice for 
\begin{inparaenum}
	\item[(1)] \emph{binary classification}; and
	\item[(2)] \emph{multi\-class classification}.
\end{inparaenum}

\textbf{Binary classification.} 
For a data point $\evz=(\xp, \y) \in \domainX\times \labelSet$, 
with $\labelSet= \{0,1\}$, 
and a coalition $\coal\subseteq \modelSet, \coal\neq \emptyset$,
we consider the \emph{point-wise} characteristic function game 
$\game_{\evz}=(\modelSet, \charBin)$ with
\begin{equation}\label{eq:charBin}
	\charBin(\evz,\coal) \doteq \begin{cases}
		1\,, & \text{if }\{\y\}= \tieSet(\coal)\,, \\
		\frac{1}{2}\,, & \text{if } \y \in \tieSet(\coal) \text{ and } |\tieSet(\coal)|=2\,,\\
		0\,, & \text{otherwise}\, .
	\end{cases}
\end{equation}
We let $\charBin(\evz,\emptyset)\doteq0$ for technical purposes.
In other words, for a data point  $\evz = (\xp, \y)$, 
the function $\charBin$ assigns a score to each coalition $\coal\subseteq\modelSet$, 
which is:
\begin{inparaenum}[(1)]
  \item[] 1 if the majority of models in $\coal$ correctly identify label~\y;
  \item[] 1/2 if there is a tie; and
  \item[] 0 if the majority of models in $\coal$ make an incorrect prediction.
\end{inparaenum}

The game $\game_{\evz}$ is defined for a single data point \evz. %
For a test set \dataSet, 
we define the game \sumGame
by averaging the point-wise values $\charBin(\evz,\coal), \evz\in\dataSet$, i.e., 
$\charBin(\dataSet,\coal) = \frac{1}{|\dataSet|} \sum_{(\xp, \y)\in \dataSet} \charBin(\evz,\coal)$. 
The Banzhaf and Shapley values satify linearity and an
efficient computation of the values for $\charBin(\evz,\coal)$ implies 
efficient computation for the values for $\charBin(\dataSet,\coal)$.
Thus, we focus on computing the Banzhaf and Shapley values
of $\game_\evz$ for a single point $\evz$.
We drop the first argument of \charBin, 
when clear from the context.
It can be shown that \charBin does not feature standard properties enabling efficient 
computation. 
\begin{lemma}\label{lemma:charNonProp}
	The characteristic function $\charBin(\evz,\cdot): 2^\modelSet \to \{0,\frac{1}{2},1\}$
  is neither super\-modular, nor~super\-additive, nor mono\-tone.
\end{lemma}

All proofs can be found in~\Cref{app:proofs}. Since $\charBin(\evz,\cdot)$ is not monotone and it can be $\frac{1}{2}$, our game is not a simple game (i.e., a monotone game with values in $\{0,1\}$), in particular not a WMG (see \Cref{def:WMG}). 
Therefore, we cannot use existing standard algorithms to compute the power indices for our formulation. 
From the previous lemma,
one may think that \emph{exact} computation of power indices for our new game
is intractable. 
Surprisingly, we can show that both the Banzhaf and Shapley values 
can be computed efficiently, see~\Cref{subsec:binClass}. 

\begin{theorem}
\label{theo:main01}
Let $\modelSet$ be an ensemble of classifiers, 
$\dataSet$ a test dataset, 
and $\sumGame=(\modelSet,\charBin(\dataSet,\cdot))$ the characteristic function game over \dataSet.
For all classifiers $\model_i\in\modelSet$
the Banzhaf values $\banz_i(\sumGame)$ and 
the Shapley values $\shap_i(\sumGame)$
can be computed in linear-time complexity $\bigO(|\dataSet|\cardM)$.
\end{theorem}

The term $|\dataSet|$ in the running time comes from averaging 
$\banz_i(\game_{\evz})$ and $\shap_i(\game_{\evz})$ over all $\evz\in\dataSet$
to compute 
$\banz_i(\sumGame)$ and $\shap_i(\sumGame)$.
The linear term $\bigO(\cardM)$ is the time needed to compute~$\banz_i(\game_{\evz})$ 
and~$\shap_i(\game_{\evz})$. 

\textbf{Multiclass classification.} 
We consider a set of labels $\labelSet=\{\y_1,\dots,\y_\totLabs\}$ with $\totLabs>2$.
For a single data point 
$\evz = (\xp, \y)$, 
we define the game $\gameMC_\evz=(\modelSet, \charMC)$
where $\charMC(\evz,\emptyset)\doteq 0$ and for $\coal\subseteq \modelSet, \coal\neq\emptyset$ it is
\begin{equation}\label{eq:charMC}
	\charMC(\evz,\coal) \doteq 
	\begin{cases}
		\frac{1}{|\tieSet(\coal)|}\,, & \text{if } \y \in\tieSet(\coal) \,, \text{ and}\\
		0\,, & \text{otherwise} \, .
	\end{cases}
\end{equation}
We then define the characteristic function for dataset \dataSet by 
$\charMC(\dataSet,\coal) = \frac{1}{|\dataSet|} \sum_{(\xp, \y)\in \dataSet} \charMC(\evz,\coal)$. 
Clearly, the characteristic function \charMC generalizes \charBin from~\Cref{eq:charBin}. 

To state our result
in the multi\-class case 
we need one additional definition.
First, without loss of generality (simply by permuting the class labels), 
we will assume that for any data point $(\xp,\y)$ it holds $\y = \y_1 \in \labelSet$. 
We let $\cardM_j$
be the number of classifiers predicting class $\y_j, j\in [\totLabs]$ 
for data point~$(\xp, \y)$. 
Finally, let $\unique = \{\VotVec(\xp, \modelSet) = (\cardM_1,\cardM_2,\dots,\cardM_\totLabs): (\xp,\y)\in\dataSet \text{ and } \cardM_2\le \dots\le\cardM_\totLabs \}$ be the \emph{set} of \emph{unique} vote profiles, where $\cardM_1$ is the number of correct classifiers.
Clearly, $|\unique|\le|\dataSet|$; moreover, $|\unique|\le\cardM+1$ for binary classification,
since a profile is determined by $\cardM_1$.
In practice, $|\unique|$ can be much smaller than $|\dataSet|$ also in the multi\-class case:
when most classifiers are accurate, many points share their vote profile
(see $|\unique|$ in \Cref{tab:clip_datasets}).
We now state the generalization of \Cref{theo:main01} in the multi\-class case.

\begin{theorem}
\label{theo:mainMC}
Let $\modelSet$ be an ensemble of classifiers, 
$\dataSet$ be a test dataset, 
and $\sumGameMC=(\modelSet,\charMC(\dataSet,\cdot))$ be the characteristic function game over \dataSet.
For all classifiers $i\in [\cardM]$
the Banzhaf values $\banz_i(\sumGame)$ can be computed exactly in time 
$\bigO(|\unique| \cardM \totLabs^2+|\dataSet|\cardM)$, 
and the Shapley values $\shap_i(\sumGame)$
can be computed exactly in time $\bigO(|\unique| \cardM^3 \totLabs + |\dataSet|\cardM)$.
\end{theorem}

\section{Methods}
\label{sec:algs}

We present our polynomial-time algorithms for exact computation of Banzhaf and Shapley values. All missing proofs are reported in~\Cref{app:proofs}.

\subsection{Binary classification}
\label{subsec:binClass}

\para{Intuition.} 
As explained in the previous section, 
we focus on computing the Banzhaf and Shapley values  
$\banz_i(\game_\evz)$ and $\shap_i(\game_\evz)$
for a single data point $\evz$, and then average over all data in \dataSet
to compute $\banz_i(\game_\dataSet)$ and $\shap_i(\game_\dataSet)$.
For each classifier $i\in [\cardM]$ the power indices $\banz_i$ and $\shap_i$ 
are given by a weighed sum of the marginal gain $\charBin(\coal)-\charBin(\coal-i)$, 
over all coalitions $\coal\subseteq[n]$ with $i\in\coal$. 
We observe that for a \emph{correct} classifier, 
if $\coal-i\neq \emptyset$ then a non-zero marginal gain can only attain value $\frac{1}{2}$; 
and a symmetrical argument applies to an incorrect classifier.
Furthermore, a marginal gain of $\frac{1}{2}$ is attained only in two cases, 
namely, 
\begin{inparaenum}[($a$)]
	\item if the classifier $i$ introduces a tie,
              in which case $\charBin(\coal-i)=0$ and $\charBin(\coal)=\frac{1}{2}$; or
	\item if $i$ breaks a tie, 
              in which case $\charBin(\coal-i)=\frac{1}{2}$ and $\charBin(\coal)=1$.
\end{inparaenum}
This key observation is formalized in~\Cref{lemma:marginalBinaryUnweighted}.
Interestingly, given the special structure of the sets $\coal-i$, 
we can compute efficiently the number of such coalitions through combinatorial arguments. 
In addition, by combining with symmetry arguments
we obtain our efficient algorithms from~\Cref{theo:main01}.

We now show the closed-form solutions of the power indices.
\begin{theorem}[Banzhaf]
\label{theo:Banz01} 
Let $\modelSet$ be an ensemble, and
$\evz=(\xp,\y) \in \dataSet$ be a test data point. 
Let $\corrPred = |\{\model \in \modelSet : \model(\xp) = \y\}|$ 
be the number of classifiers in \modelSet that make a correct prediction for \evz, 
and $\wrongPred=\cardM-\corrPred$ 
the number of classifiers that make an incorrect prediction.
For each $i\in [\cardM]$ it is
	\begin{equation}
  \label{eq:closed01Banz}
	\banz_i(\game_{\evz}) =
		\frac{1}{2^\cardM}\left(\binom{\cardM}{\wrongPred}+1\right) ~\text{ if } \model_i(\xp) = \y\,,
		\quad
		\banz_i(\game_{\evz}) = -\frac{1}{2^\cardM}\left(\binom{\cardM}{\corrPred}-1\right) ~\text{ if } \model_i(\xp) \neq \y\,.
  \end{equation}
\end{theorem}

\begin{theorem}[Shapley]
\label{theo:Shap01} 
	Under the same notation as in~\Cref{theo:Banz01}, 
  for each $i\in [\cardM]$ it is %
	\begin{equation}
  \label{eq:closed01Shap}
		\shap_i(\game_{\evz})=
			\frac{1}{\cardM} + \frac{\shapTermCorr}{2} ~~~\text{ if } \model_i(\xp) = \y\,,
			\quad
			\shap_i(\game_{\evz})= -\frac{\shapTermWrong}{2} ~~~\text{ if } \model_i(\xp) \neq \y\,,
	\end{equation}
	where, if we let $\CorW\in \{+, - \}$ and $\WorC \in \{+,-\}$ with $\CorW\neq \WorC$ then 
	\begin{equation*}
  \begin{small}
  \hspace{-6mm}
  \shapTerm_{\CorW} =
		\frac{(\cardM-2\cardM_{\CorW})!(2\cardM_{\CorW}-1)!}{\cardM!}{\cardM_{\WorC} \choose \cardM_{\CorW}}
		+
    \!\!\sum_{h=1}^{\min\{\cardM_{\WorC},\cardM_{\CorW}-1\}}\!\! {\cardM_{\CorW}-1 \choose h-1} {\cardM_{\WorC} \choose h} \frac{(\cardM-2h-1)!(2h-1)!}{\cardM!} (2\cardM_{\CorW} + \cardM - 4h),
  \end{small}
	\end{equation*}
	where ${m \choose k}= 0$, for any $k>m$ or $k< 0$.
\end{theorem}

The proofs of Theorems~\ref{theo:Banz01} and~\ref{theo:Shap01} 
are given in~\Cref{appendix:proofs-binary-case}. 
The resulting approach is shown in~\Cref{alg:mainAlg01}. 
It receives as input the power index $\powerIdx\in\{\shap,\banz\}$ to be computed, 
the trained ensemble~$\modelSet$, and the test dataset $\dataSet$. 
The algorithm initializes $\powerIdx_i=0$, for all $i\in[\cardM]$, 
and then iterates over all data points in \dataSet. 
For each data point $(\xp, \y)$, 
it computes $\corrPred$ the number of classifiers that correctly predict the current label~\y. 
It uses $\corrPred$ to compute $\powerIdx_i$ 
via a closed-form solution through subroutine \getpsi, 
the details of which are provided in~\Cref{alg:computeBANZ01}.
The correctness of~\Cref{alg:mainAlg01} follows from~\Cref{theo:Banz01,theo:Shap01}
and the linearity of power indices.

To bound the time complexity we assume that 
evaluating the prediction $\model(\xp)$ of each classifier $\model \in \modelSet$
can be done in $\bigO(1)$ time, and that each arithmetic operation takes $\bigO(1)$ time.\footnote{Our analysis can be trivially extended when evaluating a model on a test point has a different complexity.}
The correctness of~\Cref{theo:main01} and stated complexity follows immediately 
as a corollary of~\Cref{theo:Banz01,theo:Shap01,theo:runtimeBin}.

\begin{theorem}
\label{theo:runtimeBin}
\Cref{alg:mainAlg01} computes the Banzhaf values $\powerIdx_i = \banz_i(\sumGame)$ and 
the Shapley values $\powerIdx_i = \shap_i(\sumGame)$
over the game $\sumGame$ in time $\bigO(|\dataSet|\cardM)$. 
The memory required is $\bigO(\cardM)$.
\end{theorem}

\subsection{Multi\-class classification}\label{subsec:multiclass}

\para{Intuition.} 
Let $\model_i$ be a correct classifier; 
an analogous argument will apply to an incorrect classifier.
Consider all coalitions $\coal\subseteq[\cardM]$ containing $i$ 
where $\model_i$ yields non-zero marginal contribution $\charMC(\coal)-\charMC(\coal-i)$.
These are the only coalitions that contribute to the power indices $\shap_i$ or $\banz_i$. 
We characterize the structure of these coalitions and the exact marginal gain 
introduced by $\model_i$ in~\Cref{lemma:marginalMCUnweighted}.
Since the multi\-class setting introduces a much more complex structure, 
we cannot count these coalitions directly. 
Instead, we leverage ordinary generating functions \citep{Flajolet2013BookComb,kakimura2026computing} (see \Cref{def:ogf}).

Fix an integer $h\ge 1$ and consider coalitions of \emph{incorrect} classifiers where no incorrect label receives more than $h$ votes.
We design a polynomial $P_h(\gfvar)$ with variable $\gfvar$, 
where the $k$-th coefficient
$[\gfvar^k]P_h(\gfvar)$ is the number of coalitions in $\coalset(k,h)$, i.e., 
coalitions where exactly $k$ incorrect labels receive $h$ votes, and all remaining labels receive fewer than $h$ votes (\Cref{lemma:genFunction}).
To form a coalition in $\coalset(k,h)$, we select, independently for each incorrect label $\y_j$, $j\in[2,\totLabs]$, either: 
($a$) fewer than $h$ classifiers; or
($b$) exactly $h$ of the $\cardM_j$ classifiers voting for $\y_j$; and case ($b$) holds for exactly $k$ labels.
There are $A_j(h)=\sum_{q=0}^{h-1}\binom{\cardM_j}{q}$ ways to select classifiers for case ($a$) and $B_j(h)=\binom{\cardM_j}{h}$ ways to select classifiers for case ($b$).
We multiply $B_j(h)$ by $\gfvar$, such that the exponent $k\ge0$ of $\gfvar^k$ is the number of incorrect labels for case~($b$). 
This yields $P_h(\gfvar)=\prod_{j=2}^{\totLabs}(A_j(h)+B_j(h)\gfvar)$ of degree at most $\totLabs-1$, with $[\gfvar^k]P_h(\gfvar) = |\coalset(k,h)|$.
We then show that the marginal gain of a correct classifier on a coalition formed by correct classifiers and a coalition in $\coalset(k,h)$ depends only on $k$ and on whether there are $h-1$ or $h$ other correct classifiers (\Cref{lemma:marginalMCUnweighted}).
We use this result to obtain the Banzhaf value $\banz_i$ for each correct classifier $i$ by summing the marginal gains weighted by $[\gfvar^k]P_h(\gfvar)$ for $h\in[\cardM_1]$ and $k\in[\totLabs-1]$  (\Cref{theo:BanzMC}).
Dividing $P_h(\gfvar)$ by the factor $(A_j(h)+B_j(h)\gfvar)$ of label $\y_j$ yields $\banz_i$ for all classifiers voting for $\y_j$, $j\in[2,\totLabs]$ (\Cref{cor:genFunctionMinusJ}).
Each $P_h(\gfvar)$ can be computed in $\bigO(\totLabs^2)$ time.

On the other hand,
the Shapley formula weights each coalition $\coal$ by a function of its size $|\coal|$,
while each coefficient $[\gfvar^k]P_h(\gfvar)$ counts coalitions of different sizes together.
To fix this issue we rely on a bi\-variate polynomial $P_h(\gfvar,u)$ whose coefficient $[\gfvar^k u^m]P_h(\gfvar,u)$
is the number of coalitions in $\coalset(k,h)$ with exactly $m$ classifiers (\Cref{lemma:bivariateGenFunction}).
To do so, we track the size of the coalitions through a second variable $u$:
\begin{inparaenum}[($i$)]
	\item in case ($a$) above, selecting $q<h$ classifiers voting for $\y_j$ contributes $u^q$; while
	\item in case ($b$), selecting $h$ classifiers contributes $\gfvar u^h$.
\end{inparaenum}
The degrees of $P_h(\gfvar,u)$ in $\gfvar$ and $u$ are at most $\totLabs-1$ and $\cardM$, respectively.
We compute $P_h(\gfvar,u)$ by dynamic programming, and obtain the Shapley values by combining the coefficients of these polynomials over $h\ge1$ as in the Banzhaf case.
\ifarxiv
We also develop an algorithm based on fast polynomial multiplication in \Cref{appsubsec:fft}.
\fi

Finally, for the computation of both Banzhaf and Shapley  values, 
we heavily exploit symmetries to avoid redundant computations, 
achieving a highly-optimized algorithm in practice. We proceed to state our technical results; the proofs are given in~\Cref{app:proofs}.

\subsubsection{Computation of Banzhaf values in the multi\-class setting}\label{subsubsec:MCBZ}

We first introduce the uni\-variate polynomial we use to compute efficiently 
the Banzhaf values.

\begin{lemma}\label{lemma:genFunction}
	Let $\evz=(\xp,\y)$ be a test data point. 
	For any $h\ge 1$, consider the generating function
	\[
	P_h(\gfvar) = \prod_{j=2}^\totLabs (A_j(h) + B_j(h) \gfvar)\, , 
  \text{~~where~~} 
	A_j(h) = \sum_{q=0}^{h-1} \binom{\cardM_j}{q} 
  \text{~~and~~} 
  B_j(h) = \binom{\cardM_j}{h} \,.
	\] 
	Fix $h\in [\cardM]$ and $k\in [\totLabs-1]$. 
  Let $\coalset(k,h)$ be the set of coalitions $\coal \subseteq \modelSet'$ such that:
	\begin{inparaenum}[($i$)]
		\item $\modelSet' = \modelSet \setminus \{\model \in \modelSet :\allowbreak \model(\xp) = \y\}$;
		\item $|\tieSet(\coal)| = k$; 
		\item $\VotVecCoord_j(\coal) = h$ for each $j \in \tieSet(\coal)$.
	\end{inparaenum}
	Then
	$|\coalset(k,h)| = [\gfvar^k]P_h(\gfvar)\,.$
\end{lemma}

In other words, the coefficients of $P_h(\gfvar)$ correspond to the number of sets $|\coalset(k,h)|$
that are used to obtain a closed-form solution for $\banz_i(\gameMC_{\evz})$, 
for all correct classifiers $\model_i$. 
To obtain $\banz(\gameMC_{\evz})$ 
for all $\model_i$ voting for an incorrect label $\y_j\neq \y_1$,
we need to remove the contribution of classifiers voting for $\y_j$ 
from the sets in $\coalset(k,h)$ as specified by the next corollary. 

\begin{corollary}
\label{cor:genFunctionMinusJ}
Under the setting of~\Cref{lemma:genFunction}, fix $j \in \{2, \dots, \totLabs\}$,  
and let $\coalset_{-j}(k, h)$ be defined analogously to $\coalset(k,h)$ 
for coalitions $\coal \subseteq \modelSet' \setminus \{\model \in \modelSet : \model(\xp) = \y_j\}$. 
That is, no coalition $\coal$ from $\coalset_{-j}(k, h)$ contains a classifier voting for class $\y_j$. 
Then it holds that
\[
	|\coalset_{-j}(k,h)| = [\gfvar^k]\frac{P_h(\gfvar)}{A_j(h) + B_j(h)\gfvar} \,  .
\]
\end{corollary}

We show how to use the coefficients of $P_h(\gfvar)$ 
to obtain a closed-form solution for $\banz_i$'s. 

\begin{theorem}
\label{theo:BanzMC} 
Let $\evz=(\xp,\y_1) \in  \dataSet$ be a test data point. 
Let $f_1(k, h) = k{\cardM_1-1 \choose h}  + {\cardM_1-1 \choose h-1}$ for any $k, h \ge 0$ and 
$f_j(k, h) = \binom{\cardM_1}{h}\left[{\cardM_j-1 \choose h}  + \frac{1}{k+1}{\cardM_j-1 \choose h-1}\right]$
for any $k, h \ge 0$ and $j\in\{2,\ldots,\totLabs\}$. 
For any correct classifier $\model_i$ with $\model_i(\xp) = \y_1$ and 
any incorrect classifier $i'$ with $\model_{i'}(\xp) = \y_j$ it holds
\[
	\banz_i =  
  \tfrac{1}{2^{\cardM-1}} \left(1+ \sum_{h=1}^{\cardM_1}\sum_{k=1}^{\totLabs-1} \tfrac{f_1(k, h)}{k+1}  [\gfvar^k]P_h(\gfvar) \right) 
  \text{~~and~~} 
  \banz_{i'} = -\tfrac{1}{2^{\cardM-1}} \sum_{h=1}^{\cardM_j}\sum_{k=0}^{\totLabs-2} \tfrac{f_j(k, h)}{k+2}[\gfvar^k]\tfrac{P_h(\gfvar)}{A_j(h)+B_j(h)\gfvar}  \, .
	\]
\end{theorem}

\Cref{theo:BanzMC} has strong algorithmic implications: 
by computing the coefficients of polynomials $P_h(\gfvar)$ and $P_h(\gfvar)/(A_j+B_j(\gfvar))$, 
for $j\in\{2,\ldots,\totLabs\}$ and $h\ge 1$, 
we can efficiently compute $\banz_i(\gameMC_{\evz})$ 
for all $i\in[\cardM]$ and $\evz\in\dataSet$. 
We detail the entire procedure in~\Cref{alg:BanzMC}, which, 
in addition to the closed-form solutions of~\Cref{theo:BanzMC}, 
exploits \emph{symmetries} to further reduce the time complexity.

\begin{theorem}
\label{theo:runtimeMCBZ}
\Cref{alg:BanzMC} computes the exact Banzhaf values $\banz_{i}(\sumGameMC)$ 
of the base classifiers of an ensemble $\modelSet$
in the multi\-class setting in time $\bigO(|\unique| \cardM \totLabs^2+|\dataSet|\cardM)$. 
The memory required is $\bigO(\totLabs (\cardM+|\unique|))$.
\end{theorem}

We can further reduce the complexity of computing $P_h(\gfvar)$ 
from  $\bigO(\totLabs^2)$  to $\bigO(\totLabs\log\totLabs)$ by 
performing polynomial multiplication via the fast Fourier transform (FFT) algorithm, 
at the expense of higher constants and a more complex implementation. 
In practice, we find that our $\bigO(\totLabs^2)$ solution is highly efficient, 
as shown in our experimental evaluation (e.g., see~\Cref{tab:rtMC}).

\subsubsection{Computation of Shapley values in the multi\-class setting}\label{subsubsec:MCSHAP}

As already mentioned, 
we cannot use $P_h(\gfvar)$ to obtain the values of $\shap_i(\gameMC_{\evz})$, 
since the $k$-th coefficient $[\gfvar^k]P_h(\gfvar)$
accounts for sets of different cardinality. 
To fix this issue, we design a bi\-variate ordinary generating function 
allowing us to keep track of the size of each set contributing to $\coalset(k,h)$ 
by partitioning it as $\coalset(k,h) = \cup_{m\ge 1} \coalset(k,h,m)$ and 
requiring each set of $\coalset(k,h,m)$ to have cardinality~$m$. 
We introduce our main result below, 
while we defer the detailed technical results to~\Cref{appsubsec:shap}. 

For $h\ge1$ and $j\in[2,\totLabs]$, let
$A_j(h;u)\doteq\sum_{q=0}^{h-1}\binom{\cardM_j}{q}u^q$, $B_j(h;u)\doteq\binom{\cardM_j}{h}u^h$, and
$P_h(\gfvar,u)\doteq\prod_{j=2}^{\totLabs}(A_j(h;u)+B_j(h;u)\gfvar)$.
The exponent of $u$ will count the classifiers in a coalition, i.e., $[\gfvar^ku^m]P_h(\gfvar,u)=|\coalset(k,h,m)|$ (see \Cref{lemma:bivariateGenFunction}).
When $u=1$ we recover $A_j(h)$, $B_j(h)$ and $P_h(\gfvar)$ from \Cref{lemma:genFunction}.
For ease of notation let $\eta(a) \doteq \frac{(\cardM- a)!(a-1)!}{\cardM!}$, for $1\le a\le \cardM$.
Furthermore, let
\begin{align*}
f_1(k, h, m) & \doteq \frac{1}{k+1} \left( k\binom{\cardM_1-1}{h}\eta(m+h+1) + \binom{\cardM_1-1}{h-1}\eta(m+h) \right) \text{ and  }  \\
f_j(k, h, m) & \doteq \frac{1}{k+2}\binom{\cardM_1}{h} \left[\binom{\cardM_j-1}{h}\eta(m+2h+1)+\frac{1}{k+1}\binom{\cardM_j-1}{h-1}\eta(m+2h)\right]\,, \, j\in[2,\totLabs] \,.
\end{align*}
By convention, a term with a zero binomial is $0$, irrespective of the argument of $\eta(\cdot)$.
\begin{theorem}
\label{theo:ShapleyMCCorrect}
	Let $\evz=(\xp,\y) \in \dataSet$ be a test data point. %
	For any classifier $\model_i$ making a correct prediction, 
  i.e., $\model_i(\xp) = \y_1$, it holds that its Shapley value over $\gameMC_{\evz}$ is given by
	\[
	\shap_i(\gameMC_{\evz}) = 
    \frac{1}{\cardM} + \sum_{h=1}^{\cardM_1}\sum_{k=1}^{\totLabs-1}\sum_{m=kh}^{\cardM-\cardM_1}  
    f_1(k, h, m) [\gfvar^k u^m]P_{h}(\gfvar, u)  \,  .
	\]
	Furthermore,
	for any classifier $\model_i$ predicting an incorrect class $\model_i(\xp)=\y_j$, 
  $j\in\{2,\ldots,\totLabs\}$ it holds 
	\[
	-\shap_i(\gameMC_{\evz})  =  \sum_{h=1}^{\cardM_j}\sum_{k=0}^{\totLabs-2} \sum_{m=hk}^{\cardM-(\cardM_1+\cardM_j)} f_j(k,h,m) [\gfvar^ku^m]\frac{P_h(\gfvar,u)}{A_j(h;u)+B_j(h;u)\gfvar} \,  .
	\]
\end{theorem}

The above result implies that 
computing the coefficients of the polynomials $P_h(\gfvar,u)$, 
with degrees bounded by $\totLabs-1$ in $\gfvar$ and $\cardM-\cardM_1$ in $u$, 
enables a closed-form solution for $\shap_i$. 
To compute the coefficients of $P_h(\gfvar,u)$ we develop a 
dynamic-programming algorithm presented in~\Cref{app:missingCodes}.

\begin{theorem}
\label{theo:runtimeMCSHAP}
\Cref{alg:ShapMC} computes the exact Shapley values $\shap_{i}(\sumGameMC)$ 
of the base classifiers of an ensemble $\modelSet$
in the multi\-class setting in time $\bigO(|\unique| \cardM^3 \totLabs+|\dataSet|\cardM)$. 
The memory required is $\bigO(\totLabs (\cardM+|\unique|))$.
\end{theorem}

We further introduce \suiteShapImp, improving the above time complexity bound in \Cref{app:refinedComplexity}, which we evaluate in \Cref{appsubsec:synth}. 
The underlying idea of such method is to reduce the size of the dynamic programming matrices used to evaluate \Cref{theo:ShapleyMCCorrect}, we defer to \Cref{app:refinedComplexity} for more details.

\section{Experiments} 
\label{sec:exps}

\newcounter{qcounter}
\stepcounter{qcounter}
We consider the following research questions:
\begin{inparaitem}[\textbf{(Q\arabic{qcounter})}]
	\item evaluate the efficiency and scalability of \suite;\stepcounter{qcounter}
	\item compare the rankings produced by \suite and existing methods;\stepcounter{qcounter}
	\item investigate the use of the rankings of \suite for ensemble pruning;\stepcounter{qcounter} and
	\item use \suite for prompt ensembling of vision models on a zero-shot learning classification task.
\end{inparaitem}
We defer to \Cref{appsubsec:synth} for synthetic experiments validating the tightness of our time complexity bounds.

\para{Algorithms.}
Our algorithms for exact computation are denoted by $\suite$-\shap and $\suite$-\banz. 

\emph{Binary classification.} We consider the following baselines, 
described in detail in \Cref{appsubsec:appsetup}:
\texttt{ISCG} \citep{Ykhlef2017GraphShapEnsamble}; 
\texttt{Tr}-\texttt{MLE}\ \citep{Owen1972MLE}; 
\texttt{Tr}-\texttt{PS} \citep{Maleki2013Shap}; and
\texttt{WMG} \citep{Uno2012WMG}.
These baselines are associated with games defined in prior work.
We also consider 
\texttt{MC}-\shap and \texttt{MC}-\banz: Monte-Carlo approximations 
with sample size set to $\cardM^2$.%

\emph{Multiclass classification.} 
We only show the scalability of \suite (see~\Cref{tab:rtMC}), as all other methods, 
with the exception of \texttt{ISCG}, 
are not applicable to the multiclass setting, or have very high runtime.
We show an application of the output of \suite for the multi\-class setting in~\Cref{subsec:CLIPexp}.
For the prompt-ensembling  of~\Cref{subsec:CLIPexp} we also consider \texttt{Tr}-\texttt{MLE} and \texttt{Tr}-\texttt{PS}, which we apply through a ``correct-versus-rest'' reduction detailed in \Cref{appsec:bl}.

\para{Setup.} 
Details on our experimental setup are in \Cref{app:reprod}.\footnote{The code to reproduce our results is available at~\codeURL.}
on six popular datasets from the UCI repository (see~\Cref{appsubsec:appsetup}). 
For each dataset we obtain a 75/15/10 split of the data, 
where 75\% is used for training, 
15\% for computing the power indices, and 
10\% as test data for the experiments on sub-ensemble selection, see also \Cref{app:reprod}.
Baselines \texttt{Tr}-\texttt{MLE} and \texttt{Tr}-\texttt{PS} are from 
\cite{Rozemberczki2021ShapEnsemble}.
Unless otherwise stated we considered five random (seeded) splits of each dataset, processed as detailed above, and our metrics are averaged over such instances.
We then execute each method once per split.
Except for \texttt{MC} and \texttt{Tr}-\texttt{PS} all other methods are  deterministic.

\para{Ensembles.} We consider three different ensembles
\begin{inparaenum}[($i$)]
	\item \ensRF: a random forest with $\cardM=50$ trees;
	\item \ensAB: an ensemble built with AdaBoost containing $\cardM=150$ models; and
	\item \ensHET: and heterogeneous ensemble containing $\cardM=30$ models,
  10 decision trees and 5 of each of the remaining types: SVMs, logistic regressors,
  gradient boosting, and multi layer perceptrons. See~\Cref{app:reprod} for more details.
\end{inparaenum}

\begin{figure}[t]
	\centering
	\includegraphics[width=0.9\textwidth]{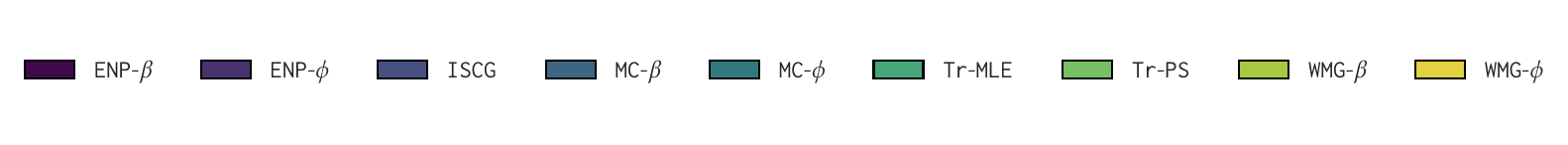}\\
	\begin{tabular}{ccc}
		\includegraphics[width=0.27\textwidth]{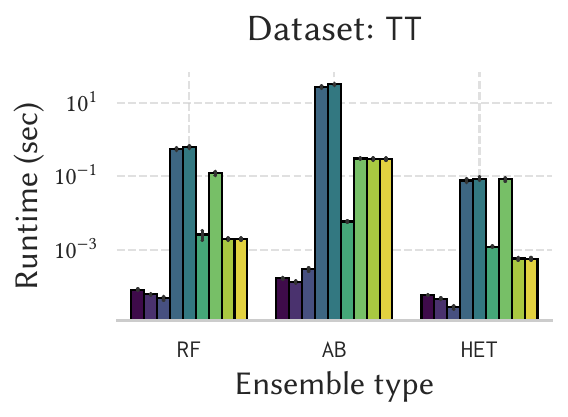}& \includegraphics[width=0.27\textwidth]{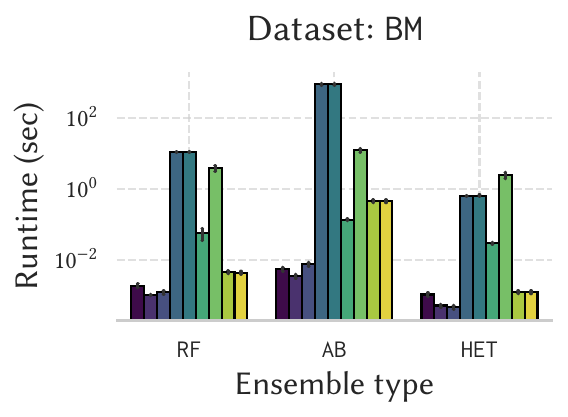} &
		\includegraphics[width=0.27\textwidth]{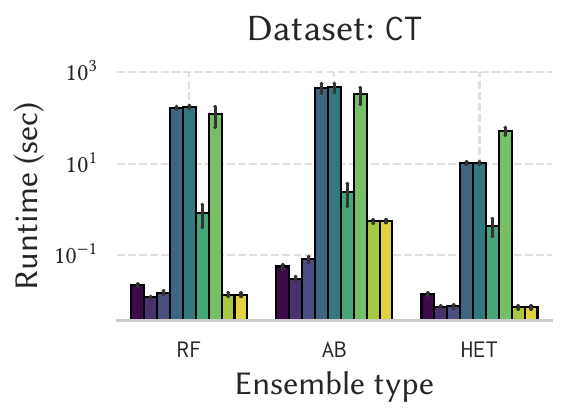}
	\end{tabular}
	\caption{Runtime (mean and standard deviation) on our largest datasets, binary classification.}
	\label{fig:rt1}
\end{figure}

\subsection{Efficiency and scalability}
\label{subsec:expsEff}
We start with answering \textbf{Q1}. 
The results are reported in~\Cref{fig:rt1} and~\Cref{fig:rt2}.
\suite methods are among the fastest for binary classification.
Even on the largest datasets, where the sampling methods \texttt{Tr}-\texttt{PS} and \texttt{MC} require up to hundreds and thousands of seconds, respectively,
\suite methods terminate in less than one second, as captured by~\Cref{theo:runtimeBin}.
The exact method \texttt{WMG} is also very efficient, however it requires between one and three orders of magnitude more time than \suite on \ensAB, note that such configuration has the largest number of base classifiers ($\cardM=150$).
In fact, the dynamic program of~\citet{Uno2012WMG} has a cost at least quadratic in $\cardM$ for each game of $\unique$, while~\Cref{alg:mainAlg01} evaluates a closed form.
\texttt{ISCG} has a comparable runtime to \suite, however it does not faithfully capture plurality voting ensembles (see \Cref{subsec:rankingComp} and \Cref{appsec:bl}).
Finally, \texttt{Tr}-\texttt{MLE} is more than one order of magnitude slower compared to \suite on our largest datasets, while providing only approximate solutions.
The runtimes of \suite for the multi\-class setting are reported in~\Cref{tab:rtMC}, where we observe a high runtime for computing Shapley values on our largest dataset (\texttt{CT}) on ensemble \texttt{AB} with $n=150$ due to the $\bigO(\cardM^3)$ term in the complexity. We would like to remind, however,  
that our method is the first specifically tailored to solve the multiclass~case.

\begin{figure}[t]
	\centering
	\begin{tabular}{cc}
		\includegraphics[width=0.49\textwidth]{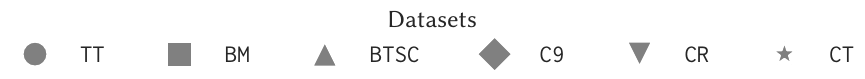}
		\includegraphics[width=0.45\textwidth]{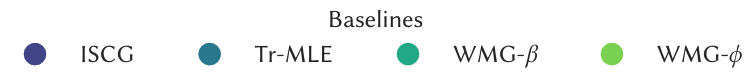} &%
	\end{tabular}
	\includegraphics[width=0.8\textwidth]{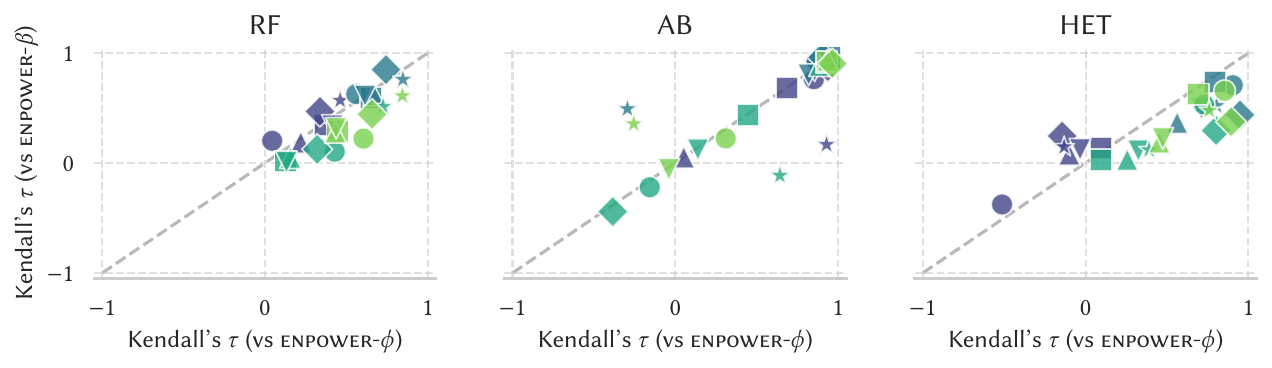} 
	\caption{For each pair (dataset, baseline) encoded by a pair (shape, color), we display the Kendall's $\tau$ with respect to the exact rankings obtained by $\suite$-$\shap$ and $\suite$-$\beta$ on the $x$-axis, and $y$-axis respectively, averaged over five splits. Each plot corresponds to a different ensemble.}
	\label{fig:corr}
\end{figure}

\subsection{Comparison of rankings of base classifiers}
\label{subsec:rankingComp}

Next, we consider \textbf{Q2}. 
We compare the rankings of the base classifiers produced by \suite and the baselines 
using Kendall's $\tau$. 
We recall, that rankings are used for different applications, including ensemble pruning, data validation, and interpretability. 
We report the results %
in~\Cref{fig:corr}.
We first observe that \texttt{ISCG} has low %
ranking correlation  with \suite on \ensRF and \ensHET, and a high correlation on \ensAB, except on \texttt{BTSC}.
Similarly, except for few instances,
the correlations of \texttt{WMG}-\banz and \suite-\banz do not exceed~0.5.
Interestingly, \texttt{WMG}-\shap exhibits slightly higher correlations with  \suite-\shap.
Finally,
there is quite high correlation ($>\!0.5$) among the rankings of \suite-\shap and
\texttt{Tr}-\texttt{MLE}, despite the latter being approximate and on a different game, with the exception of \texttt{CT} on \ensAB.

We note that the correlation values are particularly dependent on the baseline considered.
\texttt{Tr}-\texttt{MLE} has $\tau>0.5$ with \suite-\shap on most configurations except \texttt{CT} on \ensAB.
In contrast, baselines \texttt{WMG} and \texttt{ISCG} range from negative correlation to values above $0.9$ for a fixed ensemble and different datasets.
For example, the value of $\tau$ between  \texttt{WMG}-\banz and \suite-\shap on \ensAB is below $-0.3$ and above $0.8$, on datasets \texttt{C9} and \texttt{BTSC}, respectively.
In summary, we observe that \suite provides rankings that are not fully captured by previous methods, providing a unique characterization of plurality-voting~ensembles with random tie-breaking.

\begin{figure}[t]
	\centering
	\includegraphics[width=0.9\textwidth]{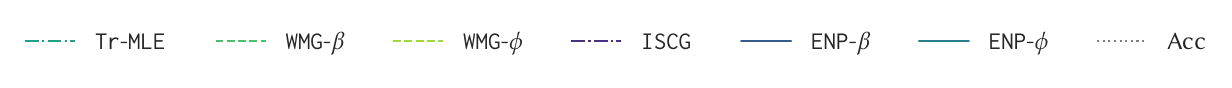}\\
	\begin{tabular}{cc}
		\includegraphics[width=0.45\textwidth]{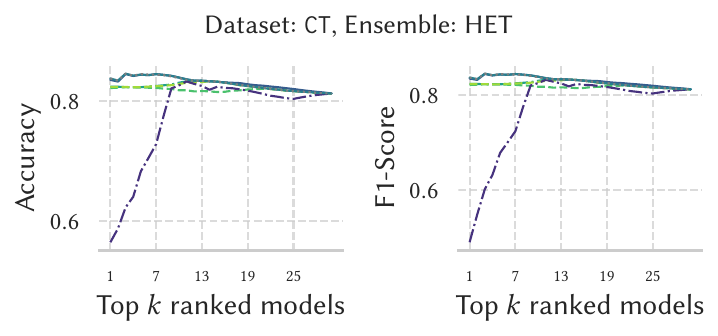}& \includegraphics[width=0.45\textwidth]{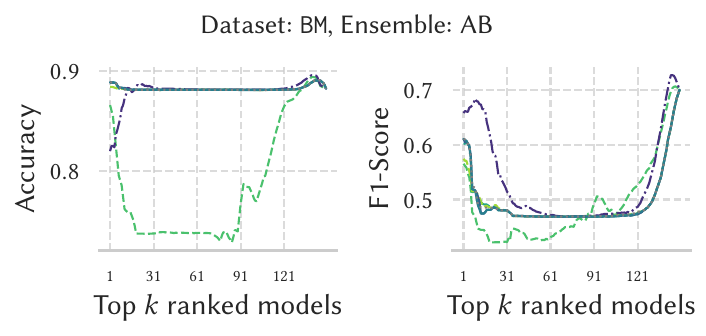} \\
	\end{tabular}
	\caption{Generalization metrics for sub-ensemble selection, averaged over five random splits.}
	\label{fig:gener1}
\end{figure}
\subsection{Sub-ensemble selection}
\label{sec:subens}
We now consider \textbf{Q3}.
Given an ensemble $\modelSet$ and a ranking %
of its base classifiers, 
we assess the generalization of the sub-ensemble obtained by retaining the top-$k$ classifiers.
We use rankings according to \suite and the baselines.
We additionally include a baseline ranking base classifiers by their individual accuracy on $\dataSetVal$, denoted as \textsc{Acc}.
For each sub-ensemble we evaluate its accuracy and (macro) $F_1$-score.
We compute the rankings of all methods on $\dataSetVal$, and both metrics on $\dataSetTra$, of the same split.
Classifiers with equal values are ordered by increasing index $\model_i, i\in[\cardM]$.
We vary the parameter $k$ in $[\cardM]$.
The predictions of a sub-ensemble are obtained by plurality voting with random tie-breaking, i.e., the voting rule of $\charBin$.
That is, for a test point where the two classes receive the same number of votes we draw the prediction uniformly at random.
We compute both metrics on the resulting predictions, and report their mean over the five random splits.

We show some representative results in~\Cref{fig:gener1} (see~\Cref{subsec:expsSubApp} for an extended discussion). 
We observe that sub-ensembles produced with rankings by~\suite are often among those with the largest accuracy and $F_1$-score.
For example, on \texttt{CT} with \ensHET, \suite matches or outperforms all baselines for all values of $k$; on \texttt{BM} with \ensAB, \suite matches \textsc{Acc} for all values of $k$, while the accuracy of \texttt{WMG}-\banz is much lower than that of all other methods for intermediate values of $k$.
Interestingly, on the same configuration, \texttt{ISCG} attains a higher $F_1$-score for small values of $k$.
To answer \textbf{Q3}, while our methods are not directly optimized for sub-ensemble selection, 
they can provide a valuable contribution in addition to existing methods, given their high practical applicability.

\subsection{Large vision model prompt ensembling}
\label{subsec:CLIPexp}

\begin{figure}[t]
	\centering
	\includegraphics[width=0.55\textwidth]{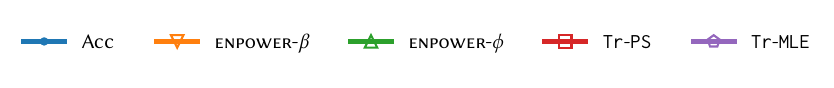}\\
	\begin{tabular}{cc}
		\includegraphics[width=0.47\textwidth]{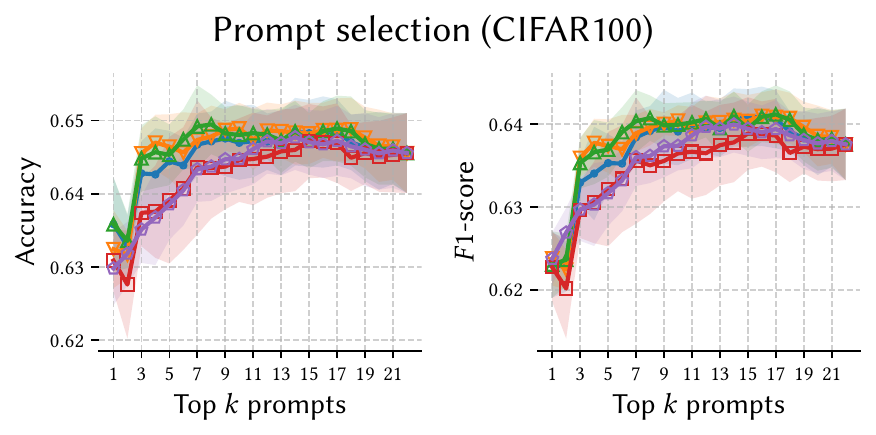}& 		
    \includegraphics[width=0.47\textwidth]{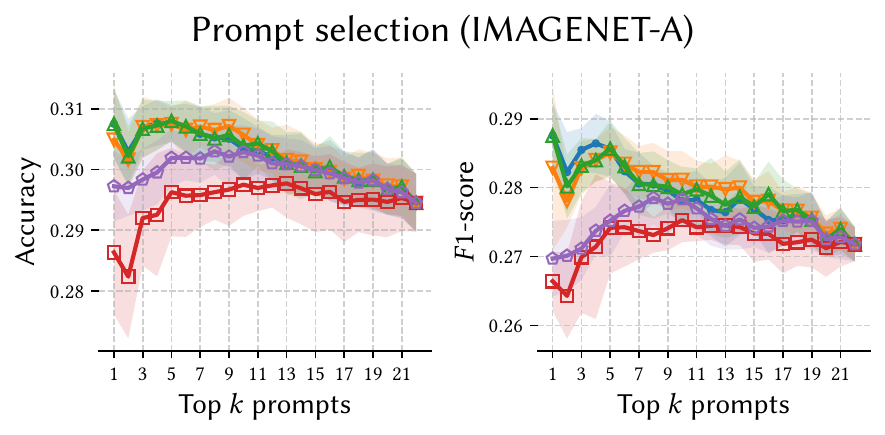}
	\end{tabular}
	\caption{Accuracy and $F_1$-score test metrics for the prompt ensembling of CLIP on $\dataSetTra$ (see~\Cref{subsec:CLIPexp}), for varying $k$. We report mean and standard deviation (over multiple splits).}
	\label{fig:accf1CLIP}
\end{figure}

Last, we consider \textbf{Q4}. 
Recall that prompt ensembling is an LLM-engineering technique used 
to improve performance on tasks such as zero-shot image classification~\citep{radford2021learningCLIP}. 
Here, we select $\cardM=22$ general-purpose prompt templates, 
which act as our models;
see~\Cref{tab:results_imagenet-a}. 
For each image in the input dataset we first compute its normalized embedding using CLIP, 
obtaining $\bm{x}_{\text{img}}$. Then we proceed as follows:
\begin{inparaenum}[($i$)]
	\item we instantiate each prompt template $i\in[\cardM]$ with the $\totLabs$ possible classes, obtaining prompts $S^i_1,\dots,S^i_\totLabs$;
	\item we use CLIP to compute the normalized embedding of each of the prompts
  obtaining vectors $\bm{v}_{1}^i, \dots, \bm{v}_{\totLabs}^i$;
	\item we compute the dot products $\bm{x}_{\text{img}}\cdot \bm{v}_{ j}^i$;
	\item we select the class $j^{*}(i)= \arg\max_j\bm{x}_{\text{img}}\cdot \bm{v}_{j}^i$ 
  as vote for model %
  $i\in[n]$;%
  \footnote{Compared to the procedure of CLIP \citep{radford2021learningCLIP} each model selects a class, not a probability distribution.}
	\item the final predictions of the ensemble are evaluated according to plurality voting with random tie-breaking. 
\end{inparaenum} 
We split each dataset in half obtaining $\dataSet_{\text{val}}$ and $\dataSetTra$ (over five random splits).
We use $\dataSet_{\text{val}}$ to compute the rankings of the models according to three different algorithms: 
\textsc{Acc}: ranking prompts by raw accuracy, a common practice for LLM engineering; \suite-$\banz$ and \suite-$\shap$: ranking using our definitions of $\banz$ and $\shap$ for the multi\-class setting.
We also rank prompts with \texttt{Tr}-\texttt{MLE} and \texttt{Tr}-\texttt{PS} (see~\Cref{appsec:bl}).
The resulting rankings are reported in~\Cref{tab:results_cifar100,tab:results_pets,tab:results_imagenet-a}.
Finally, we use $\dataSetTra$ of each split to evaluate the accuracy and $F_1$-score
on the sub-ensemble obtained by retaining only the top-$k$ models,
for $k\in\{1,\dots,\cardM \}$.

We report the result %
for the CIFAR and IMAGENET-A datasets  in~\Cref{fig:accf1CLIP}, the missing dataset is in~\Cref{subsec:CLIPpromptApp}, and the best and average metrics over $k$ in~\Cref{tab:clipSummary}.
We observe that the sub-ensembles obtained through the rankings computed by~\suite
attain a higher accuracy and $F_1$-score than the ones of \texttt{Tr}-\texttt{MLE} and \texttt{Tr}-\texttt{PS}, for the best $k$ and on average over all $k\in[\cardM]$, on all three datasets.
Importantly, \texttt{Tr}-\texttt{MLE} and \texttt{Tr}-\texttt{PS} not accounting for how  incorrect models for the $\totLabs-1$ incorrect labels, differently from our game $\charMC$ (see \Cref{appsec:bl,sec:RWExtended}).
Furthermore the selection of \suite matches or improves also over the baseline \textsc{Acc}.
We also note that in general, the best performing ensembles ($k\in [1,17]$, see~\Cref{tab:clipSummary}) have smaller sizes than the full ensemble of $\cardM=22$ prompts.
Highlighting that using a larger set of prompts to classify challenging images
often yields less accurate predictions.
Our results suggest that the Banzhaf and Shapley values computed by~\suite 
can be used to improve the predictive power of vision models under prompt ensembling. 
Another interesting application is enhancing vision model interpretabilty by comparing the rankings of the prompts across different methods, 
as shown in~\Cref{tab:results_cifar100,tab:results_pets,tab:results_imagenet-a}, for which we defer to the discussion in~\Cref{subsec:CLIPpromptApp}.

\section{Conclusions}
We introduce new characteristic function games modeling the predictions of models 
in plurality-voting ensembles with random tie-breaking, in binary and multi\-class settings.
We then design new \emph{exact} polynomial-time algorithms, named collectively as \suite, 
to compute the Banzhaf and Shapley values of the models in the ensemble. 
Our algorithms are based on non-trivial combinatorial identities, and the use of ordinary generating functions for the multi\-class case.
We show in practice that the indices obtained by \suite:
\begin{inparaenum}[($i$)]
	\item  are extremely efficient to compute and are not captured by previous methods;
	\item can be applied to ensemble pruning and LLM prompt ensembling with remarkable results.
\end{inparaenum}
Interesting future directions include:
\begin{inparaenum}[($i$)]
	\item use our indices for data validation;
	\item study other popular power indices under our new game.
\end{inparaenum}

\begin{ack}
The authors declare no competing interests that could create a conflict of interest.
This research was supported by the Swedish Research Council project ExCLUS (2024-05603), the Marie Curie Doctoral Training Network ARMADA (101168951), and the Wallenberg AI, Autonomous Systems and Software Program (WASP) funded by the Knut and Alice Wallenberg Foundation. Computational resources were provided on the Berzelius system funded by the Knut and Alice Wallenberg foundation and operated by NAISS.
\end{ack}

\cleardoublepage

\bibliographystyle{plainnat}
\bibliography{biblio}

\cleardoublepage
\section*{Limitations}
We now list the major limitations and assumptions of our work.
\begin{itemize}
	\item our methods are only applicable to ensembles constituting of base models that provide a hard classification;
	\item there are inputs for which our methods do not match baseline methods; we describe such behavior in~\Cref{sec:exps} and discuss the related trade-offs. For example, on ensemble sub-selection, our methods may achieve high accuracy but low $F_1$-score, or vice-versa. To support our findings we provide additional results in~\Cref{sec:additionalExps}.
	\item we force each base classifier to provide a hard classification for our experiment regarding prompt ensembling in~\Cref{subsec:CLIPexp}, to fit our framework. We remark that the original work of~\cite{radford2021learningCLIP} used soft votes.
\end{itemize}

\section*{LLM usage}
Large language models (LLMs) were used to assist portions of this research and manuscript preparation. In particular, LLMs were used to:
\begin{itemize}
	\item assist with proofreading and improving presentation;
	\item help explore proof ideas and check edge cases;
	\item assist with implementation and debugging;
	\item support data analysis and visualization;
	\item assist with literature exploration.
\end{itemize}

All technical results, proofs, experiments, and conclusions were verified by the authors, who assume full responsibility for the content of the manuscript.

\appendix
\clearpage
\etocdepthtag.toc{mtappendix}
\etocsettagdepth{mtmain}{none}
\etocsettagdepth{mtappendix}{subsection}
\etocsettocstyle{\section*{Appendix content}}{}
\tableofcontents
\clearpage

\section{Extended discussion on related work}
\label{sec:RWExtended}

The generalization of an ensemble is characterized by a \emph{bias-variance-diversity} complexity trade-off that must be considered during learning~\citep{Dietterich2000Ensembles,Kuncheva2003DiversityMetrics}. 
\emph{Diversity}, which informally captures correlations among the models' predictions,  
requires a careful balancing as it directly impacts bias and variance~\citep{Wood2023Diversity}. 
Diversity motivates the need for ensemble \emph{pruning}~\citep{Li2012Pruning,Lu2010PruningScoreDiv}, 
an active area of research~\citep{Kulkarni2012Pruning,zhang2006ensemble}, 
which aims at selecting a sub-ensemble with better generalization performance, 
reducing significantly the size of the full ensemble and controlling the bias-variance-diversity trade-off.
Our methods %
provide a valuable addition to existing pruning techniques (as we show in~\Cref{sec:exps}), 
especially given the efficiency of \suite and the fact that we our methods are model-agnostic.
We defer to dedicated textbooks and surveys for more details~\citep{Sagi2018EnsSurv,Zhou2025EnsembleMethods}.

Weighted majority games (WMGs) have been studied extensively in the literature, 
and under unitary weights, 
power indices for WMGs can be computed in polynomial time~\citep{Uno2012WMG,kakimura2026computing}.
However, as we discuss below, WMGs do not capture our setting.
First, we recall that a WMG is a simple game, i.e., a monotone game in which coalitions are rewarded with values that are either 0 or 1, and that its quota, which is necessary for achieving a majority,
is \emph{fixed in advance} (see~\Cref{def:WMG}).
Consider an ensemble $\modelSet$ with unitary weights, as in our setting, and a quota $q\in[\cardM]$.
Given a model $i\in[\cardM]$ and a coalition $\coal$ such that $i\notin\coal$, the marginal contribution $\chF(\coal+i)-\chF(\coal)$ is nonzero if and only if $|\coal|=q-1$. 
That is, the Shapley and Banzhaf values of model $i\in[\cardM]$ in a WMG depend only on the coalitions of size $q-1$.
For example, under the majority quota $q=\lfloor\cardM/2\rfloor+1$, no coalition with fewer than $\lfloor\cardM/2\rfloor$ classifiers contributes to the power indices of the models. 
In our experiments, we rank the models by their power indices and retain the top-$k$ models, for every $k\in[\cardM]$ (see~\Cref{sec:exps}).
When $k<\lfloor\cardM/2\rfloor$, the power indices of a WMG ignore every coalition of size at most $k$.
In contrast, under our game, the power indices can depend on coalitions of any size (see~\Cref{lemma:marginalBinaryUnweighted}).

Furthermore, most of the existing game formulations on which prior work is based, 
including the works by~\citet{Rozemberczki2021ShapEnsemble, Uno2012WMG} and 
\citet{Ykhlef2017BanzWeighted} cannot be adapted to the multi\-class setting 
without incurring high computational cost and loss in information. 
The reason is that typical approaches define $\bigO(\totLabs)$ binary games and 
aggregate the indices over all pairs retaining the correct label, or all possible pairs solving $\bigO(\totLabs^2)$ binary games. 
Despite being inefficient, this approach does not capture how the predictions 
for different classes interact with each other, as it is done by our definition of $\charMC$. 
Note that the work by \citet{Ykhlef2017GraphShapEnsamble} is the only method applicable for the multi\-class setting, which however does not model accurately how the base models interact over the final ensembles' predictions. 
That is, the work of~\cite{Ykhlef2017GraphShapEnsamble} assumes that classifiers correspond to nodes on a clique $Q$, and the characteristic function for a subset $\coal$ of models is proportional to the weighted volume of the induced subgraph $Q[\coal]$, using the classical result from \citet{Deng1994WMG} (see also \Cref{appsec:bl}).
This differs significantly from aggregation mechanisms used in ensemble models.

Finally, the rule underlying our characteristic function $\charMC$, i.e., plurality voting with tie-breaking, has been studied in social choice.
\citet{Bolger1990} and \citet{Amer1998} extend the Banzhaf index to games with $r$ alternatives (here $r=\totLabs$). 
Their characteristic function is defined on the partitions of the players into $r$ groups, one per alternative.
The marginal contribution of a player $p$ corresponds to the variation in the value of the characteristic function when the vote of $p$ moves from one group to another.
The resulting power index evaluates the marginal contributions over the partitions of $[\cardM]$, of which there are up to $\totLabs^{\cardM}$, instead of over the $2^{\cardM}$ coalitions considered in our work.
Note that such indices range over all possible votes, thus in an unweighted setting all players are symmetric, and every player receives the same value, i.e., the indices cannot distinguish the classifiers.
\citet{Kurz2020,Mayer2021} characterize the weight vectors that yield the same outcome for every voting profile under plurality and other voting rules,
and \citet{Kurz2021} define new power indices for weighted committees.
\citet{Chua2002} compare power indices for weighted plurality, modeled as a simple game.
We remark that none of the above works considers a notion of correctness for the votes of the players.
In contrast, our power indices are computed for the votes of the classifiers, and vary across classifiers with different predictions even though the votes are unweighted, as shown both by our theoretical results (see \Cref{theo:BanzMC,theo:ShapleyMCCorrect}) and in practice (see \Cref{sec:exps}).

\section{Background}
\label{appendix:section:background}

\subsection{Cooperative game theory}
\label{appendix:subsection:CGT}

Both the Banzhaf and Shapley values are grounded in well-established axiomatic foundations 
from cooperative game theory~\citep{Chalkiadakis2012ComputationalBook}, which we recall below. 
The Shapley values is uniquely characterized by the four properties of
dummy player, symmetry, linearity, and efficiency, 
whereas the Banzhaf value satisfies the properties of dummy player, symmetry and linearity. 

\begin{property}[Dummy player]
\label{property:dummy}
Given a game $G= (N,v)$ a player $i\in N$ is said to be a \emph{dummy player} 
if it holds that $v(S+i) = v(S)$ for all coalitions $S\subseteq[N]$ that do not contain $i$.
The value $\psi_i(v)$ of a dummy player in a game $G$ should equal to 0. 
\end{property}

\begin{property}[Symmetry]
\label{property:symmetry}
If $i$ and $j$ are two players in a game $G= (N,v)$ 
who are equivalent, in the sense that 
$v(S+i)=v(S+j)$ for any coalition $S\subseteq N$ that contains neither $i$ nor $j$, 
then it should be $\psi_i(v)=\psi_j(v)$.
\end{property}

\begin{property}[Linearity]
\label{property:linearity}
If two coalition games $G_v= (N,v)$ and $G_w= (N,w)$ are combined, 
then their combined values should correspond to the  values derived from $v$ 
and the values derived from~$w$, that is, it should be
$\psi_i(v+w) = \psi_i(v)+\psi_i(w)$, for every $i\in N$
Also, for any real number $\alpha$, it should be 
$\psi_i(\alpha\,v) = \alpha\,\psi_i(v)$, for every $i\in N$.
\end{property}

\begin{property}[Efficiency]
\label{property:efficiency}
The sum of the values of all players in a game $G= (N,v)$ 
should be equal to the value of the grand coalition, i.e., 
$\sum_{i\in N} \psi_i(v) = v(N)$. 
\end{property}

A \emph{sum game} is the cooperative game whose characteristic function is 
the point\-wise sum of those of a set of games defined over the same set of players.

\begin{definition}[Sum game]
\label{def:sumGame}
	Given an integer $k\ge 2$, and characteristic function games $G_1= (N,v_1), \dots, G_k= (N,v_k)$ the 
  \emph{sum} of $G_1, \dots, G_k$ denoted with $G_+ = G_1+\dots + G_k = \sum_i G_i$ 
  is a game $G_+=(N,v_+)$ such that for any coalition 
  $S\subseteq N$ it holds $v_+(S) = \sum_{i=1}^k v_i(S)$.
\end{definition}

\emph{Weighted majority games} (WMGs) are a classical model of voting systems 
in which each player carries a weight and a coalition wins 
precisely when its total weight reaches a fixed quota
\citep{Chalkiadakis2012ComputationalBook,Matsui2000WMG}.

\begin{definition}[Weighted majority game (WMG)]
\label{def:WMG}
	Let $N = \{1,\dots,n\}$ be a finite set of players. A \emph{weighted majority game} is a characteristic function game $(N,v)$ for which given weights $w_1,\dots,w_n \in \mathbb{R}_{\ge 0}$ and a quota $q > 0$ such that, for every coalition $S \subseteq N$,
	\[
	v(S) =
	\begin{cases}
		1 & \text{if } \sum_{i \in S} w_i \ge q, \\
		0 & \text{otherwise} \,  .
	\end{cases}
	\]
\end{definition}

\subsection{Combinatorics}

We recall some standard combinatorial results used in our analysis.

\begin{lemma}[Pascal's rule]
\label{lemma:auxNcK}
	For any $n,k\in \mathbb{N}_{>0}$ it holds
	\[
	\binom{n-1}{k}
	+
	{n-1 \choose k-1} = {n \choose k} \, ,
	\]
	where we assume ${n \choose k} = 0$ whenever $n < k$.
\end{lemma}

\begin{lemma}[Vandermonde's identity]
\label{lemma:vand}
	For any non-negative integers $a,b,c \in \mathbb{N}_{>0}$ it holds
	\[
	\sum_{j=0}^c {a\choose j}{b \choose c-j} = {a+b\choose c}\,  .
	\]
\end{lemma}

\begin{definition}[\citet{Flajolet2013BookComb}, Chapter 1]\label{def:ogf}
	The \emph{ordinary generating function}  of a sequence $\{A_j\}$ is the power series 
	\[
	P(x)=\sum_{j=0}^\infty A_j x^j\, .
	\]
	Variable $x$ is denoted as the ordinary generating function \emph{indeterminate}. We denote with $[x^\ell]P(x)$ the operation of extracting the coefficient of $x^\ell$ in $P(x)$.
\end{definition}

\section{Proofs of main results}
\label{app:proofs}

\subsection{Auxiliary technical results}
\begin{lemma}\label{lemma:marginalBinaryUnweighted}
	Let  $(\xp, \y)$ be a data point and consider the characteristic function $\charBin$.
  Without loss of generality let $y=\y_1\in \labelSet$. 
  For any classifier $\model_i\in \modelSet$, and any coalition 
  $\coal \subseteq \modelSet$ with $i\in\coal$ and $\coal-i\neq \emptyset$, 
  the marginal gain $\charBinS{\coal}-\charBinS{\coal-i}$ takes values
	\begin{equation*}
    \charBinS{\coal}-\charBinS{\coal-i} = 
		\begin{cases}
			\frac{1}{2}, & \text{if } \model_i(\xp) = \y, (\VotVecCoord_{1}(\coal-i) = \VotVecCoord_{2}(\coal-i) ) \text{ or } (\VotVecCoord_1(\coal-i) = \VotVecCoord_2(\coal-i)-1)\, ,\\
			-\frac{1}{2}, & \text{if } \model_i(\xp) \neq \y, (\VotVecCoord_1(\coal-i) = \VotVecCoord_2(\coal-i)) \text{ or } (\VotVecCoord_1(\coal-i)-1 = \VotVecCoord_2(\coal-i))\, ,\\
			0, & \text{otherwise} \, .
		\end{cases}
	\end{equation*}
	If $\coal-i = \emptyset$ then $\charBinS{\coal}-\charBinS{\coal-i}$ is 1 if $\model_i(\xp) = \y$, and 0 if $\model_i(\xp) \neq \y$.
\end{lemma}
\begin{proof}[Proof of~\Cref{lemma:marginalBinaryUnweighted}]
	We fix a data point $(\xp, \y=\y_1)$ and a classifier $\model_i \in \modelSet$.
  To compute the marginal gain $\charBinS{\coal}-\charBinS{\coal-i}$, 
  for any $\coal \subseteq \modelSet$ with $i\in \coal$, 
	first note that since $\charBin(\emptyset)=0$, 
  it follows that if $\coal-i = \emptyset$ then
  $\charBinS{\coal}-\charBinS{\coal-i}$ is $1$ when $\model_i(\xp) = \y$, and
  it is $0$ when $\model_i(\xp) \neq \y$.
  Hence we focus on the case $\coal-i\neq \emptyset$.
	
	\textbf{Correct classifier.}  ($\model_i(\xp) = \y$)
	We distinguish the following cases, which follow from the definition of~$\charBin(\cdot)$, see~\Cref{eq:charBin}:
	\begin{compactenum}[($i$)]
		\item $\charBinS{\coal-i} = 0$, hence $\VotVecCoord_1(\coal-i) < \VotVecCoord_2(\coal-i)$;
		\item $\charBinS{\coal-i} = \frac{1}{2}$, hence $\VotVecCoord_1({\coal-i}) = \VotVecCoord_2(\coal-i)$; and
		\item $\charBinS{\coal-i} = 1$, hence $\VotVecCoord_1(\coal-i) > \VotVecCoord_2(\coal-i)$.
	\end{compactenum} 
	Clearly for case ($iii$) it follows that $\charBinS{\coal} = 1$, and 
  therefore $\charBinS{\coal}-\charBinS{\coal-i}=0$. 
	For case ($i$) it holds that if $\VotVecCoord_1(\coal-i) = \VotVecCoord_2(\coal-i)-1$ 
	then $\charBinS{\coal} = \frac{1}{2}$, otherwise 
  if $\VotVecCoord_1(\coal-i) < \VotVecCoord_2(\coal-i)-1$ then clearly $\charBinS{\coal}=0$. 
  It follows that for case ($i$) 
  it is $\charBinS{\coal}-\charBinS{\coal-i}=\frac{1}{2}$ 
  if and only if $\VotVecCoord_1(\coal-i) = \VotVecCoord_2(\coal-i)-1$.
	Finally, for case ($ii$) if $\charBinS{\coal-i} = \frac{1}{2}$ it holds that $\charBinS{\coal} = 1$, proving the first part of the claim.
	
	\textbf{Incorrect classifier.} ($\model_i(\xp) \neq \y$) 
  The claim follows by a symmetrical argument 
  noting that any classifier $i\in \modelSet$ making an incorrect prediction 
  can only decrease the value of $\charBin(\coal-i)$, for any $\coal-i\subseteq\modelSet$. 
	Specifically,
	\begin{compactenum}[($i$)]
		\item if $\charBinS{\coal-i} = 1$ and $\VotVecCoord_1(\coal-i) - 1 = \VotVecCoord_2(\coal-i)$, then $\charBinS{\coal} = \frac{1}{2}$, thus $\charBinS{\coal}-\charBinS{\coal-i}=-\frac{1}{2}$.
		\item If $\charBinS{\coal-i} = \frac{1}{2}$ then $\VotVecCoord_1(\coal-i)=\VotVecCoord_2(\coal-i)$.
    Hence, a classifier $i$ making an incorrect prediction breaks the tie negatively, 
    so $\charBinS{\coal} = 0$, yielding a marginal contribution of $-\frac{1}{2}$.
		\item In all other cases, the characteristic function value does not change, yielding a marginal contribution of $0$.
	\end{compactenum}
\end{proof}

\begin{lemma}\label{lemma:marginalMCUnweighted}
	Let $(\xp, \y)$ be a data point and consider the characteristic function $\charMC$. 
  Without loss of generality let $\y=\y_1\in \labelSet$. 
	For any classifier $\model_i\in \modelSet$ with $\model_i(\xp) = \y$, 
  and any coalition $\coal \subseteq \modelSet$ with $i\in\coal$ and $\coal-i\neq \emptyset$, 
  the marginal gain $\charMC(\coal)-\charMC(\coal-i)$ takes values
	\[
	\begin{cases}
		\frac{1}{|\tieSet(\coal-i)|+1}, & \text{if } |\tieSet(\coal-i)|>0 \text{ and } \VotVecCoord_1(\coal-i) = \VotVecCoord_j(\coal-i)-1, \text{ for each } j\in\tieSet(\coal-i)\, ,\\
		\frac{|\tieSet(\coal-i)|-1}{|\tieSet(\coal-i)|}, & \text{if } |\tieSet(\coal-i)|>1 \text{ and } \VotVecCoord_1(\coal-i) = \VotVecCoord_j(\coal-i), \text{ for each } j\in\tieSet(\coal-i)\, ,\\
		0, & \text{otherwise} \, .
	\end{cases}
	\]
	On the other hand, 
  for any $\model_i\in \modelSet$ with $\model_i(\xp) = \y_j$, 
  where $\y_j\neq \y_1$, and any coalition $\coal \subseteq \modelSet$ with $i\in\coal$ and $\coal-i\neq \emptyset$, the marginal gain $\charMC(\coal)-\charMC(\coal-i)$ takes values
	\[
	\begin{cases}
		-\frac{1}{|\tieSet(\coal-i)|}, & \text{if } \VotVecCoord_{j}(\coal-i) = \VotVecCoord_{j'}(\coal-i), \text{ for each } j'\in\tieSet(\coal-i) \text{ and } 1 \in \tieSet(\coal-i)\, ,\\
		-\frac{1}{|\tieSet(\coal-i)|(|\tieSet(\coal-i)|+1)}, & \text{if } \VotVecCoord_{j}(\coal-i) = \VotVecCoord_{j'}(\coal-i)-1, \text{ for each } j'\in\tieSet(\coal-i) \text{ and } 1 \in \tieSet(\coal-i)\, ,\\
		0, & \text{otherwise} \, .
	\end{cases}
	\]
	Finally, if $\coal-i = \emptyset$, 
  then $\charMC(\coal)-\charMC(\coal-i)$ is $1$ if $\model_i(\xp) = \y$, and 
  it is $0$ if $\model_i(\xp) \neq \y$.
\end{lemma}

\begin{proof}[Proof of~\Cref{lemma:marginalMCUnweighted}]
	We fix a data point $(\xp, \y=\y_1)$ and a classifier $\model_i \in \modelSet$.
  We compute the marginal gain $\charMC(\coal)-\charMC(\coal-i)$ 
  for any $\coal \subseteq \modelSet$ with $i\in \coal$. 
	First note that since $\charMC(\emptyset)=0$, if $\coal-i = \emptyset$ then
  $\charMC(\coal)-\charMC(\coal-i)$ is $1$ when $\model_i(\xp) = \y$, and
  it is $0$ when $\model_i(\xp) \neq \y$.
  Hence we focus on the case $\coal-i\neq \emptyset$.
	
	We distinguish the following cases, which follow from the definition of~$\charMC(\cdot)$, 
  see~\Cref{eq:charMC}:
	\begin{compactenum}[($i$)]
		\item $\charMC(\coal-i) = 0$, hence $\VotVecCoord_1(\coal-i) < \VotVecCoord_j(\coal-i)$ for $j\in \tieSet(\coal-i)$;
		\item $\charMC(\coal-i) = 1/|\tieSet(\coal-i)|$, hence $\VotVecCoord_1({\coal-i}) = \VotVecCoord_j(\coal-i)$ for $j\in \tieSet(\coal-i)$, $|\tieSet(\coal-i)|\ge1$;
	\end{compactenum} 
	 
  \textbf{Correct classifier.} ($\model_i(\xp) = \y$)
We first consider case ($i$) for which $1\notin \tieSet(\coal-i)$. 
Note that $\VotVecCoord_1(\coal) = \VotVecCoord_1(\coal-i)+1$ and hence if 
  $\VotVecCoord_1(\coal-i)+1 < \VotVecCoord_j(\coal-i) = \VotVecCoord_j(\coal)$ then 
  $\charMC(\coal) = 0$ and therefore $\charMC(\coal)-\charMC(\coal-i)=0$. 
  Thus, it holds $\charMC(\coal)-\charMC(\coal-i)\neq 0$ if and only if 
  $\VotVecCoord_1(\coal-i)+1 = \VotVecCoord_j(\coal-i)$ for each $j\in \tieSet(\coal-i)$. 
  Therefore $\charMC(\coal) = 1/(|\tieSet(\coal-i)|+1)$ and in turn 
  $\charMC(\coal)-\charMC(\coal-i)= 1/(|\tieSet(\coal-i)|+1)$. 
	
  For case ($ii$) note first that $\VotVecCoord_j(\coal)=\VotVecCoord_j(\coal-i)$ for each $j\ge 2$,
  since $\model_i(\xp) = \y$.
  Hence, for each $j\in \tieSet(\coal-i), j\ge 2$, it holds
  $\VotVecCoord_1(\coal) = \VotVecCoord_1(\coal-i)+1 = \VotVecCoord_j(\coal-i) + 1> \VotVecCoord_j(\coal)$,
  and this also holds for each $j\notin \tieSet(\coal-i), j\ge 2$, that is
  $\VotVecCoord_j(\coal-i) < \VotVecCoord_1(\coal-i)$.
  Therefore $\tieSet(\coal) = \{1\}$ and $\charMC(\coal)=1$.
Note that it must be that $|\tieSet(\coal-i)|>1$ otherwise if $|\tieSet(\coal-i)|=1$ then $\charMC(\coal) = \charMC(\coal-i)$, yielding no marginal gain for the addition of $i$ to $\coal-i$.
Therefore, we have that $\charMC(\coal)-\charMC(\coal-i) = 1 - 1/|\tieSet(\coal-i)|$. 
This yields the claim for correct classifiers.

  \textbf{Incorrect classifier.}  ($\model_i(\xp) = \y_j, j\ge 2$)
	For case ($i$) we have that $\VotVecCoord_1(\coal) = \VotVecCoord_1(\coal-i)$,
  since $\model_i(\xp) \neq \y_1$, while
  $\max_{j'\in[\totLabs]}\VotVecCoord_{j'}(\coal) \ge \max_{j'\in[\totLabs]}\VotVecCoord_{j'}(\coal-i) > \VotVecCoord_1(\coal-i)$,
  where the last inequality is the defining property of case ($i$).
  Hence $1\notin \tieSet(\coal)$, and therefore $\charMC(\coal) = 0 = \charMC(\coal-i)$.
  That is, for case ($i$) it holds that $\charMC(\coal)-\charMC(\coal-i) = 0$.
	
  For case ($ii$) we distinguish the following two sub-cases
	\begin{compactenum}
		\item[($ii$.$a$)] $\VotVecCoord_j(\coal-i) = \VotVecCoord_{j'}(\coal-i)$ for each $j'\in\tieSet(\coal-i)$ and $1\in \tieSet(\coal-i)$;
		\item[($ii$.$b$)]  $\VotVecCoord_j(\coal-i) = \VotVecCoord_{j'}(\coal-i)-1$ for each $j'\in\tieSet(\coal-i)$ and $1\in \tieSet(\coal-i)$.
	\end{compactenum}
	For ($ii$.$a$) first note that $\charMC(\coal-i)=1/(|\tieSet(\coal-i)|)$ by~\Cref{eq:charMC}. 
  Thus, $\VotVecCoord_j(\coal) = \VotVecCoord_j(\coal-i)+1 > \VotVecCoord_{j'}(\coal-i)$ 
  for each $j'\in\tieSet(\coal-i)$,
  implying $\charMC(\coal) = 0$ and therefore 
  $\charMC(\coal) -\charMC(\coal-i) = 0-1/(|\tieSet(\coal-i)|) = -1/|\tieSet(\coal-i)|$. 

  For ($ii$.$b$) it holds $\charMC(\coal-i) = 1/(|\tieSet(\coal-i)|)$. 
	Since $\VotVecCoord_j(\coal-i) = \VotVecCoord_{j'}(\coal-i)-1$ we have that $\VotVecCoord_j(\coal) = \VotVecCoord_j(\coal-i)+1= \VotVecCoord_{j'}(\coal-i)$ and therefore  
  $j\notin \tieSet(\coal-i)$ while $j\in\tieSet(\coal)$. 
	This implies that $\charMC(\coal) = 1/(1+|\tieSet(\coal-i)|)$ and therefore 
  $\charMC(\coal) - \charMC(\coal-i) 
      = 1/(1+|\tieSet(\coal-i)|) - 1/(|\tieSet(\coal-i)|) 
      = -1/[(|\tieSet(\coal-i)|)(1+|\tieSet(\coal-i)|)]$. 

	Finally, it is sufficient to observe that in all other cases where $1\notin \tieSet(\coal-i)$ the addition of $i$ can never alter the value of $\charMC(\coal-i)$. Similarly, if $1\in\tieSet(\coal-i)$ and $\VotVecCoord_{j'}(\coal-i) - \VotVecCoord_j(\coal-i) \ge 2$ for $j' \in \tieSet(\coal-i)$, then $\VotVecCoord_j(\coal) = \VotVecCoord_j(\coal-i)+1 < \VotVecCoord_{j'}(\coal-i)$, thus $\tieSet(\coal)=\tieSet(\coal-i)$ and the addition of $i$ yields $\charMC(\coal)=\charMC(\coal-i)$. 
	Since $\VotVecCoord_{j'}(\coal-i) = \max_{h\in[\totLabs]}\VotVecCoord_h(\coal-i)$ for $j'\in\tieSet(\coal-i)$, our case split for $\VotVecCoord_{j'}(\coal-i) - \VotVecCoord_j(\coal-i)$ in $\{0,1, \ge 2\}$ is exhaustive, which proves the claim for an incorrect classifier and concludes the proof.
\end{proof}

\subsection{Proofs of~\Cref{sec:prelim}}

\begin{lemma}[Restatement of \Cref{lemma:charNonProp}]
  \label{lemma:no-super-no-mono}
	The characteristic function $\charBin: 2^\modelSet \to \{0,\frac{1}{2},1\}$ is  
	\begin{compactenum}
		\item[($i$)] neither supermodular;
		\item[($ii$)] nor superadditive;
		\item[($iii$)] nor monotone.
	\end{compactenum}
\end{lemma}
\begin{proof}%
  We prove the lemma by providing counterexamples. 
	Consider an ensemble~$\modelSet = \{\model_1, \dots, \model_5\}$ with 5 classifiers. 
  Fix a data-point $z=(x,1)$ and let $\model_1(x)=\model_2(x)=\model_5(x)=1$ and $\model_3(x)=\model_4(x)=0$. Fix the coalition $\coal=\{\model_1,\dots,\model_4\}$. 
  Clearly, $\charBinS{\coal} = \frac{1}{2}$. 
  Let $\coal_1 = \{\model_1,\model_2\}$ and $\coal_2 = \{\model_3,\model_4\}$, 
  and note that $\coal_1\cap \coal_2 = \emptyset$. 
	Further, let $ \{\model_3\} = \coal_3 \subseteq  \coal_4 =\{\model_1,\model_2,\model_3\}$. 
  We have:
	\begin{compactenum}
		\item[($i$)] $1-1=\charBinS{\coal_4 \cup \{\model_5\}} - \charBinS{\coal_4} < \charBinS{\coal_3 \cup \{\model_5\}} - \charBinS{\coal_3} = \frac{1}{2}-0\, ;$
		\item[($ii$)] $\frac{1}{2} = \charBinS{\coal}= \charBinS{\coal_1\cup\coal_2} < \charBinS{\coal_1} + \charBinS{\coal_2} = 1 + 0\,  ;$
		\item[($iii$)] $\frac{1}{2} = \charBinS{\coal} < \charBinS{\coal_1} = 1 \, .$
	\end{compactenum}
\end{proof}

\begin{corollary}\label{corol:nonConvex}
	The characteristic function game $G=(\modelSet, \charBin)$ is non-convex.
\end{corollary}
\begin{proof}[Proof of \Cref{corol:nonConvex}]
	It is a well-known result in cooperative game theory that 
  a game $G=(N, v)$ is convex if and only if 
  its characteristic function $v$ is supermodular~\citep{shapley1971cores}.
\end{proof}

\subsection{Proofs of~\Cref{subsec:binClass}}
\label{appendix:proofs-binary-case}

\begin{proof}[Proof of~\Cref{theo:Banz01}]
	Fix a test data point $\evz=(\xp,\y)\in\dataSet$ and, without loss of generality, let $\y=\y_1$.
	Expressing the summation in~\Cref{def:shapAndBanz} over the subsets of $\modelSet-i$,
  	the Banzhaf value of a classifier $i\in[\cardM]$ corresponds to 
	\[
	\banz_i(\game_{\evz})
	= \frac{1}{2^{\cardM-1}}\sum_{\coal\subseteq \modelSet-i}\left(\charBin(\coal+i)-\charBin(\coal)\right)
	= \E_{\coal}\left[\charBin(\coal+i)-\charBin(\coal)\right] \enspace ,
	\]
	where $\coal$ is a uniformly random subset of $\modelSet-i$. 
	That is, each classifier $j\in \modelSet-i$ is included in $\coal$ independently with probability $1/2$.
	Let $\corrPred(\coal) \doteq |\{j\in\coal: \model_j(\xp) = \y\}|$ and
  	$\wrongPred(\coal) \doteq |\{j\in\coal: \model_j(\xp) \neq \y\}|$
  	be the number of correct and of incorrect classifiers in $\coal$.
  	Recall $\corrPred$ and $\wrongPred$ corresponds to the number of correct and incorrect classifiers in $\modelSet$.
  	Furthermore, we have $\VotVecCoord_1(\coal) = \corrPred(\coal)$ and $\VotVecCoord_2(\coal) = \wrongPred(\coal)$.
	Note that the random variables $\corrPred(\coal)$ and $\wrongPred(\coal)$ are independent, since a model is either correct or not.
	Recall that we use the convention ${m \choose k} = 0$ for $k<0$ or $k>m$.

	\textbf{Correct classifier.} ($\model_i(\xp) = \y$)
	The set $\modelSet-i$ contains $\corrPred-1$ correct and $\wrongPred$ incorrect classifiers,
  	thus $\corrPred(\coal)\sim \Bin(\corrPred-1,1/2)$ and $\wrongPred(\coal) \sim \Bin(\wrongPred,1/2)$.
	By~\Cref{lemma:marginalBinaryUnweighted}, for $\coal\neq\emptyset$ the marginal gain of $i$ is $\frac{1}{2}$
  	if $\wrongPred(\coal) - \corrPred(\coal) \in \{0,1\}$ and it is $0$ otherwise, while for $\coal=\emptyset$ it is $1$.
	Since, conditioning on $\coal = \emptyset$ it holds $\wrongPred(\coal) - \corrPred(\coal) = 0$, we have that for every $\coal\subseteq\modelSet-i$,
	\[
	\charBin(\coal+i)-\charBin(\coal)
      = \frac{1}{2}\ind[\wrongPred(\coal) - \corrPred(\coal) \in \{0,1\}] + \frac{1}{2}\ind[\coal = \emptyset] \enspace.
	\]
	Thus, 
	\begin{equation}
	\label{eq:banzExpectationCorrect}
	\banz_i(\game_{\evz}) = \frac{1}{2}\Pr[\wrongPred(\coal) = \corrPred(\coal)]
      + \frac{1}{2}\Pr[\wrongPred(\coal) = \corrPred(\coal)+1] + \frac{1}{2}\Pr[\coal = \emptyset] \enspace,
	\end{equation}
	by taking the expectation over $\coal$.
	We now compute the three terms separately.
	\begin{align*}
	\Pr[\wrongPred(\coal) = \corrPred(\coal)]
      &= \sum_{k\ge 0}\Pr[\corrPred(\coal) = k]\Pr[\wrongPred(\coal) = k]
      = \sum_{k\ge 0}{\corrPred-1 \choose k}{\wrongPred \choose k}
        \left(\frac{1}{2}\right)^{\corrPred-1+\wrongPred}\\
      &= \frac{1}{2^{\cardM-1}}\sum_{k\ge 0}{\corrPred-1 \choose k}{\wrongPred \choose \wrongPred-k}
      = \frac{1}{2^{\cardM-1}}{\cardM-1 \choose \wrongPred} \enspace .
	\end{align*}
	The first equality uses the independence of $\corrPred(\coal)$ and $\wrongPred(\coal)$.
	The second one uses that both $\corrPred(\coal)$ and $\wrongPred(\coal)$ are binomial random variables with parameter $1/2$.
	The third equality follows by $\corrPred+\wrongPred=\cardM$ and the symmetry of the binomial coefficients, i.e., ${a \choose k} = {a \choose a-k}$.
	Finally, the last equality uses the Vandermonde's identity~(\Cref{lemma:vand}) with $a = \corrPred-1$, $b=\wrongPred$ and $c = \wrongPred$.
	The second term of~\Cref{eq:banzExpectationCorrect} is obtained analogously,
	\begin{align*}
	\Pr[\wrongPred(\coal) = \corrPred(\coal)+1]
      &= \sum_{k\ge 0}\Pr[\corrPred(\coal) = k]\Pr[\wrongPred(\coal) = k+1]\\
      &= \frac{1}{2^{\cardM-1}}\sum_{k\ge 0}{\corrPred-1 \choose k}{\wrongPred \choose k+1}\\
      &= \frac{1}{2^{\cardM-1}}\sum_{k\ge 0}{\corrPred-1 \choose k}{\wrongPred \choose \wrongPred-1-k}
      = \frac{1}{2^{\cardM-1}}{\cardM-1 \choose \wrongPred-1} \enspace ,
	\end{align*}
	where the last equality follows by Vandermonde's identity with $a=\corrPred-1$, $b=\wrongPred$ and $c=\wrongPred-1$.
	The third term of \Cref{eq:banzExpectationCorrect} corresponds to $\Pr[\coal=\emptyset] = 2^{-(\cardM-1)}$,
 	 since $\coal$ is a uniformly random subset of the $\cardM-1$ classifiers in $\modelSet-i$.
	Combining all three terms above with~\Cref{eq:banzExpectationCorrect} we have
	\[
	\banz_i(\game_{\evz})
      = \frac{1}{2^{\cardM}}\left({\cardM-1 \choose \wrongPred} + {\cardM-1 \choose \wrongPred-1}\right)
        + \frac{1}{2^{\cardM}}
      = \frac{1}{2^{\cardM}}\left({\cardM \choose \wrongPred} + 1\right) \enspace,
	\]
	where we used Pascal's rule~(\Cref{lemma:auxNcK}), which yields the first part of the claim.

	\textbf{Incorrect classifier.} ($\model_i(\xp) \neq \y$)
	The set $\modelSet-i$ contains $\corrPred$ correct and $\wrongPred-1$ incorrect classifiers, hence $\corrPred(\coal)\sim \Bin(\corrPred,1/2)$ and $\wrongPred(\coal)\sim\Bin(\wrongPred-1,1/2)$.
	By~\Cref{lemma:marginalBinaryUnweighted}, for $\coal \neq \emptyset$ the marginal gain of $i$ is $-\frac{1}{2}$ if $\corrPred(\coal) - \wrongPred(\coal) \in \{0,1\}$ and it is $0$ otherwise, while for $\coal = \emptyset$ it is $0$.
	Hence, for every $\coal \subseteq \modelSet-i$,
	\[
	\charBin(\coal+i)-\charBin(\coal)
      = -\frac{1}{2}\ind[\corrPred(\coal) - \wrongPred(\coal) \in \{0,1\}] + \frac{1}{2}\ind[\coal = \emptyset] \enspace ,
	\]
	where the last term accounts for the correction of the marginal gain over $\coal=\emptyset$.
	Thus
	\[
	\banz_i(\game_{\evz}) = -\frac{1}{2}\Pr[\corrPred(\coal) = \wrongPred(\coal)] - \frac{1}{2}\Pr[\corrPred(\coal) = \wrongPred(\coal)+1]
      + \frac{1}{2^{\cardM}} \enspace ,
	\]
	by taking the expectation over $\coal$.
	The two probabilities in the above sum follow from the computations of the previous case, by simply exchanging the roles of correct and of incorrect classifiers. 
	That is, it is sufficient to replace $\corrPred-1$ with $\wrongPred-1$ and $\wrongPred$ with $\corrPred$. Therefore
	\[
	\Pr[\corrPred(\coal) = \wrongPred(\coal)] = \frac{1}{2^{\cardM-1}}{\cardM-1 \choose \corrPred} \, ,
	\qquad
	\Pr[\corrPred(\coal) = \wrongPred(\coal)+1] = \frac{1}{2^{\cardM-1}}{\cardM-1 \choose \corrPred-1} \enspace .
	\]
	Combining the above we have
	\[
	\banz_i(\game_{\evz}) = -\frac{1}{2^{\cardM}}{\cardM \choose \corrPred} + \frac{1}{2^{\cardM}}
      = -\frac{1}{2^{\cardM}}\left({\cardM \choose \corrPred} - 1\right) \enspace ,
	\]
	by Pascal's rule~(\Cref{lemma:auxNcK}), which yields the claim and concludes the proof.
\end{proof}

\begin{proof}[Proof of~\Cref{theo:Shap01}]
	Similar to the proof of~\Cref{theo:Banz01} we prove the claim by distinguishing between correct and incorrect classifiers.
	
	\textbf{Correct classifier.} ($\model_i(\xp) =\y$). 
	Without loss of generality we fix $\y = \y_1$. Hence, $\corrPred =\VotVecCoord_1(\modelSet)$. From~\Cref{def:shapAndBanz} and~\Cref{lemma:marginalBinaryUnweighted} it follows that the coalitions $i\in \coal, \coal\subseteq\modelSet$ yielding a non-zero term for the difference $\charBin(\coal)-\charBin(\coal-i)$ in the computation of the Shapley value $\shap_i$ are such that,
	\begin{inparaenum}[(1)]
		\item $\coal-i = \emptyset$;
		\item $\VotVecCoord_1(\coal-i) = \VotVecCoord_2(\coal-i)$; and 
		\item $\VotVecCoord_1(\coal-i) = \VotVecCoord_2(\coal-i)-1$.
	\end{inparaenum}
	Clearly in case (1)  $\charBin(\{i\})-\charBin(\emptyset)=1$. Next, in cases (2) and (3) it holds that $\charBin(\coal)-\charBin(\coal-i) = \frac{1}{2}$, since $\model_i(\xp) = y$. Recall that we let $\corrPred = |\{\model \in \modelSet: \model(\xp)=y\}|$, and $\wrongPred = \cardM-\corrPred$.
	We can therefore express the Shapley values of~\Cref{eq:shap} as
	\begin{align*}
		\shap_i &= \frac{1}{n} +\frac{1}{2} \left(T_1 + T_2\right), %
\end{align*}
where
\begin{align*}
	T_1 &=\sum_{h=1}^{\min\{\corrPred-1, \wrongPred\}} \frac{(\cardM-(2h+1))!(2h)!}{\cardM!} {\corrPred-1 \choose h} {\wrongPred \choose h} 
\end{align*}
corresponds to the contribution for the Shapley value $\shap_i$ provided by coalitions of the type $\coal-i$ for which it holds $\VotVecCoord_1(\coal-i) = \VotVecCoord_2(\coal-i)$. Furthermore,
\begin{align*}
	T_2 &= \sum_{h=1}^{\min\{\corrPred, \wrongPred\}} \frac{(\cardM-2h)!(2h-1)!}{\cardM!} {\corrPred-1 \choose h-1} {\wrongPred \choose h}
\end{align*}
corresponds to the contribution for the Shapley value $\shap_i$ provided by coalitions $\coal-i$ such that $\VotVecCoord_1(\coal-i) = \VotVecCoord_2(\coal-i)-1$. Finally, the term $1/n$ accounts for $\coal-i=\emptyset$. Hence,
\begin{align*}
	T_1 + T_ 2 &= \frac{(\cardM-2\corrPred)!(2\corrPred-1)!}{\cardM!}{\wrongPred \choose \corrPred}\\
	&+ \sum_{h=1}^{\min\{\corrPred-1, \wrongPred\}} {\wrongPred \choose h} \frac{(\cardM-2h-1)!(2h-1)!}{n!}\left[
	(\cardM-2h) {\corrPred-1 \choose h-1} +2h {\corrPred-1 \choose h}  \right]
\end{align*}
By simple manipulations we obtain the first part of the claim
\begin{align*}
	(\cardM-2h) {\corrPred-1 \choose h-1} +2h {\corrPred-1 \choose h} &\stackrel{(a)}{=} (\cardM-2h) {\corrPred-1 \choose h-1} +2(\corrPred - h) {\corrPred-1 \choose h-1}\\
	&= {\corrPred-1 \choose h-1} (\cardM + 2\corrPred-4h)\, ,
\end{align*}
where in $(a)$ we used the known identity $k\binom{n}{k} = n\binom{n}{k-1}$.

\textbf{Incorrect classifier.} ($\model_i(\xp) \neq\y$). Following the analysis of the case above, by noticing the analogy in the argument for incorrect classifiers, with respect to correct classifiers, we can express 
\begin{align*}
	\shap_i &= -\frac{1}{2}\Bigg(
	\sum_{h=1}^{\min\{\wrongPred-1, \corrPred\}} \frac{(\cardM-(2h+1))!(2h)!}{\cardM!} {\wrongPred-1 \choose h} {\corrPred \choose h}\\ 
	&+  \sum_{h=1}^{\min\{\wrongPred, \corrPred\}} \frac{(\cardM-2h)!(2h-1)!}{\cardM!} {\wrongPred-1 \choose h-1} {\corrPred \choose h}
	\Bigg)\, .
\end{align*}
Following the same manipulations for the case of correct classifiers we obtain the claim.
\end{proof}

\begin{proof}[Proof of~\Cref{theo:runtimeBin}]
	First we note that by the \emph{symmetry} property, it is a trivial fact that $\shap_i =\shap_{i'}$ and $\banz_i=\banz_{i'}$ for any two classifiers $i\neq i'$ and $i,i'\in [\cardM]$ such that $\model_i = \model_{i'}$. Furthermore, recall that symmetry holds for both $\shap$ and $\banz$. 
	This property allows us to compute the closed-form solution of $\shap_i$ and $\banz_i$ only \emph{once} for each game $\game_{\evz}$, and assign such values to all classifier according to their predictions, as detailed below.
	
	\para{Banzhaf power index.} 
	The values of $\banz_i, i\in[\cardM]$ for all classifiers can be computed by evaluating the closed-form solution of~\Cref{theo:Banz01}, in $\bigO(1)$ time. 
	Furthermore, note that the value of $\banz_i(\game_{\evz}), i\in[\cardM]$ will be equal for all correct classifiers, and incorrect classifiers respectively over a fixed $(\xp,\y)\in \dataSet$ by symmetry. 
	Hence, we can compute the two terms of~\Cref{theo:Banz01}, a single time for each test data point. 
	Therefore, the complexity of obtaining $\banz_i(\game_{(\xp, \y)})$ is bounded by $\sum_{i=1}^\cardM \bigO(1)=\bigO(\cardM)$. 
	To obtain $\banz_i(\sumGame)$ we use the linearity property of the Banzhaf power index, performing $\bigO(|\dataSet|)$ computations, one on each test data point in $\dataSet$.  
	Yielding as final time complexity 
	\[
	\sum_{(\xp,\y) \in \dataSet} \sum_{i=1}^\cardM \bigO(1) = \bigO(|\dataSet|\cardM)\,  ,
	\]
	which yields the first part of the claim.
	
	\textbf{Shapley power index.} By~\Cref{theo:Shap01} the values of $\shap_i, i\in[\cardM]$ for all classifiers can be computed through the evaluation of the summation from~\Cref{eq:closed01Shap}.
	The summation can be evaluated once per each game $\game_{\evz}$ efficiently in linear time $\bigO(\cardM)$
  (see~\Cref{app:practicalImpl}). 
	Similar to the Banzhaf case above, using the \emph{symmetry} of Shapley values for a game $\game_{\evz}$ then all correct classifiers will have the same Shapley value. 
	Similarly, all incorrect classifiers will also attain the same Shapley value, hence we only need to compute the summation in~\Cref{eq:closed01Shap} at one time.
	That is we have to obtain $\shap_i$ for a classifier $i\in[\cardM]$ such that $\model_i(\xp) = \y$.
	Then we know, thanks to the \emph{efficiency} property, that
	\[
	\charBin([\cardM]) = \sum_{i\in[\cardM]} \shap_i = \sum_{i: \model_i(\xp)=\y} \shap_i + \sum_{i: \model_i(\xp)\neq\y} \shap_i\, .
	\]
	It follows that the value $\shap_i$, for any $i$ such that $\model_i(\xp)\neq\y$,
  can be computed by $(\charBin([\cardM])-\sum_{i: \model_i(\xp)=\y} \shap_i)/\wrongPred$, where $\wrongPred = |\{i:\model_i(\xp)\neq\y\}|$ provided that $\wrongPred>0$.
	Therefore, the final complexity of the algorithm is bounded by 
	\[
	\sum_{(\xp,\y) \in \dataSet} \left(\bigO(\cardM) + \sum_{i=1}^{\cardM} \bigO(1) \right) = \bigO(|\dataSet|\cardM)\, .
	\]
	
	Concerning the memory complexity, note that by symmetry, we have to compute the values of $\banz_i$ (or $\shap_i$) only once per vote profile $\VotVec \in \unique$. Storing the values of the power indices (for all vote profiles) requires $\bigO(|\unique|)$ memory and retaining the final output uses $\bigO(\cardM)$ memory. Since $|\unique|\le\cardM+1$ in the binary classification setting, the total memory is $\bigO(|\unique|+\cardM) = \bigO(\cardM)$, concluding the proof.
\end{proof}

\subsection{Proofs of~\Cref{subsubsec:MCBZ}}
\label{appendix:proofs-multiclass-case}

\begin{proof}[Proof of~\Cref{lemma:genFunction}]
	Consider the expansion of the $k$-th order coefficient in $P_h(\gfvar)$ for $k\in[1,\totLabs-1]$
	\begin{align*}
		[\gfvar^k]P_h(\gfvar) &= [\gfvar^k]\prod_{j=2}^\totLabs (A_j(h) + B_j(h) \gfvar) \\
		&= \sum_{\substack{I\subseteq\{2,\dots,\totLabs\} \\ |I|=k}} \left(\prod_{j\in I}B_j(h) \prod_{j\notin I} A_j(h)\right) \\
		&= \sum_{\substack{I\subseteq\{2,\dots,\totLabs\} \\ |I|=k}} \left(\prod_{j\in I}\binom{\cardM_j}{h} \prod_{j\notin I} \sum_{q=0}^{h-1}\binom{\cardM_j}{q}\right)\, .
	\end{align*}
	The first equality follows from the definition of $P_h(\gfvar)$, the second from the expansion of a product of binomial terms, and the third by the definitions of $A_j(h)$ and $B_j(h)$.
	
	We now show that $|\coalset(k,h)|$ equals the above term via a combinatorial argument. 
	By definition, $\coalset(k,h)$ contains subsets of classifiers $\coal \subseteq \modelSet'$ that exclude the correct label $\y_1$. 
	For any $\coal \in \coalset(k,h)$, exactly $k$ labels must have $h$ votes, 
	while all other labels $j \in \{2, \dots, \totLabs\}$ must have strictly fewer than $h$ votes.
	
	We can construct such coalitions by first selecting a subset of labels $I \subseteq \{2, \dots, \totLabs\}$ of size $|I|=k$ to attain $h$ votes. 
	For a fixed $I$, we calculate the number of choices $T_I$,
	\begin{itemize}
		\item \textbf{forcing $\VotVecCoord_j(\coal)=h$ for $j \in I$:} we must select exactly $h$ classifiers for each label in $I$. 
		There are $\prod_{j\in I} \binom{\cardM_j}{h}$ ways to do this;
		\item \textbf{forcing $\VotVecCoord_j(\coal) < h$ for $j \notin I$:} for each label $j \in \{2, \dots, \totLabs\} \setminus I$, we must select $q \in \{0, \dots, h-1\}$ classifiers. 
		Summing over possible values of $q$ for a single label $j$ gives $\sum_{q=0}^{h-1} \binom{\cardM_j}{q} = A_j(h)$. 
		Applying the product rule across all labels not in $I$ gives $\prod_{j\notin I} A_j(h)$.
	\end{itemize}
	Thus, $T_I = \prod_{j\in I}\binom{\cardM_j}{h} \prod_{j\notin I} A_j(h)$. 
	Summing $T_I$ over all possible subsets $I$ of size $k$ yields,
	\[
	|\coalset(k,h)| = \sum_{\substack{I\subseteq\{2,\dots,\totLabs\} \\ |I|=k}} \left(\prod_{j\in I}\binom{\cardM_j}{h} \prod_{j\notin I} \sum_{q=0}^{h-1}\binom{\cardM_j}{q}\right) = [\gfvar^k] P_h(\gfvar) \, ,
	\]
	which concludes the proof.
\end{proof}

\begin{proof}[Proof of~\Cref{cor:genFunctionMinusJ}]
	By the definition of $P_h(\gfvar)$, factoring out the term for label $\y_j$ yields
	\[
	\frac{P_h(\gfvar)}{A_j(h) + B_j(h)\gfvar} = \prod_{\substack{j'=2 \\ j' \neq j}}^\totLabs (A_{j'}(h) + B_{j'}(h)\gfvar) \,  .
	\]
	This expression is exactly the generating function from~\Cref{lemma:genFunction} applied to the ensemble $\modelSet' \setminus \{\model \in \modelSet : \model(\xp) = \y_j\}$ and the label set $\labelSet \setminus \{\y_1, \y_j\}$. 
	Since coalitions $\coal \in \coalset_{-j}(k,h)$ contain no classifiers with predictions for $\y_j$, it follows that $\VotVecCoord_j(\coal) = 0$. 
	Applying an identical combinatorial counting argument as in~\Cref{lemma:genFunction} to the new product above yields the claim.
\end{proof}

\begin{proof}[Proof of~\Cref{theo:BanzMC}]
	Let us first prove the claim for a correct classifier.
	
	\para{Correct classifier $(\model_i(\xp)=\y)$.} 
	Let $\coalset_1 =\{\coal\subseteq[\cardM]: i\in \coal, \coal-i \neq \emptyset, |\tieSet(\coal-i)|>1, \VotVecCoord_1(\coal-i)=\VotVecCoord_j(\coal-i), j\in\tieSet(\coal-i)\}$, and let
	$\coalset_2 =\{\coal\subseteq[\cardM]: i\in \coal, \coal-i \neq \emptyset, |\tieSet(\coal-i)|>0, \VotVecCoord_1(\coal-i)=\VotVecCoord_j(\coal-i)-1, j\in\tieSet(\coal-i)\}$. 
	Furthermore let $\coalset_1(k) = \{\coal\in\coalset_1: |\tieSet(\coal-i)| = k\}$ for $k\in[2,\totLabs]$ and $\coalset_2(k) = \{\coal\in\coalset_2: |\tieSet(\coal-i)| = k\}$ for $k\in[1,\totLabs-1]$. 
	By~\Cref{lemma:marginalMCUnweighted} we have that
	\begin{align}
		2^{\cardM-1}\banz_i &= \sum_{i\in\coal\subseteq[\cardM]} \charMC(\coal) - \charMC(\coal-i)\label{eqproofcorrMC:1}\\
		&= 1+\sum_{\coal \in \coalset_1} \frac{|\tieSet(\coal-i)|-1}{|\tieSet(\coal-i)|}  + \sum_{\coal\in\coalset_2} \frac{1}{|\tieSet(\coal-i)|+1}\label{eqproofcorrMC:2}\\
		&= 1+\sum_{k=2}^{\totLabs} \sum_{\coal\in \coalset_1(k)} \frac{k-1}{k} + \sum_{k=1}^{\totLabs-1} \sum_{\coal\in\coalset_2(k)} \frac{1}{k+1}\label{eqproofcorrMC:3}\\
		&= 1+\sum_{k=2}^{\totLabs} \sum_{h=1}^{\cardM_1 -1}  \frac{k-1}{k} {\cardM_1-1 \choose h} |\coalset(k-1,h)| + \sum_{k=1}^{\totLabs-1}\sum_{h=0}^{\cardM_1-1} \frac{1}{k+1} {\cardM_1-1 \choose h} |\coalset(k,h+1)|\label{eqproofcorrMC:4}\\
		&= 1+\sum_{h=1}^{\cardM_1-1}{\cardM_1-1 \choose h} \sum_{k=1}^{\totLabs-1} \frac{k}{k+1} |\coalset(k,h)| + \sum_{h=1}^{\cardM_1}{\cardM_1-1 \choose h-1}\sum_{k=1}^{\totLabs-1}\frac{1}{k+1}  |\coalset(k,h)|\label{eqproofcorrMC:5}\\
		&= 1+\sum_{h=1}^{\cardM_1-1}\sum_{k=1}^{\totLabs-1} \frac{|\coalset(k,h)|}{k+1} \left[k{\cardM_1-1 \choose h}  + {\cardM_1-1 \choose h-1}\right] +\sum_{k=1}^{\totLabs-1} \frac{1}{k+1}|\coalset(k,\cardM_1)|\label{eqproofcorrMC:6}\\
		&= 1+\sum_{h=1}^{\cardM_1-1}\sum_{k=1}^{\totLabs-1} \frac{[x^k]P_{h}(x)}{k+1} \left[k{\cardM_1-1 \choose h}  + {\cardM_1-1 \choose h-1}\right]+\sum_{k=1}^{\totLabs-1} \frac{[x^k]P_{\cardM_1}(x)}{k+1}\label{eqproofcorrMC:7}\\
		&=
		1+\sum_{h=1}^{\cardM_1}\sum_{k=1}^{\totLabs-1} \frac{[x^k]P_{h}(x)}{k+1} \left[k{\cardM_1-1 \choose h}  + {\cardM_1-1 \choose h-1}\right]\label{eqproofcorrMC:8}\,  .
	\end{align}

  In the above derivation, \Cref{eqproofcorrMC:1} holds is by definition of Banzhaf value and 
  \Cref{eqproofcorrMC:2} follows by~\Cref{lemma:marginalMCUnweighted}. 
  \Cref{eqproofcorrMC:3} holds by the definition of $\coalset_1(k)$ and $\coalset_2(k)$, 
  noting that they form a partition of $\coalset_1$ and $\coalset_2$, respectively. 
	\Cref{eqproofcorrMC:4} follows from a combinatorial selection to count sets from $\coalset_1(k)$ and $\coalset_2(k)$: we first choose the number of correct classifiers from $\{\model : \model(\xp) = \y\}$ and combine it with the definition of~$\coalset(k,h)$. 
	\Cref{eqproofcorrMC:5,eqproofcorrMC:6,,eqproofcorrMC:7} follow from simple manipulations, including re-indexing and carrying out boundary terms. 
	Finally, the claim in~\Cref{eqproofcorrMC:8} follows from the fact that $f(k, \cardM_1) = 1$ since $\binom{\cardM_1-1}{\cardM_1} = 0$. %
	
	We now prove the remaining part of the claim for an incorrect classifier.
	
  \medskip
	\para{Incorrect classifier ($\model_i(\xp) = \y_j, j\neq 1$).} We now let $\coalset_3 =\{\coal\subseteq[\cardM]: i\in \coal, \coal-i \neq \emptyset,  |\tieSet(\coal-i)|\ge 2, \VotVecCoord_1(\coal-i)=\VotVecCoord_j(\coal-i)=\VotVecCoord_{j'}(\coal-i), {j'}\in\tieSet(\coal-i)\}$,
	and $\coalset_4 =\{\coal\subseteq[\cardM]: i\in \coal, \coal-i \neq \emptyset, |\tieSet(\coal-i)|\ge1, \VotVecCoord_j(\coal-i)+1=\VotVecCoord_1(\coal-i)=\VotVecCoord_{j'}(\coal-i), j'\in\tieSet(\coal-i)\}$. 
	In analogy with the previous case we partition the two sets above as $\coalset_3(k) = \{\coal\in\coalset_3: |\tieSet(\coal-i)| = k\}$ for $k\in[2,\totLabs]$ and $\coalset_4(k) = \{\coal\in\coalset_4: |\tieSet(\coal-i)| = k\}$ for $k\in[1,\totLabs-1]$. We can compute the Banzhaf value as
	\begin{align}
		-2^{\cardM-1}\banz_i &= -\sum_{i\in\coal\subseteq[\cardM]} \charMC(\coal) - \charMC(\coal-i) \label{eqproofwrongMC:1}\\
		&= \sum_{\coal \in \coalset_3} \frac{1}{|\tieSet(\coal-i)|}  +  \sum_{\coal\in\coalset_4}\frac{1}{|\tieSet(\coal-i)|(|\tieSet(\coal-i)|+1)}\label{eqproofwrongMC:2}\\
		&= \sum_{k=2}^{\totLabs}\sum_{\coal \in \coalset_3(k)} \frac{1}{k}  + \sum_{k=1}^{\totLabs-1} \sum_{\coal\in\coalset_4(k)} \frac{1}{k(k+1)}\label{eqproofwrongMC:3}\\
		&=
		\sum_{k=2}^{\totLabs} \sum_{h=1}^{\cardM_j-1} \frac{1}{k} \binom{\cardM_j-1}{h}\binom{\cardM_1}{h} |\coalset_{-j}(k-2,h)| \nonumber\\
		&+
		\sum_{k=1}^{\totLabs-1}\sum_{h=0}^{\cardM_j-1} {\cardM_j-1 \choose h} {\cardM_1\choose h+1}|\coalset_{-j}(k-1, h+1)| \frac{1}{k(k+1)} \label{eqproofwrongMC:4}\\
		&=
		\sum_{h=1}^{\cardM_j} \sum_{k=0}^{\totLabs-2} \binom{\cardM_j-1}{h}\binom{\cardM_1}{h} \frac{|\coalset_{-j}(k,h)|}{k+2}  +
		\sum_{h=1}^{\cardM_j}\sum_{k=0}^{\totLabs-2}{\cardM_j-1 \choose h-1}  {\cardM_1\choose h}  \frac{|\coalset_{-j}(k, h)|}{(k+1)(k+2)} \label{eqproofwrongMC:5}\\
		&=
		\sum_{h=1}^{\cardM_j}\sum_{k=0}^{\totLabs-2}\frac{|\coalset_{-j}(k,h)|}{k+2}\binom{\cardM_1}{h} \left[\binom{\cardM_j-1}{h}+\frac{1}{k+1}\binom{\cardM_j-1}{h-1}\right]\label{eqproofwrongMC:6}\\
		&=
		\sum_{h=1}^{\cardM_j}\sum_{k=0}^{\totLabs-2}\frac{f_j(k, h)}{k+2}[\gfvar^k]\frac{P_h(\gfvar)}{A_j(h)+B_j(h)\gfvar}\label{eqproofwrongMC:7}\,  .
	\end{align}
	In the above derivations,~\Cref{eqproofwrongMC:1} follows by~\Cref{def:shapAndBanz}.
  \Cref{eqproofwrongMC:2} follows by~\Cref{lemma:marginalMCUnweighted}.
	\Cref{eqproofwrongMC:3} follows by noting that $\coalset_3(k)$ and $\coalset_4(k)$ form a partition of $\coalset_3$ and $\coalset_4$ respectively. 
  \Cref{eqproofwrongMC:4} follows by counting the sets in $\coalset_3(k)$ and $\coalset_4(k)$  with a combinatorial argument. 
	We now discuss how to build sets from $\coalset_3(k)$. That is, \begin{inparaenum}[($i$)]
		\item we first select the $h\ge 1$ classifiers voting for $\y_j$ for which we have $\binom{\cardM_j-1}{h}$ possible choice;
		\item we then select the number of correct classifiers to match the $h$ classifiers in $\coal-i$ selected, for which we have $\binom{\cardM_1}{h}$ possible choices 
		\item for any $k\ge 2$ we have to fix $k-2$ labels to tie the $h$ votes for $\y_j$ and $\y_1$, and all remaining labels (with possibly no votes) to have strictly less than $h$ votes. This corresponds exactly to the definition of $|\coalset_{-j}(k-2,h)|$.
	\end{inparaenum}
	Summing over all possible values of $h$ for fixed $k$ yields the first part of~\Cref{eqproofwrongMC:4}.
	The sets of $\coalset_4(k)$ follow an analogous construction yielding \Cref{eqproofwrongMC:4}.
	Finally, \Cref{eqproofwrongMC:5,eqproofwrongMC:6} follow by simple algebraic manipulations and summation re-indexing. 
  Finally,~\Cref{eqproofwrongMC:7} follows by the definition of $f_j(k,h)$.
  This concludes the proof.
\end{proof}

\begin{proof}[Proof of~\Cref{theo:runtimeMCBZ}]
	Consider two data points $d_1= (\xp_1,\y_1)$ and $d_2=(\xp_2,\y_2)$ with identical vote profiles. By the symmetry axiom, and as captured by our \Cref{theo:BanzMC}, the values of $\banz_i(\gameMC_{d_1})$ and $\banz_i(\gameMC_{d_2})$ depend strictly on the values $\cardM_j$ for $j\in [\totLabs]$. 
	That is, under two identical vote profiles $\VotVec^{d_1}= \VotVec^{d_2}$, it holds that $\banz_i(\gameMC_{d_1})=\banz_{i'}(\gameMC_{d_2})$ if either both $i$ and $i'$ vote for the correct label, or if $\model_i(\xp_1) = y_j$ and $\model_{i'}(\xp_2) = y_{j'}$ where $j,j'\neq 1$ and $\cardM_j=\cardM_{j'}$. 
	This observation allows us to compute the Banzhaf values exactly once per unique vote profile, caching the result and reusing it in $\bigO(1)$ time for any two points with the same vote profile.
	
	For a unique vote profile of $(\xp, \y)$, the algorithm first obtains $\cardM_j, j\in[\totLabs]$ in $\bigO(\cardM)$ time. 
	Then, the algorithm iterates over $h \in [\cardM_1]$. 
	Inside this loop, it computes the coefficients of the polynomial $P_h(\gfvar)$ in $\bigO(\totLabs^2)$ time. 
	The contribution to $\banz_i$ for any correct classifier provided by $P_h(\gfvar)$ is computed in $\bigO(\totLabs)$ time. 
	Next, for each incorrect class $j\in[2,\totLabs]$, the algorithm obtains the polynomial $Q_h^j(\gfvar)\doteq P_h(\gfvar)/(A_j(h) + B_j(h)\gfvar)$ in $\bigO(\totLabs)$ time. 
	Evaluating the contribution to $\banz_i$ for any classifier voting for $\y_j$ in $\bigO(\totLabs)$ time. 
	By caching $P_h(\gfvar)$ the algorithm can process all incorrect classes in $\sum_{j=2}^\totLabs \bigO(\totLabs)$ time.
	After processing each unique profile of $\unique$, we iterate through all classifiers correcting for the factor $2^{\cardM-1}$ as shown by \Cref{theo:BanzMC}, which takes $\bigO(\cardM)$ per vote profile. 
	Finally, the output corrects for the factor $|\dataSet|$, using the linearity property of power indices.
	
	The overall time complexity is bounded by
	\[
	\sum_{u=1}^{|\unique|} \left( \bigO(\cardM) + \sum_{h=1}^{\cardM_1} \left[ \bigO(\totLabs^2) + \bigO(\totLabs) + \sum_{j=2}^\totLabs \bigO(\totLabs) \right] \right) + \bigO(|\dataSet|\cardM) = \bigO(|\unique| \cardM \totLabs^2 + |\dataSet|\cardM) \, ,
	\]
	where the $\bigO(|\dataSet|\cardM)$ is required to evaluate the predictions of all classifiers of $\modelSet$. 
  If predictions are provided in input, e.g., through a preprocessing, then such term can be dropped.
	
	Finally, the algorithm requires no more than $\bigO(|\unique|\totLabs + \cardM + \totLabs \cardM) = \bigO(\totLabs (\cardM+|\unique|))$ auxiliary space. That is, 
	\begin{inparaenum}
		\item $\bigO(|\unique|\totLabs)$ is required to store the dictionaries of unique vote profiles and the Banzhaf values of their game;
		\item $\bigO(\cardM)$ is necessary to accumulate the global output for the ensemble;
		\item finally $\bigO(\totLabs)$ is required to store the coefficients of the polynomial $P_h(\gfvar)$ for any iteration $h\ge1$,
	\end{inparaenum} which concludes the proof.
\end{proof}

\subsection{Proofs of~\Cref{subsubsec:MCSHAP}}\label{appsubsec:shap}

\begin{lemma}\label{lemma:bivariateGenFunction}
	Let $(\xp,\y)\in \dataSet$ be a test data point. Without loss of generality, let $\y = \y_1 \in \labelSet$. For any $h \ge 1$ and $j\in[2,\totLabs]$, we define the two polynomials
	\[
	A_j(h; u) = \sum_{q=0}^{h-1} \binom{\cardM_j}{q} u^q \quad \text{and} \quad B_j(h; u) = \binom{\cardM_j}{h} u^h \, .
	\]
	Consider the ordinary generating function $P_h(\gfvar, u) = \prod_{j=2}^\totLabs (A_j(h; u) + \gfvar B_j(h; u))$. 
	
	Fix $h\in [\cardM]$, $k\in \{1, \dots, \totLabs-1\}$, and $m \ge 0$. Let $\coalset(k,h,m)$ be the set of coalitions $\coal \subseteq \modelSet'$ such that:
	\begin{inparaenum}[($i$)]
		\item $\modelSet' = \modelSet \setminus \{\model \in \modelSet : \model(\xp) = \y\}$;
		\item $|\tieSet(\coal)| = k$; 
		\item $\VotVecCoord_j(\coal) = h$ for each $j \in \tieSet(\coal)$; and
		\item the number of classifiers in $\coal$ is exactly $|\coal| = m$.
	\end{inparaenum}
	Then it holds that
	\[
	|\coalset(k,h,m)| = [\gfvar^k u^m]P_h(\gfvar, u) \, .
	\]
\end{lemma}
\begin{proof}
	We now expand $P_h(\gfvar, u)$. To do so we first extract the $k$-th coefficient in $\gfvar$. That is, analogously to the proof of~\Cref{lemma:genFunction} it holds that
	\begin{align*}
		[\gfvar^k u^m]P_h(\gfvar, u) &= [u^m] \sum_{\substack{I\subseteq\{2,\dots,\totLabs\} \\ |I|=k}} \left(\prod_{j\in I}B_j(h; u) \prod_{j\notin I} A_j(h; u)\right) \\
		&= [u^m] \sum_{\substack{I\subseteq\{2,\dots,\totLabs\} \\ |I|=k}} \left(\prod_{j\in I}\binom{\cardM_j}{h} u^h \prod_{j\notin I} \sum_{q_j=0}^{h-1}\binom{\cardM_j}{q_j} u^{q_j}\right)\\
		&= \sum_{\substack{I\subseteq\{2,\dots,\totLabs\} \\ |I|=k}} \left(\prod_{j\in I}\binom{\cardM_j}{h} [u^{m-kh}]  \prod_{j\notin I} \sum_{q_j=0}^{h-1}\binom{\cardM_j}{q_j} u^{q_j}\right)\\
		&= \sum_{\substack{I\subseteq\{2,\dots,\totLabs\} \\ |I|=k}} \left(\prod_{j\in I}\binom{\cardM_j}{h} \sum_{\substack{q_j \in \{0, \dots, h-1\} \text{ for all } j \notin I \\ \sum_{j \notin I} q_j = m-kh}} \prod_{j\notin I} \binom{\cardM_j}{q_j}\right) \,  .
	\end{align*}
	The first equality extracts the coefficient of $[\gfvar^k]$, and the second equality expands terms $A$ and $B$ by their definition. 
	The third equality follows by noting that $\prod_{j\in I} u^h = u^{kh}$. Therefore, the extraction of $[u^m]$ corresponds to extracting $[u^{m-kh}]$ from the polynomial $\prod_{j\notin I}\sum_{q_j=0}^{h-1}\binom{n_j}{q_j}u^{q_j}$.
	Finally, the last equality follows by the definition of polynomial multiplication (Cauchy product).
	We now show that $|\coalset(k,h,m)|$ corresponds to the above term by a combinatorial argument. 
	By definition, $\coalset(k,h,m)$ contains subsets $\coal \subseteq \modelSet'$, where classifiers voting for the correct label are excluded from $\modelSet'$. 
	For any $\coal \in \coalset(k,h,m)$, exactly $k$ labels must receive exactly $h$ votes, all other incorrect labels must receive strictly fewer than $h$ votes, and the total number of classifiers in $\coal$ must sum to $m$.
	
	We now construct a coalition from $\coalset(k,h,m)$ as follows. 
	We first select a subset of labels $I \subseteq \{2, \dots, \totLabs\}$ of size $|I|=k$ corresponding to the $k$ labels receiving exactly $h$ votes. 
	For a fixed $I$, we compute the number of sets that belong to $\coalset(k,h,m)$ as follows
	\begin{itemize}
		\item \textbf{forcing $\VotVecCoord_j(\coal)=h$ for $j \in I$:} we select exactly $h$ classifiers for each label in $I$. This yields $\prod_{j\in I} \binom{\cardM_j}{h}$ possible combinations. Note that, the selection of such classifiers collectively contributes to $kh$ classifiers in the final set $\coal \in\coalset(k,h,m)$.
		\item \textbf{forcing $\VotVecCoord_j(\coal) < h$ for $j \notin I$:} for each remaining incorrect label $j \in \{2, \dots, \totLabs\} \setminus I$, we must select $q_j \in \{0, \dots, h-1\}$ classifiers. 
		For a specific sequence of choices $(q_j)_{j \notin I}$, there are $\prod_{j\notin I} \binom{\cardM_j}{q_j}$ ways to select the classifiers.
		Note that, given the additional constrain of $|\coal|=m$, then it must hold that a valid selection yields $\sum_{j\notin I} q_j = m-\sum_{j\in I}k = m-kh$.
	\end{itemize}
	Therefore, by summing over all possible subsets $I$ we get
	\[
	|\coalset(k,h,m)| =  \sum_{\substack{I\subseteq\{2,\dots,\totLabs\} \\ |I|=k}} \left(\prod_{j\in I}\binom{\cardM_j}{h} \sum_{\substack{q_j \in \{0, \dots, h-1\} \text{ for all } j \notin I \\ \sum_{j \notin I} q_j = m-kh}} \prod_{j\notin I} \binom{\cardM_j}{q_j}\right)= [\gfvar^k u^m]  P_h(\gfvar, u) \, ,
	\]
	concluding the proof.
\end{proof}

\begin{corollary}\label{cor:bivariateGenFunctionMinusJ}
	Under the setting of~\Cref{lemma:bivariateGenFunction}, fix $j \in \{2, \dots, \totLabs\}$. 
	Let $\coalset_{-j}(k, h,m)$ be defined analogously to $\coalset(k,h,m)$ for coalitions $\coal \subseteq \modelSet' \setminus \{\model \in \modelSet : \model(\xp) = \y_j\}$. 
	That is, a coalition $\coal$ from $\coalset_{-j}(k, h,m)$ contains no classifier voting for class $\y_j$. Then it holds that
	\[
	|\coalset_{-j}(k,h,m)| = [\gfvar^ku^m]\frac{P_h(\gfvar,u)}{A_j(h;u) + B_j(h;u)\gfvar} \,  .
	\]
\end{corollary}
\begin{proof}
	The proof follows an analogous argument to that of~\Cref{cor:genFunctionMinusJ}.
	By the definition of $P_h(\gfvar,u)$, factoring out the term for label $\y_j, j\in[2,\totLabs]$ yields
	\[
	\frac{P_h(\gfvar,u)}{A_j(h;u) + B_j(h;u)\gfvar} = \prod_{\substack{j'=2 \\ j' \neq j}}^\totLabs (A_{j'}(h;u) + B_{j'}(h;u)\gfvar) \,  .
	\]
	This expression is exactly the generating function from~\Cref{lemma:bivariateGenFunction} applied to the ensemble $\modelSet' \setminus \{\model \in \modelSet : \model(\xp) = \y_j\}$ and the label set $\labelSet \setminus \{\y_1, \y_j\}$. 
	Since coalitions $\coal \in \coalset_{-j}(k,h,m)$ contain no classifiers with predictions for $\y_j$, it follows that $\VotVecCoord_j(\coal) = 0$. 
	Applying an identical combinatorial counting argument as in~\Cref{lemma:bivariateGenFunction} to the new product above yields the claim.
\end{proof}

\begin{proof}[Proof of~\Cref{theo:ShapleyMCCorrect}]
	Let us first prove the claim for a correct classifier.
	
	\para{Correct classifier $(\model_i(\xp)=\y)$.} 
	Similarly to the proof of~\Cref{theo:BanzMC}. Let $\coalset_1 =\{\coal\subseteq[\cardM]: i\in \coal, \coal-i \neq \emptyset, |\tieSet(\coal-i)|>1, \VotVecCoord_1(\coal-i)=\VotVecCoord_j(\coal-i), j\in\tieSet(\coal-i)\}$, and let
	$\coalset_2 =\{\coal\subseteq[\cardM]: i\in \coal, \coal-i \neq \emptyset, |\tieSet(\coal-i)|>0, \VotVecCoord_1(\coal-i)=\VotVecCoord_j(\coal-i)-1, j\in\tieSet(\coal-i)\}$. 
	We have that,
	\begin{align}
		\shap_i &= \sum_{i\in \coal \subseteq \modelSet} \eta(|\coal|)(\chF(\coal) - \chF(\coal - i))\label{eqproofcorrMCS:1}\\
		&= \frac{1}{\cardM} + \sum_{\coal \in \coalset_1} \eta(|\coal|)\frac{|\tieSet(\coal-i)|-1}{|\tieSet(\coal-i)|} + \sum_{\coal \in \coalset_2} \eta(|\coal|)\frac{1}{|\tieSet(\coal-i)|+1}\label{eqproofcorrMCS:2}\\
		&= \frac{1}{\cardM} + \sum_{k=2}^\totLabs \sum_{h=1}^{\cardM_1-1} \sum_{m=(k-1)h}^{\cardM-\cardM_1} \frac{k-1}{k}\eta(m+h+1) \binom{\cardM_1-1}{h} |\coalset(k-1,h,m)|\nonumber\\
		&\quad + \sum_{k=1}^{\totLabs-1} \sum_{h=0}^{\cardM_1-1} \sum_{m=k(h+1)}^{\cardM-\cardM_1} \frac{1}{k+1}\eta(m+h+1) \binom{\cardM_1-1}{h} |\coalset(k,h+1,m)|\label{eqproofcorrMCS:3}\\
		&= \frac{1}{\cardM} + \sum_{k=1}^{\totLabs-1} \sum_{h=1}^{\cardM_1-1} \sum_{m=kh}^{\cardM-\cardM_1} \frac{k}{k+1}\eta(m+h+1) \binom{\cardM_1-1}{h} |\coalset(k,h,m)|\nonumber\\
		&\quad + \sum_{k=1}^{\totLabs-1} \sum_{h=1}^{\cardM_1} \sum_{m=kh}^{\cardM-\cardM_1} \frac{1}{k+1}\eta(m+h) \binom{\cardM_1-1}{h-1} |\coalset(k,h,m)|\label{eqproofcorrMCS:4}\\
		&= \frac{1}{\cardM} + \sum_{k=1}^{\totLabs-1} \sum_{h=1}^{\cardM_1} \sum_{m=kh}^{\cardM-\cardM_1} \frac{1}{k+1} \left( k\eta(m+h+1)\binom{\cardM_1-1}{h} + \eta(m+h)\binom{\cardM_1-1}{h-1} \right) |\coalset(k,h,m)|\label{eqproofcorrMCS:5}\\
		&= \frac{1}{\cardM} + \sum_{k=1}^{\totLabs-1} \sum_{h=1}^{\cardM_1} \sum_{m=kh}^{\cardM-\cardM_1} f_1(k,h,m) [\gfvar^k u^m]P_h(\gfvar,u) \label{eqproofcorrMCS:6}\,  ,
	\end{align}
	where~\Cref{eqproofcorrMCS:1} follows by the definition of Shapley (\Cref{eq:shap}), 
  \Cref{eqproofcorrMCS:2} follow by~\Cref{lemma:marginalMCUnweighted}, and the definition of $\coalset_1$ and $\coalset_2$. 
  \Cref{eqproofcorrMCS:3} follows by summing over all sets of $\coalset_1$ and $\coalset_2$. 
	We show how to account for the summation over the sets in $\coalset_1$, the summation accounting for sets in $\coalset_2$ follows an analogous formulation. 
	That is,
	\begin{inparaenum}[($i$)]
		\item first we fix the number of classes in $\tieSet(\coal-i)$ that vote for an incorrect label (i.e., $k$);
		\item for each incorrect label we fix the cardinality of the classifiers voting for it (i.e., $h$)
		\item we then weight the number of sets $ \coal\in \coalset_1$ with $i\in \coal$ with the product rule, by first selecting $\binom{\cardM_1-1}{h}$ classifiers voting for the correct label, excluding $i$ and multiplying it by $|\coalset(k-1,h,m)|$. 
	\end{inparaenum}
	Therefore, this corresponds to sets $\coal\in \coalset_1$ where there are $h$ classifiers voting for the correct label, $k$ labels attaining $h$ votes over $\VotVec(\coal)$ and all remaining labels having strictly less than $h$ votes. Finally, by the definition of $|\coalset(k-1,h,m)|$ each set has exactly cardinality $(m+h+1)$. Combining all the above each set must be weighted $\eta(m+h+1) (k-1)/k$ to retain correctness.
	\Cref{eqproofcorrMCS:4} follows by re-indexing the summations. 
	\Cref{eqproofcorrMCS:5} follows by grouping the two summations.
	Finally, \Cref{eqproofcorrMCS:6} follows by the definition of $f_1(k,h,m)$ combined with~\Cref{lemma:bivariateGenFunction}.
	
	\para{Incorrect classifier ($\model_i(\xp) = \y_j, j\neq 1$).} We now let $\coalset_3 =\{\coal\subseteq[\cardM]: i\in \coal, \coal-i \neq \emptyset,  |\tieSet(\coal-i)|\ge 2, \VotVecCoord_1(\coal-i)=\VotVecCoord_j(\coal-i)=\VotVecCoord_{j'}(\coal-i), {j'}\in\tieSet(\coal-i)\}$,
	and $\coalset_4 =\{\coal\subseteq[\cardM]: i\in \coal, \coal-i \neq \emptyset, |\tieSet(\coal-i)|\ge1, \VotVecCoord_j(\coal-i)+1=\VotVecCoord_1(\coal-i)=\VotVecCoord_{j'}(\coal-i), j'\in\tieSet(\coal-i)\}$. 
	\begin{align}
		-\shap_i &= -\sum_{i\in \coal \subseteq \modelSet} \eta(|\coal|)(\chF(\coal) - \chF(\coal - i))\label{eqproofwrMCS:1}\\
		&= \sum_{\coal \in \coalset_3} \frac{\eta(|\coal|)}{|\tieSet(\coal-i)|}  +  \sum_{\coal\in\coalset_4}\frac{\eta(|\coal|)}{|\tieSet(\coal-i)|(|\tieSet(\coal-i)|+1)}\label{eqproofwrMCS:2}\\
		&= \sum_{k=2}^\totLabs \sum_{h=1}^{\cardM_j-1} \sum_{m=h(k-2)}^{\cardM-(\cardM_1+\cardM_j)} \frac{1}{k}\eta(m+2h+1) \binom{\cardM_1}{h}\binom{\cardM_j-1}{h} |\coalset_{-j}(k-2,h,m)|\nonumber\\
		&\quad + \sum_{k=1}^{\totLabs-1} \sum_{h=0}^{\cardM_j-1} \sum_{m=(k-1)(h+1)}^{\cardM-(\cardM_1+\cardM_j)} \frac{1}{k(k+1)}\eta(m+2(h+1)) \binom{\cardM_1}{h+1}\binom{\cardM_j-1}{h} |\coalset_{-j}(k-1,h+1,m)|\label{eqproofwrMCS:3}\\
		&= \sum_{k=0}^{\totLabs-2} \sum_{h=1}^{\cardM_j-1} \sum_{m=hk}^{\cardM-(\cardM_1+\cardM_j)} \frac{1}{k+2}\eta(m+2h+1) \binom{\cardM_1}{h}\binom{\cardM_j-1}{h} |\coalset_{-j}(k,h,m)|\nonumber\\
		&\quad + \sum_{k=0}^{\totLabs-2} \sum_{h=1}^{\cardM_j} \sum_{m=hk}^{\cardM-(\cardM_1+\cardM_j)} \frac{1}{(k+1)(k+2)}\eta(m+2h) \binom{\cardM_1}{h}\binom{\cardM_j-1}{h-1} |\coalset_{-j}(k,h,m)|\label{eqproofwrMCS:4}\\
		&= \sum_{k=0}^{\totLabs-2} \sum_{h=1}^{\cardM_j} \sum_{m=hk}^{\cardM-(\cardM_1+\cardM_j)} f_j(k,h,m) [\gfvar^ku^m]\frac{P_h(\gfvar,u)}{A_j(h;u)+B_j(h;u)\gfvar},\label{eqproofwrMCS:5}\\
	\end{align}
	where~\Cref{eqproofwrMCS:1} follows by~\Cref{eq:shap}. 
  \Cref{eqproofwrMCS:2} follows by~\Cref{lemma:marginalMCUnweighted} and the definitions of $\coalset_3$ and $\coalset_4$. \Cref{eqproofwrMCS:3} follows by a combinatorial selection of the sets in $\coalset_3$ and $\coalset_4$, with an analogous argument detailed for the case of correct classifiers, combined with the definition of $|\coalset_{-j}(k,h,m)|$.  \Cref{eqproofwrMCS:4} follows by re-indexing the summations. Finally, \Cref{eqproofwrMCS:5} follows by grouping the two summations noting that $\binom{\cardM_j-1}{\cardM_j} =0$, the definition of $f_j, j\in[2,\totLabs]$ and~\Cref{cor:bivariateGenFunctionMinusJ}.
\end{proof}

\begin{proof}[Proof of~\Cref{theo:runtimeMCSHAP}]
	As shown in the proof of~\Cref{theo:runtimeMCBZ} we only need to compute the values of $\shap_i$ for all classifiers $i\in[\cardM]$ yielding a unique vote profile $\VotVec\in\unique$, by symmetry.
	
	Fix a data point $(\xp,\y)$ the our~\Cref{alg:ShapMC} first evaluates the predictions of all classifiers in $\bigO(\cardM)$ time, and applies the mapping such that $\y=\y_1$ in $\bigO(\totLabs)$ time.
	Note that any label $\y_j$ with $\cardM_j=0$ can be ignored, since $A_j(h;u)=1$ and $B_j(h;u)=0$ for $h\ge1$, i.e., its factor in $P_h(\gfvar,u)$ is $1$. Thus, we assume $\cardM_j\ge1$ for $j\in[2,\totLabs]$.
	\Cref{alg:ShapMC} then iterates over $h\in[\cardM_1]$ initializing a matrix $P_h$ of size $\totLabs\times(\cardM-\cardM_1+1)$,  which contains the coefficient of the polynomial $P_h(\gfvar,u)$.
	All matrices $P_h, h\ge 1$ can be computed within a time complexity bounded by 
	\begin{align*}
		\sum_{h=1}^{\cardM_1}\sum_{j=2}^{\totLabs} \sum_{i=0}^{j-2} \sum_{m=0}^{\sum_{j'=2}^{j-1} \cardM_{j'}} \bigO(\min(\cardM_j, h)) 
		&\le \sum_{h=1}^{\cardM_1} \sum_{j=2}^{\totLabs} \bigO(\totLabs (\cardM-\cardM_1) \cardM_j)\\
		&= \bigO(\cardM_1 \totLabs (\cardM-\cardM_1)) \sum_{j=2}^{\totLabs} \cardM_j \\
		&= \bigO(\cardM_1 \totLabs (\cardM-\cardM_1)^2) = \bigO(\totLabs \cardM^3) \,  ,
	\end{align*}
	where the first step follows by the independence of the summations on $m$ and $i$, and by the facts that $j-2\le \totLabs$; $\sum_{j'=2}^{j-1}\cardM_{j'}$ $\le \sum_{j'=2}^{\totLabs} \cardM_{j'}= \cardM-\cardM_1$; and that $\min(a,b)\le a$. 
	The first equality follows by the linearity of the sum, and the second equality follows again by $\sum_{j'=2}^{\totLabs} \cardM_{j'}= \cardM-\cardM_1$. 
	Clearly the last step bounds $(\cardM-\cardM_1)\le \cardM$ and $\cardM_1\le \cardM$.
	For each matrix $P_h$ the algorithm computes the Shapley values of a correct classifier in $\bigO(\totLabs\cardM)$ time, using the result of~\Cref{theo:ShapleyMCCorrect}.
	For each incorrect label, the algorithm computes the coefficients of the polynomial $Q_j(\gfvar,u)\doteq P_h(\gfvar,u)/(A_j(h;u)+B_j(h;u)\gfvar)$ updating the matrix $P_h$ in place with a time complexity bounded by $\bigO(\totLabs(\cardM-\cardM_1)\cardM_j)$, it then computes the related Shapley value in $\bigO(\totLabs\cardM)=\bigO(\totLabs\cardM\cardM_j)$ time, and it then restores the matrix $P_h$ adding class $j\in[2,\totLabs]$ in $\bigO(\totLabs(\cardM-\cardM_1)\cardM_j)$ time. 
	Therefore, summing over $j\in[2,\totLabs]$ and using $\sum_{j=2}^{\totLabs}\cardM_j=\cardM-\cardM_1$, we obtain a time complexity bounded by $\bigO(\totLabs\cardM^2)$ time for each $h$, and $\bigO(\totLabs\cardM^3)$ over all $h\in[\cardM_1]$. 
	Therefore, the time complexity of~\Cref{alg:ShapMC} is bounded by $\bigO(|\unique|\cardM^3\totLabs+|\dataSet|\cardM)$.
	The memory complexity is bounded by $\bigO(|\unique|\totLabs+ \totLabs\cardM ) = \bigO(\totLabs(|\unique|+\cardM))$.
\end{proof}

\setcounter{algorithm}{1} %

\begin{algorithm2e}[t]
	\caption{\suiteBin: exact computation of Banzhaf and Shapley values for binary classifiers}
	\setcounter{AlgoLine}{0} %
	\PrintSemicolon          %
	\label{alg:mainAlg01}
	\KwIn{$\powerIdx\in \{\banz, \shap\}$; 
		ensemble $\modelSet=\{\model_1,\dots, \model_\cardM\}$;
		dataset $\dataSet$
	}
	\lFor{$i \in [\cardM]$}{
		$\powerIdx_i(\sumGame) \gets 0$\label{alglineSB01:initIdx}
	}
	\ForEach{$(\xp, \y) \in \dataSet$\label{alglineSB01:loopPoints}}{
		$\corrPred \gets \sum_{\model \in \modelSet} \ind[\model(\xp)=\y]$\label{alglineSB01:getCorrect}\,\;
		
		\For{$i \gets 1$ \KwTo $\cardM$\label{alglineSB01:loopModels}}{
			$\powerIdx_i(\sumGame) \gets \powerIdx_i(\sumGame)
			+ \getpsi(\cardM, \corrPred, \model_i(\xp), \y)$\label{alglineSB01:subroutine}\,\;
		}
	}
	\Return{
		$\left\{
		\frac{\powerIdx_i(\sumGame)}{|\dataSet|}
		\right\}_{i \in [\cardM]}$\label{alglineSB01:return}
	}
\end{algorithm2e}

\begin{algorithm2e}[t]
	\setcounter{AlgoLine}{0}   
	\SetEndCharOfAlgoLine{;}
	\caption{$\computeBANZ(n, \corrPred, \model_i(\xp), \y)$ and 
		$\computeSHAP(n, \corrPred, \model_i(\xp), \y)$}
	\label{alg:computeBANZ01}
	
	\KwIn{number of classifiers $\cardM$, number of correct predictions $\corrPred$, 
		prediction $\model_i(\xp)$, label $\y$}
	
	\SetKwProg{Fn}{Procedure}{}{end}
	
	\Fn{$\computeBANZ(n, \corrPred, \model_i(\xp), \y)$}{
		$\wrongPred \gets \cardM - \corrPred$\;
		
		\If{$\model_i(\xp) = \y$}{
			\Return $\left(\binom{\cardM}{\wrongPred}+1\right)2^{-\cardM}$\;
		}
		\Else{
			\Return $-\left(\binom{\cardM}{\corrPred}-1\right)2^{-\cardM}$\;
		}
	}
	
	\BlankLine
	\hrule
	\BlankLine
	
	\Fn{$\computeSHAP(n, \corrPred, \model_i(\xp), \y)$}{
		$a \gets 0$; $b \gets 0$; $\shap_i \gets 0$\;
		
		\If{$\model_i(\xp) = \y$}{
			$a\gets \dfrac{(\cardM-2\corrPred)!(2\corrPred-1)!}{\cardM!}
			\binom{\wrongPred}{\corrPred}$\;
			
			$b\gets \min\{\wrongPred, \corrPred-1\}$\;

			$\shap_i \gets \dfrac{1}{n}
			+ \dfrac{1}{2}\left(
			a + \sum_{h=1}^{b}
			\binom{\corrPred-1}{h-1}
			\binom{\wrongPred}{h}
			\dfrac{(\cardM-2h-1)!(2h-1)!}{\cardM!}
			(2\corrPred + \cardM - 4h)
			\right)$\;
		}
		\Else{
			$a\gets \dfrac{(\cardM-2\wrongPred)!(2\wrongPred-1)!}{\cardM!}
			\binom{\corrPred}{\wrongPred}$\;
			
			$b\gets \min\{\corrPred, \wrongPred-1\}$\;

			$\shap_i \gets -\dfrac{1}{2}\left(
			a + \sum_{h=1}^{b}
			\binom{\wrongPred-1}{h-1}
			\binom{\corrPred}{h}
			\dfrac{(\cardM-2h-1)!(2h-1)!}{\cardM!}
			(2\wrongPred + \cardM - 4h)
			\right)$\;
		}
		
		\Return $\shap_i$\;
	}
\end{algorithm2e}

\section{Pseudocode}
\label{app:missingCodes}

\subsection{Binary classification}

\para{Computation of Banzhaf and Shapley values.}
Pseudocode for our main algorithm, which is presented in~\Cref{subsec:binClass},
is shown in~\Cref{alg:mainAlg01}, 

We detail the missing subroutine \texttt{Get-}$\psi$ of~\Cref{alg:mainAlg01} in~\Cref{alg:computeBANZ01}. 
That is, \Cref{alg:computeBANZ01} simply computes the Banzhaf (when $\powerIdx=\banz$) or Shapley (when $\powerIdx=\shap$) according to~\Cref{theo:Banz01} and~\Cref{theo:Shap01} respectively.

We remark that in practice we do not execute the loop in~\Cref{alglineSB01:loopModels} of~\Cref{alg:mainAlg01}.
That is, we need to execute \texttt{Get-}$\psi$ exactly twice for any point $(\xp,\y)\in \dataSet$ of the dataset $\dataSet$: the first execution is to obtain the values of the power index on $\charBin(\xp)$ for all correct models and the second execution is to obtain the remaining values. 
That is, by symmetry all correct models will have the same power index on the game $\game = (\modelSet, \charBin(\xp))$, and similarly for all the incorrect models.
We discuss how to implement efficiently our method in~\Cref{app:practicalImpl}.

\subsection{Multiclass classification}

\para{Banzhaf computation over $\charMC$ (\suiteMC-$\banz$).} 
We detail the pseudocode of our approach that computes the values of $\banz_i(\sumGameMC)$ in~\Cref{alg:BanzMC}. 
The algorithm proceeds by processing each element independently in the loop from~\Cref{alglineBZMC:2}. It applies the relabeling discussed in~\Cref{subsec:multiclass} from \Cref{alglineBZMC:4} to~\Cref{alglineBZMC:5}, setting class 1 as the correct one and all the remaining classes as incorrect ones, we use $D[\cdot]$ to denote a dictionary data structure.
The main loop of the algorithm from~\Cref{alglineBZMC:8} simply computes the power indices as dictated by~\Cref{theo:BanzMC}.
Note that, while obtaining the coefficients of $P_h$, with our approach, requires $\bigO(\totLabs^2)$ (\Cref{alglineBZMC:10}), we compute the term $A_j(h)$ (\Cref{alglineBZMC:9}) using $A_{j-1}(h)$, i.e., avoiding to recompute the summation for each value of $h$.
The loop in~\Cref{alglineBZMC:12} computes the Banzhaf values for all incorrect models, according to the class they voted to.
We then update the Banzhaf values of each model in~\Cref{alglineBZMC:14} according to the class for which each model voted. 
The final output in~\Cref{alglineBZMC:return} reports the average of the values as captured by the definition of $\game_{\dataSet}$.

Note that our actual implementation does not compute the power indices $\banz_i$ all points in $\dataSet$ as we apply the symmetry property to process only the points in $\dataSet$ yileding \emph{unique} vote profiles as discussed in~\Cref{sec:prelim}. 
We detail how to process only the unique vote profiles in the pseudocode of \suiteMC-$\shap$ that we will discuss next.

\para{Shapley computation over $\charMC$ (\suiteMC-$\shap$).} \Cref{alg:ShapMC} displays the pseudocode of \suiteMC-$\shap$ computing the Shapley values for the multiclass case implementing our theoretical results of~\Cref{theo:ShapleyMCCorrect}.

The algorithm is similar to~\Cref{alg:BanzMC}, while slightly more complex due to the introduction of the bivariate polynomial from~\Cref{theo:ShapleyMCCorrect}. 
We now detail the main differences with respect to~\Cref{alg:BanzMC}. 
First, as mentioned previously, and as captured by our theoretical results we need to compute the Shapley values only for the points yileding unique vote profiles. To do so, we use a dictionary $J_{\text{{unique}}}$ (\Cref{alglineSHMC:5}) which caches the vote profiles seen so far and the Shapley values for the specific game (\Cref{alglineSHMC:missingDict3}). 
If the vote profile has been observed we can simply use the cached values, and update the global Shapley values of all points (\Cref{alglineSHMC:missingDict1} and \Cref{alglineSHMC:26}).
Otherwise the algorithm processes the vote profile, by implementing~\Cref{theo:ShapleyMCCorrect}.
In~\Cref{alglineSHMC:13} we compute $P_h$ with a dynamic programming approach storing the coefficients of the bivariate polynomial necessary to compute the Shapley values for all models. 

The dynamic programming approach is detailed in~\Cref{alg:computeBivariatePoly}.
Such approach computes a matrix
$P_h[k,m]$ containing all coefficients of the bivariate polynomial $P_h(\gfvar,u)$.
That is, each row $k\in[0,\totLabs-1]$ will correspond to the coefficients of the polynomial in $u$ after the extraction of the $k$-th coefficient in $\gfvar$ (i.e., $[\gfvar^k]P_h(\gfvar,u)$).
In other words, the entry $P_h[k,m]$ contains the counts of all coalitions of $\coalset(k,h,m)$.
The algorithm processes one incorrect class $j\in[2,\totLabs]$ at a time, computing the number of sets that such class contributes towards the coefficients of $P_h(\gfvar,u)$.
To detail the recursion used by the algorithm, for $j\in[\totLabs]$ let\footnote{We use the convention that a product $\prod_{i=a}^b x_i$ for $b<a$ is empty and evaluates to 1.}
\[
P^{(j)}_h(\gfvar,u) \doteq \prod_{j'=2}^{j}\left(A_{j'}(h;u) + B_{j'}(h;u)\gfvar\right) \enspace,
\]
be the partial product over the first $j-1$ incorrect classes (arbitrarily labelled). 
We define $P^{(1)}_h(\gfvar,u) = 1$ to correspond to the empty product. 
Note that $P^{(\totLabs)}_h(\gfvar,u) = P_h(\gfvar,u)$.
We also let $\cardM^{(j)} \doteq \sum_{j'=2}^{j}\cardM_{j'}$ be the number of classifiers voting for one of such classes.
The next lemma formalizes the recursion used in~\Cref{alg:computeBivariatePoly}, and proves its correctness.

\begin{lemma}\label{lemma:dpRecursion}
	Fix $h \ge 1$, for $j\in[2,\totLabs]$, $k\in[0,\totLabs-1]$ and $m\in[0,\cardM^{(j)}]$, let
	\begin{equation}\label{eq:dpRecursion}
	P^{(j)}_h[k,m] = \sum_{q=0}^{\min(\cardM_j, h-1,m)} \binom{\cardM_j}{q} P^{(j-1)}_h[k,m-q] + \ind[k\ge 1 \land m\ge h]\binom{\cardM_j}{h}P^{(j-1)}_h[k-1, m-h] \enspace,
	\end{equation}
	where $P^{(1)}_h[k,m] \doteq \ind[k = 0 \land m = 0]$, and $P^{(j)}_h[k,m] \doteq 0$ for $m > \cardM^{(j)}$, as $P^{(j)}_h(\gfvar,u)$ has degree at most $\cardM^{(j)}$ in $u$.
	Then, for each $j=1,\dots,\totLabs$, $k\in[0,\totLabs-1]$ and $m\in[0,\cardM^{(j)}]$ it holds that
	\[
	P^{(j)}_h[k,m] = [\gfvar^k u^m]P^{(j)}_h(\gfvar,u) \enspace.
	\]
\end{lemma}
\begin{proof}
	We proceed by induction on $j\ge 1$, i.e., the number of incorrect classes evaluated by the recursion.

	\textbf{Base case ($j=1$).} By definition $P^{(1)}_h(\gfvar,u)$ is the empty product, i.e., $P^{(1)}_h(\gfvar,u) = 1 = \gfvar^0u^0$. Therefore, for each $k\in[0,\totLabs-1]$ and $m\ge 0$ it holds that $[\gfvar^ku^m]P^{(1)}_h(\gfvar,u) = \ind[k=0 \land m=0] = P^{(1)}_h[k,m]$.

	\textbf{Inductive step.} 
	Fix $j\in[2,\totLabs]$ and assume that $P^{(j-1)}_h[k,m] = [\gfvar^ku^m]P^{(j-1)}_h(\gfvar,u)$ holds for each $k\in[0,\totLabs-1]$ and $m\ge 0$, where both sides equal $0$ for $m>\cardM^{(j-1)}$, since $P^{(j-1)}_h(\gfvar,u)$ has degree at most $\cardM^{(j-1)}$ in $u$.
	By definition of $P^{(j)}_h(\gfvar,u)$ we have $P^{(j)}_h(\gfvar,u) = P^{(j-1)}_h(\gfvar,u)\left(A_j(h;u) + B_j(h;u)\gfvar\right)$.
	Note that terms $A_j(h;u)$ and $B_j(h;u)$ do not depend on $\gfvar$, thus for $k\in[0,\totLabs-1]$ extracting the coefficient of $\gfvar^k$ yields
	\[
	[\gfvar^k]P^{(j)}_h(\gfvar,u) = A_j(h;u)\,[\gfvar^{k}]P^{(j-1)}_h(\gfvar,u) + B_j(h;u)\,[\gfvar^{k-1}]P^{(j-1)}_h(\gfvar,u) \enspace,
	\]
	where the second term is 0 when $k=0$.
	First note that each of the above terms is a polynomial in $u$.
	Thus, let $m\in[0,\cardM^{(j)}]$, we now extract the coefficient of $u^m$ from each of the two terms above.
	Considering the first term, we have
	\begin{align*}
	[u^m]\left(A_j(h;u)[\gfvar^{k}]P^{(j-1)}_h(\gfvar,u)\right) &= \sum_{q=0}^{h-1}\binom{\cardM_j}{q}[\gfvar^ku^{m-q}]P^{(j-1)}_h(\gfvar,u) \\
	&= \sum_{q=0}^{\min(\cardM_j,h-1,m)}\binom{\cardM_j}{q}P^{(j-1)}_h[k,m-q] \enspace,
	\end{align*}
	the first equality follows by the definition of $A_j(h;u)$ and the Cauchy product.
	The last equality uses the inductive hypothesis and ignores zero terms, 
	that is:
	\begin{inparaenum}[($a$)]
		\item terms with $q>m$, since $m-q<0$ and the polynomial $P^{(j-1)}_h(\gfvar,u)$ has only monomials of non-negative degree in $u$; and
		\item terms with $q > \cardM_j$, since $\binom{\cardM_j}{q}=0$.
	\end{inparaenum}
	For the second term, since $B_j(h;u) = \binom{\cardM_j}{h}u^h$, the multiplication by $u^h$ increases by exactly $h$ the degree in $u$ of each monomial of $P^{(j-1)}_h(\gfvar,u)$. Therefore, the extraction of $[u^m]$ corresponds to extracting $[u^{m-h}]$ from $P^{(j-1)}_h(\gfvar,u)$, hence
	\begin{align*}
	[u^m]\left(B_j(h;u)[\gfvar^{k-1}]P^{(j-1)}_h(\gfvar,u)\right) &= \binom{\cardM_j}{h}[\gfvar^{k-1}u^{m-h}]P^{(j-1)}_h(\gfvar,u) \\
	&= \binom{\cardM_j}{h}P^{(j-1)}_h[k-1,m-h] \enspace,
	\end{align*}
	using the inductive hypothesis.
	Note that the above term equals 0 for $m < h$, and for $\cardM_j < h$ since $\binom{\cardM_j}{h}=0$.
	Summing the two terms yields the right-hand side of~\Cref{eq:dpRecursion}, that is
	\[
	[\gfvar^ku^m]P^{(j)}_h(\gfvar,u) = P^{(j)}_h[k,m] \enspace,
	\]
	concluding the induction, and yielding the claim.
\end{proof}

Note that by the above result we have  $P^{(\totLabs)}_h[k,m] = |\coalset(k,h,m)|$, since $P^{(\totLabs)}_h(\gfvar,u) = P_h(\gfvar,u)$ and by~\Cref{lemma:bivariateGenFunction}.

\Cref{alg:add_class} implements~\Cref{eq:dpRecursion} in place, i.e., using a single matrix $P_h$, initialized to $P^{(1)}_h$ by~\Cref{alg:computeBivariatePoly}.
The term for $q=0$ in~\Cref{eq:dpRecursion} equals $P^{(j-1)}_h[k,m]$, i.e., the current value of the entry, which is maintained in place.
By~\Cref{eq:dpRecursion}, the value of an entry $[k,m]$ is obtained from the entries $[k,m-q]$ for $q\in[1,\min(\cardM_j,h-1)]$ and $[k-1,m-h]$, i.e., from entries in the same row with a smaller column, or in a smaller row.
Since the algorithm processes both rows and columns in decreasing order, these entries are processed after $[k,m]$, and thus each entry (when processed) holds its value for $P^{(j-1)}_h$ when read, preserving the correctness of our result.
Note that by reversing the recursion (e.g., see~\citep{Uno2012WMG}) we can remove any class in-place, as we do in \Cref{alg:remove_class}.
We update the Shapley values $\shap_i$ to be reported in output in~\Cref{alglineSHMC:28}, and finally output the average values over $\dataSet$ (\Cref{alglineSHMC:return}). 
That is upon the computation of the coefficients of the polynomial $P_h(\gfvar,u)$ our~\Cref{alg:ShapMC} simply implements the result from~\Cref{theo:ShapleyMCCorrect}. 

Note that in the loop computing the Shapley values for the incorrect classifiers (\Cref{alglineSHMC:18}), we modify in place the matrix $P_{h}[k,m]$, for a space efficient implementation, by safely removing classes as mentioned above using~\Cref{alg:remove_class}.

We discuss in~\Cref{app:practicalImpl} how to implement our methods without incurring in numerical issues due to the presence of large binomial and factorial terms.

\begin{algorithm2e}[t]
	\setcounter{AlgoLine}{0} 
	\PrintSemicolon          
	\SetEndCharOfAlgoLine{;}
	
	\caption{\suiteMC-$\banz$: exact computation of Banzhaf values for multi\-class classification}
	\label{alg:BanzMC}
	\KwIn{dataset $\dataSet$, ensemble $\modelSet=\{\model_1,\dots, \model_\cardM\}$, labels $\labelSet$.}
	
	$\banz_i \gets 0$ for $i\in[\cardM]$\label{alglineBZMC:1}\;
	
	\For{$(\xp, \y)\in \dataSet$\label{alglineBZMC:2}}{
		$\cardM_1 \gets \sum_{\model \in \modelSet} \ind[\model = \y]$\label{alglineBZMC:3}\;
		$D[\y] \gets 1$; $j\gets 2$\label{alglineBZMC:4}\;
		\For{$r\gets 1$ \KwTo $\totLabs$}{
			\If{$\y_{r}\neq\y$}{
				$\cardM_j\gets\sum_{\model \in \modelSet} \ind[\model = \y_{r}]$; $D[\y_r] \gets j$; $j\gets j+1$\label{alglineBZMC:5}\;
			}
		}   
		
		\lFor{$j\gets 1$ \KwTo $\totLabs$}{$\banz'_{j} \gets 0$\label{alglineBZMC:6}}
		
		$\banz'_1\gets 1$\label{alglineBZMC:7}\;
		
		\For{$h \gets 1$ \KwTo $\cardM_1$\label{alglineBZMC:8}}{
			
			\lFor{$j\gets 2$ \KwTo $\totLabs$}{$A_j(h) \gets \sum_{q=0}^{h-1}{\cardM_j \choose q}$, $B_j(h) \gets {\cardM_j \choose h}$\label{alglineBZMC:9}}   
			
			$P_h(\gfvar) \gets \prod_{j=2}^{\totLabs}(A_j(h) + B_j(h)\gfvar)$\label{alglineBZMC:10}\;
			
			$\banz'_{1} \gets \banz'_{1} + \sum_{k=1}^{\totLabs-1}  \frac{1}{k+1} \left[k{\cardM_1-1 \choose h}  + {\cardM_1-1 \choose h-1}\right] [\gfvar^k]P_{h}(\gfvar)$\label{alglineBZMC:11}\;
			
			\For{$j \gets 2$ \KwTo $\totLabs$\label{alglineBZMC:12}}{
				$\banz'_{j} \gets \banz'_{j} - \sum_{k=0}^{\totLabs-2}\frac{1}{k+2}\binom{\cardM_1}{h}\left[{\cardM_j-1 \choose h}  + \frac{1}{k+1}{\cardM_j-1 \choose h-1}\right] [\gfvar^k]\frac{P_h(\gfvar)}{A_j(h) + B_j(h)\gfvar}$\label{alglineBZMC:13}\;
			}
		}
		
		\For{$i\gets 1$ \KwTo $\cardM$\label{alglineBZMC:14}}{
			$\y_\model \gets \model_i(\xp)$\;
			$j \gets D[\y_\model]$\;
			$\banz_i \gets \banz_i + \banz'_j/2^{\cardM-1}$\label{alglineBZMC:15}\;
		}
	}
	
	\Return{ $\left\{ \banz_i/|\dataSet| \right\}_{i \in [\cardM]}$ }\label{alglineBZMC:return}\;
\end{algorithm2e}

\begin{algorithm2e}[t]
	\setcounter{AlgoLine}{0} 
	\DontPrintSemicolon          
	\SetEndCharOfAlgoLine{;}
	
	\caption{\suiteMC-$\shap$: exact computation of Shapley values for multi\-class classification}
	\label{alg:ShapMC}
	\KwIn{dataset $\dataSet$, ensemble $\modelSet=\{\model_1,\dots, \model_\cardM\}$, labels $\labelSet$.}
	
	$\shap_i \gets 0$ for $i\in[\cardM]$\label{alglineSHMC:1}\;
	
	$J_{\text{{unique}}} \gets$ empty dictionary\label{alglineSHMC:5}\;
	\For{$(\xp, \y)\in \dataSet$\label{alglineSHMC:2}}{
		$\cardM_1 \gets \sum_{\model \in \modelSet} \ind[\model = \y]$\label{alglineSHMC:3}\;
		$D[\y] \gets 1$; $j\gets 2$\label{alglineSHMC:4}\;
		$\VotVec \gets (\cardM_1,0,\dots,0)\in \mathbb{R}^\totLabs$\label{alglineSHMC:missingDict}\;
		\For{$r\gets 1$ \KwTo $\totLabs$\label{alglineSHMC:6}}{
			\If{$\y_{r}\neq\y$\label{alglineSHMC:7}}{
				$\cardM_j\gets\sum_{\model \in \modelSet} \ind[\model = \y_{r}]$\label{alglineSHMC:8}; $\VotVecCoord_{j} \gets \cardM_j$; $D[\y_r] \gets j$; $j\gets j+1$\label{alglineSHMC:9}\;
			}
		}
		\If{$\VotVec\in J_{\text{{unique}}}$\label{alglineSHMC:missingDict1}}{
			$(\shap_1',\dots,\shap_\totLabs') \gets J_{\text{{unique}}}[\VotVec]$ \label{alglineSHMC:missingDict2}\;
			}
		\Else{
		
		\lFor{$j\gets 1$ \KwTo $\totLabs$}{$\shap'_{j} \gets 0$\label{alglineSHMC:10}}
		
		$\shap'_1\gets 1/\cardM$\label{alglineSHMC:11}\;
		
		\For{$h \gets 1$ \KwTo $\cardM_1$\label{alglineSHMC:12}}{
			
			$P_h \gets $\texttt{ComputeCoeffsDP}($\cardM_j $ for $j\in[2,\totLabs], h$)\label{alglineSHMC:13}\;
			
			\For{$k\gets 1$ \KwTo $\totLabs-1$\label{alglineSHMC:14}}
			{
				\For{$m\gets kh$ \KwTo $\cardM-\cardM_1$\label{alglineSHMC:15}}
				{
					$\shap_1'\gets \shap_1' + f_1(k, h, m) P_h[k,m]$\label{alglineSHMC:16}\;
				}
			}
			
			\For{$\cardM_j \in$ $J_{\text{\emph{unique}}}$.\emph{keys()}\label{alglineSHMC:17}}
			{
				$P_h \gets$ \texttt{RemoveClass}($P_h, \cardM_j, h$)\label{alglineSHMC:18}\;
				
				$\shap_{\text{curr}} \gets 0$\label{alglineSHMC:19}\; 
				\For{$k\gets 0$ \KwTo $\totLabs-2$\label{alglineSHMC:20}}
				{
					\For{$m\gets kh$ \KwTo $\cardM-(\cardM_1+\cardM_j)$\label{alglineSHMC:21}}
					{
						$\shap_{\text{curr}}\gets \shap_{\text{curr}} + f_j(k, h, m) P_h[k,m]$\label{alglineSHMC:22}\;
					}
				}
				
				$P_h \gets$ \texttt{AddClass}($P_h, \cardM_j, h, \totLabs-2, \sum_{j'\ge 2} \cardM_{j'} - \cardM_j$)\label{alglineSHMC:25}\;
			}
			$J_{\text{{unique}}}[\VotVec] \gets (\shap_1',\dots, \shap_\totLabs')$\label{alglineSHMC:missingDict3}\;
		}}
		
		\For{$i\gets 1$ \KwTo $\cardM$\label{alglineSHMC:26}}{
			$\y_\model \gets \model_i(\xp)$\;
			$j \gets D[\y_\model]$\label{alglineSHMC:27}\;
			$\shap_i \gets \shap_i + \shap'_j$\label{alglineSHMC:28}\;
		}
	}
	
	\Return{ $\left\{ \shap_i/|\dataSet| \right\}_{i \in [\cardM]}$ }\label{alglineSHMC:return}\;
\end{algorithm2e}

\begin{algorithm2e}[tb]
	\setcounter{AlgoLine}{0}   
	\SetEndCharOfAlgoLine{;}
	\caption{\texttt{ComputeCoeffsDP}($n_j$ for $j\in[2,\totLabs], h$)}
	\label{alg:computeBivariatePoly}
	\KwIn{Number of classifiers $\cardM_j$ for each incorrect class $j\in[2, \totLabs]$, threshold $h$}
	$\cardM' \gets 0$;
	$P_h \gets \bm{0}\in \mathbb{R}^{\totLabs\times(\sum_{j\ge 2}\cardM_j+1)}$\;
	$P_h[0,0]\gets 1$\; 
	\For{$j\gets 2$ \KwTo $\totLabs$}
	{
		$P_h \gets $ \texttt{AddClass}($P_h, \cardM_{j}, h, j-2, n'$)\;
		$\cardM' \gets \cardM' + \cardM_j$\;
	}
	\KwRet{$P_h$}\label{alglineBivatUp:red}\;
\end{algorithm2e}

\begin{algorithm2e}[tb]
	\setcounter{AlgoLine}{0} 
	\DontPrintSemicolon          
	\SetEndCharOfAlgoLine{;}
	
	\caption{\texttt{AddClass}($P_h, \cardM_j, h, u, m'$)}
	\label{alg:add_class}
	\KwIn{Matrix $P_h$, counts $\cardM_j$, threshold $h$, max current label ties $u$, max incorrect votes $m'$}
	
	$q_{\max} \gets \min(\cardM_j, h-1)$\;
	
	\For{$k \gets u$ \textbf{\emph{downto}} $0$} 
	{
		\For{$m \gets m'$ \textbf{\emph{downto}} $0$}
		{
			\If{$P_h[k, m] > 0$}
			{
				\If{$\cardM_j \ge h$}
				{
					$P_h[k+1, m+h] \gets P_h[k+1, m+h] + P_h[k, m] \cdot \binom{\cardM_j}{h}$;\;
				}
				
				\For{$q \gets q_{\max}$ \textbf{\emph{downto}} $1$}
				{
					$P_h[k, m+q] \gets P_h[k, m+q] + P_h[k, m] \cdot \binom{\cardM_j}{q}$;\;
				}
			}
		}
	}
	\Return{$P_h$};\;
\end{algorithm2e}

\begin{algorithm2e}[tb]
	\setcounter{AlgoLine}{0} 
	\DontPrintSemicolon          
	\SetEndCharOfAlgoLine{;}
	
	\caption{\texttt{RemoveClass}($P_h, n_j, h$)}
	\label{alg:remove_class}
	\KwIn{Matrix $P_h$, classifier count $\cardM_j$, threshold $h$}
	
	$q_{\max} \gets \min(\cardM_j, h-1)$\;
	$\cardM' \gets \sum_{i=2}^\totLabs \cardM_i$\;
	
	\For{$k \gets 0$ \KwTo $\totLabs - 2$} 
	{
		\For{$m \gets 0$ \KwTo $\cardM'$}
		{
			\If{$\cardM_j \ge h$ \textbf{\emph{and}} $k \ge 1$  \textbf{\emph{and}} $m \ge h$}
			{
				$P_h[k, m] \gets P_h[k, m]  - P_h[k-1, m-h] \cdot \binom{\cardM_j}{h}$;\;
			}
			
			\For{$q \gets 1$ \KwTo $q_{\max}$}
			{
				\If{$m \ge q$}
				{
					$P_h[k, m] \gets P_h[k, m] - P_h[k, m-q] \cdot \binom{\cardM_j}{q}$;\;
				}
			}
		}
	}
	\Return{$P_h$};\;
\end{algorithm2e}

\para{Monte-Carlo sampling baseline.} We present our baseline $\texttt{MC}$ in~\Cref{alg:mc}. The algorithm takes in input a sample size $s$, the dataset $\dataSet$ and the ensemble $\modelSet$. 
The algorithm process each model $i\in[\cardM]$ independently by sampling $s$ coalitions $\coal$ such that $\mathbb{P}[\coal] = 1/2^{\cardM-1}$ for the Banzhaf value and $\mathbb{P}[\coal] = \frac{(\cardM-|\coal|)!(|\coal|-1)!}{\cardM!}$ for the Shapley index. 
That is, such probability accounts for the weight of each marginal gain $\chF(\coal+i) - \chF(\coal)$ in the respective power index, that is, the estimate is unbiased.
The output averages across samples, and over the entire dataset $\dataSet$ the marginal contribution of $\chF(\coal+i) - \chF(\coal)$ for all the $s$ samples over each point.
Since evaluating $\chF(\coal)$ on a point requires computing all the votes of the sampled coalition with size bounded by $\cardM$, the final complexity is $\bigO(s\cardM^2|\dataSet|)$.
We remark that more advanced and sample efficient techniques exists, e.g.,~\citep{mitchell2022sampling}. 
However, our baseline serves to show that a simple sampling algorithm is significantly more inaccurate and inefficient compared to our algorithms in \suite.

\begin{algorithm2e}[tb]
	\setcounter{AlgoLine}{0} 
	\SetEndCharOfAlgoLine{;}
	
	\caption{Monte-Carlo estimation of $\psi\in\{\shap,\banz\}$}
	\label{alg:mc}
	\KwIn{Dataset $\dataSet$, ensemble $\modelSet$, sample size $s$.}
	$\hat{\psi}_i \gets 0, i \in[\cardM]$\;
	\For{$i\gets 1$ \KwTo $\cardM$}
	{
		\For{$k\gets 1$ \KwTo $s$}
		{
			$\coal\gets \emptyset$ \;
			\If{$\psi = \banz$}
			{
				$\coal \gets \{j \neq i : \mathtt{FairCoin}_j = \mathtt{H}\}$.
			}
			\Else
			{
				$\sigma\gets \mathtt{RandomPermutation}(\cardM)$\;
				$\coal\gets \{j: \sigma_j<\sigma_i \}$\;
			}
			\lFor{$(\xp,\y)\in \dataSet$}{$\hat{\psi}_i \gets \hat{\psi}_i + (\chF(\coal+i) - \chF(\coal))$}
		}
		\KwRet $\left\{\hat{\psi}_i/(s|\dataSet|)\right\}_{i\in[\cardM]}$\;
	}\end{algorithm2e}

\section{Efficient and numerically stable implementation}\label{app:practicalImpl}

In this section we briefly details about our practical implementation of \suite.
We highlight how we addressed specific issues regarding numerical stability and efficiency.

\textbf{Banzhaf power index ($\suite_{\mathsf{B}}$).} For simplicity fix $i\in [\cardM]$ to be a correct classifier, an analogous discussion applies to incorrect classifiers. From~\Cref{theo:Banz01}  we can express the Banzhaf index as ${\cardM \choose \wrongPred}/2^\cardM + (\frac{1}{2}^\cardM)$. 
Therefore, computing the Banzhaf index evaluating each individual term will yield to numerical issues, given by the large values attained by $2^\cardM$ and by ${\cardM \choose \wrongPred}$.
However, we observe that ${\cardM \choose \wrongPred}/2^\cardM$ corresponds to the probability that a Binomial random variable $\mathrm{Bin}(\cardM,\frac{1}{2})$ equals $\wrongPred$. Therefore, we employ such simple results to directly compute the ratio ${\cardM \choose \wrongPred}/2^\cardM$, and we account for the missing term in $\banz_i$ using a numerically stable implementation to evaluate the value $\frac{1}{2}^\cardM$. %

\textbf{Shapley power index ($\suite_{\mathsf{B}}$).} The computation of~\Cref{eq:closed01Shap} by its definition has several drawbacks. 
First it is not numerically stable, due to the large values of the terms involved in the summation, leading to numerical issues. 
Second, there are several repeated calculations, involving factorial and binomial terms. 
To address the above issues we evaluate $\shap_i, i\in[\cardM]$ through a recursive formula.
That is, if we set $w_0 = (\cardM+2\cardM_{\CorW}-4)$, $T_0 = \frac{\cardM_{\WorC}}{\cardM (\cardM-1)(\cardM-2)}$ and 
\[
w_j = w_{j-1} -4 \quad \text{ and } \quad T_h = \frac{2 (\cardM_{\CorW}-h) (\cardM_{\WorC}-h)(2h + 1)}{(h+1)(\cardM-2h -1 )(\cardM-2h -2)} T_{h-1} \, ,
\]
for $h\ge 1$, then we can express $\shap_i$ with the summation
\[
\shap_i = \sum_{h\ge0} T_h w_h \, .
\]
The above expression allows us to compute the Shapley values from~\Cref{eq:closed01Shap} efficiently in linear time $\bigO(\cardM)$. 
That is, we use a constant number of operations to evaluate each term $T_hw_h$. 
Moreover, the above summation is numerically stable since it avoids evaluating factorials and binomials.

\textbf{Banzhaf and Shapley power indices for the multiclass case ($\suite_{\mathsf{MC}}$).}
We note that a naive implementation of~\Cref{alg:BanzMC,alg:ShapMC} would incur in numerical issues (such as integer overflows) due to the computation of large binomial and factorials.
To fix such issue our C++ implementation represents the coefficients of the dynamic programming matrices as exact integers.
Specifically, for each vote profile we use the smallest fixed-width unsigned integer to represent the numbers used in the evaluations of \Cref{theo:BanzMC} and \Cref{theo:ShapleyMCCorrect}, ranging the length from $128$ up to $1024$ bits, which supports ensembles with $\cardM \le 1023$.
We then only evaluate in floating point the final weighted sum corresponding to the power indices.
Note that long integers increase the cost of arithmetical operations such as multiplications, which is however necessary to enable the practical usability of our methods, i.e., retaining the accuracy of the exact computation of our methods.
As a concrete example, the above discussion explains the high runtime of \suite-$\shap$ on the \texttt{CT} dataset and the AdaBoost ensemble from~\Cref{tab:rtMC}.

\begin{table}[t]
	\centering
	\caption{Datasets used for binary and multi\-class classification settings in~\Cref{sec:exps}. 
		Classes ($\totLabs$), denotes the original number of labels, which are also retained for the multi\-class experiments. 
		The last two columns display the number of points in $\dataSetVal\cup\dataSetTra$ 
		with label $y=1$ and $y=0$ after binarization.}
	\label{tab:tabular_datasets}
	\resizebox{0.8\textwidth}{!}{
		\begin{tabular}{l rrrr}
			\toprule
			{Dataset (Label)} & $|\dataSetVal|+|\dataSetTra|$ & {Classes} ($\totLabs$) & $y=1$ & $y=0$ \\
			\midrule
			Blood Transfusion (\texttt{BTSC}) & 187 & 2 & 143 & 44 \\
			Credit-G (\texttt{CR}) & 250 & 2 & 179 & 71 \\
			CNAE-9 (\texttt{C9}) & 270 & 9 & 38 & 232 \\
			Titanic (\texttt{TT}) & 328 & 2 & 210 & 118 \\
			Bank Marketing (\texttt{BM}) & 11\,303 & 2 & 9\,994 & 1\,309 \\
			Covertype (\texttt{CT}) & 145\,253 & 7 & 70\,628 & 74\,625\\
			\bottomrule
		\end{tabular}
	}
\end{table}

\begin{table}[t]
\centering
\caption{Dataset statistics for the zero-shot multiclass classification experiment. For each dataset we display its total number of images, and the number of possible classes. We also show $|\dataSet_{\text{val}}|$, used to compute the ranking metrics, and $|\dataSet_{\text{test}}|$ used to evaluate the generalization of model selection. Column $|U|$ reports the  number of unique vote profiles observed on $\dataSet_{\text{val}}$ (the ensemble has $\cardM=22$ base predictors), averaged over five splits. Other columns do not depend on the split.}
\label{tab:clip_datasets}
\resizebox{0.8\textwidth}{!}{\begin{tabular}{lccccc}
\toprule
Dataset & Total images & Classes (\totLabs) & $|\dataSet_{\text{val}}|$ & $|\dataSet_{\text{test}}|$ & $|U|$ \\
\midrule
CIFAR (Cifar-100) & 10\,000 & 100 & 5\,000 & 5\,000 & $675.2 \pm 6.6$ \\
PETS (Oxford-IIIT Pets) & 3\,669 & 37 & 1\,834 & 1\,835 & $123.6 \pm 5.4$ \\
IMAGENET-A (ImageNet-A) & 7\,500 & 200 & 3\,750 & 3\,750 & $861.0 \pm 11.6$ \\
\bottomrule
\end{tabular}
}
\end{table}

\begin{table}[t]
\centering
\caption{Class-imbalance statistics for the multiclass datasets of~\Cref{subsec:CLIPexp} (computed on $\dataSet_{\text{val}}\cup\dataSet_{\text{test}}$, hence equal on every split). $\totLabs$ is the number of classes. ``$\min$/$\max$'' are the smallest/largest number of points (images) with the same class; $IR =\max/\min$ is the imbalance ratio; $CV$ is the coefficient of variation computed over the cardinality of all classes; and $H_{\text{norm}}\in(0,1]$ is the normalized Shannon entropy of the class distribution ($1$ corresponds to a uniform distribution).}
\label{tab:clip_imbalance}
\resizebox{0.8\textwidth}{!}{\begin{tabular}{lcccccc}
\toprule
Dataset & Classes ($\totLabs$) & $\min$ & $\max$ & $IR$ & CV (\%) & $H_{\text{norm}}$ \\
\midrule
CIFAR (Cifar-100) & 100 & 100 & 100 & 1.00 & 0.0 & 1.000 \\
PETS (Oxford-IIIT Pets) & 37 & 88 & 100 & 1.14 & 2.6 & 1.000 \\
IMAGENET-A (ImageNet-A) & 200 & 3 & 100 & 33.33 & 78.2 & 0.946 \\
\bottomrule
\end{tabular}
}
\end{table}

\begin{figure}[t]
	\centering
	\begin{subfigure}[b]{0.32\textwidth}
		\centering
		\includegraphics[width=0.7\textwidth]{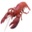}
		\caption{CIFAR}
	\end{subfigure}
	\hfill %
	\begin{subfigure}[b]{0.32\textwidth}
		\centering
		\includegraphics[width=0.45\textwidth]{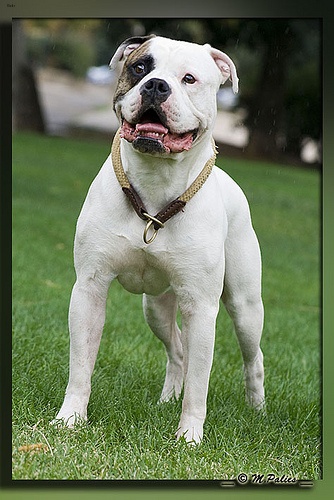}
		\caption{PETS}
	\end{subfigure}
	\hfill
	\begin{subfigure}[b]{0.32\textwidth}
		\centering
		\includegraphics[width=0.45\textwidth]{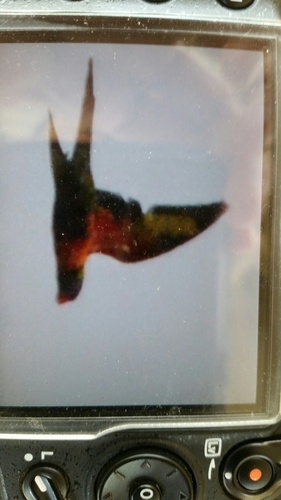}
		\caption{IMAGENET-A}
	\end{subfigure}
	
	\caption{Sample images from the three datasets considered in~\Cref{subsec:CLIPexp}. CIFAR consists of tiny, low resolution (32$\times$32 pixel) images. 
		PETS consists of high quality images of different pets with fine-grained classes, and IMAGENET-A is a notorious challenging dataset for image classification.}
	\label{fig:example_images}
\end{figure}

\section{Reproducibility}\label{app:reprod}

\subsection{Setup}\label{appsubsec:appsetup}

\para{Environment.} The code was executed on a machine with four Intel Xeon Gold 6252N CPU @ 2.30/3.60GHz and 3TB of RAM memory.
All experiments are executed in the same container based on Debian Linux, obtained through Apptainer,\footnote{\url{https://apptainer.org/ }} and that we provide with our source code.\footnote{\codeURL}
Our container builds the necessary environment specifying: the C++ compiler and the Python versions and packages.
Each run is sequential: it uses a single thread on a dedicated CPU core.
All our reported time measurements do not include the training of the ensembles or the evaluation of the predictions of base models.

All our methods use either Python or C++:
\begin{inparaenum}[($i$)]
	\item we use C++17, compiled with \texttt{g++} and optimization flags for all algorithms of \suite, i.e., \suite-$\banz$ and \suite-$\shap$ for the binary and multiclass settings, and \suiteShapImp; and our baselines \texttt{WMG} and \texttt{ISCG}, and the sampling baseline \texttt{MC}.
	\item we use Python 3.10 for the baselines \texttt{Tr}-\texttt{MLE} and \texttt{Tr}-\texttt{PS}, i.e., using the code of~\citet{Rozemberczki2021ShapEnsemble}.\footnote{\url{https://github.com/benedekrozemberczki/shapley}}
\end{inparaenum}
We also use Python to train all ensembles through Scikit-learn, and computing the predictions of CLIP by using PyTorch.

\para{CLIP.} We use CLIP's version ViT-B/32.\footnote{\url{https://github.com/openai/CLIP}}

\para{Ensembles training.} As mentioned in~\Cref{sec:exps} for our experiments comparing with sota (state-of-the-art) approaches we use the $75\%$  of data points of each dataset to train three different ensembles. 
More in detail we use the default classifiers available from Scikit-learn.
In particular,
\begin{compactenum}
	\item For both~\ensRF and~\ensAB we do not set any custom parameter other than the number of classifiers required;
	\item For~\ensHET we train each classifier with slightly different parameters to enforce heterogeneity. 
\end{compactenum}
All the details of the parameter settings are easily accessed in our code repository~\codeURL.

\subsection{Datasets}\label{appsubsec:appdata}
We describe the datasets used in our experiments of~\Cref{sec:exps}. 
The scripts to replicate the processing of our datasets is in our code repository~\codeURL.

\textbf{Binary and multi\-class classification}. The datasets used in our experiments for the binary and multi\-class classification experiments of~\Cref{sec:exps} are reported in~\Cref{tab:tabular_datasets}.
We additionally report the distribution of the classes of each dataset in~\Cref{tab:tabular_datasets}.
All datasets are publicly available.\footnote{\url{https://archive.ics.uci.edu/}.}
Each dataset was processed with a standard pipeline as follows,
\begin{itemize}
	\item for the binary classification scenario we remapped multiple classes by retaining the most frequent class as positive ($y=1$) and all other classes collectively as negative ($y=0$);
	\item numerical features were standardized to zero mean and unit variance;
	\item missing values were set to the median for the respective feature;
	\item categorical features were mapped to integers;
	\item we split each dataset with a 75/15/10 split.  
  The 75\% split of data points is used for training. 
	The 15\% spit constitutes $\dataSetVal$ used to compute the rankings of the power indices according to the various methods. 
	Finally, the remaining 10\% split is used for testing, constituting $\dataSetTra$.
	We first draw the 75\% training split at random, and then split the remaining 25\% of the points at random into $\dataSetVal$ (60\% of the remaining) and $\dataSetTra$ (40\% of remaining points).
\end{itemize}

\textbf{LLM prompt ensembling.}
For our experiments of~\Cref{subsec:CLIPexp} we rely on publicly available datasets for image classification. Their statistics are reported in~\Cref{tab:clip_datasets}, and some sample images are show in~\Cref{fig:example_images}. 
We further report the statistics for class distribution in~\Cref{tab:clip_imbalance}. 
We observe that our datasets are heterogeneous: CIFAR is perfectly balanced (imbalance ratio (IR) $1.0$), PETS is nearly balanced (IR, $1.14$), and IMAGENET-A is highly imbalanced (IR, $33.3$), showing that our methods are not affected by class distribution as captured by our theoretical results.
Each dataset was split retaining 50\% of its points for $\dataSet_{\text{val}}$ and the rest for $\dataSetTra$. 

\subsection{Baselines}\label{appsec:bl}

\para{Baselines for binary classification.} 
Here we provide more details for the baselines of our experiments. 

\texttt{ISCG}~\citep{Ykhlef2017GraphShapEnsamble} computes the Shapley values of each base classifier
through a game on a graph capturing pairwise relationship between base classifier (see~\Cref{sec:related}) and their diversity across predictions.
Note that the game of \texttt{ISCG} does \emph{not} capture how the actual ensemble aggregates their votes.
The game is an induced subgraph game~\citep{Deng1994WMG} on the complete graph where each vertex corresponds to a base classifier, with the addition of self-loops.
An edge $\{i,j\}$ has weight $w_{ij}$, and the self-loop at vertex $i$ has weight $w_{ii}$.
The value $v$ of a coalition $\coal$ corresponds to the sum of the weights of the edges and self-loops with both endpoints in $\coal$.
The Shapley value of $\model_i$ corresponds to $w_{ii}+\frac{1}{2}\sum_{j\neq i}w_{ij}$~\citep{Deng1994WMG}.
The weights depend only on which classifiers are correct.
For a point drawn uniformly at random from $\dataSetVal$, let $Z_i=1$ if $\model_i$ classifies the point correctly, and $Z_i=0$ otherwise.
Following~\citet[Eq.~(5) to (7)]{Ykhlef2017GraphShapEnsamble}, we set
\[
	w_{ii}=\Pr[Z_i=1], \qquad
	w_{ij}=\tfrac{1}{2}\big(\Pr[Z_i=1\mid Z_j=0]+\Pr[Z_j=1\mid Z_i=0]\big) \quad \text{for } i\neq j.
\]
Hence $w_{ii}$ corresponds to the accuracy of $\model_i$ on $\dataSetVal$.
While the term $\Pr[Z_i=1\mid Z_j=0]$ corresponds to the fraction of the points misclassified by $\model_j$ that $\model_i$ classifies correctly.
When $\model_j$ performs no classifications, we set the fraction to $0$.
We compute the weights on $\dataSetVal$, as for all other methods.

\texttt{Tr}-\texttt{MLE} and \texttt{Tr}-\texttt{PS} implement the algorithms 
of~\citet{Owen1972MLE} and \citet{Maleki2013Shap}, respectively.
Both methods compute the Shapley values under the definition of game introduced in \citet{Rozemberczki2021ShapEnsemble}, and both methods are approximate.
In our setting, for each point $\evz=(\xp,\y)$ in $\dataSetVal$ \texttt{Troupe} considers a simple weighted majority game, in which base classifier $\model_i$ has weight $w_i(\xp)=\frac{1}{\cardM}\ind[\model_i(\xp)=\y]$.
And for a fixed coalition $\coal\subseteq \modelSet$ it holds
\[
	v_{\mathsf{Tr}}(\evz,\coal)=\ind\Big[\sum_{i\in\coal} w_i(\xp) > q\Big], \qquad q=\tfrac{1}{2}\enspace.
\]
That is, a coalition $\coal$ attains a value of 1 for the game over $\evz$ it contains at least $\cardM/2$ base classifiers that are correct on $\xp$.
Note that the quota is relative to the size $\cardM$ of the ensemble, not to $|\coal|$, and must be fixed in advance.
The power-index that we report for $\model_i, i\in[\cardM]$ corresponds to average of its Shapley value over the games of the points in $\dataSetVal$.
We use the code of~\citet{Rozemberczki2021ShapEnsemble} with the above weights and quota.
\texttt{Tr}-\texttt{PS} samples $1\,000$ permutations, and we execute it five times.
While, \texttt{Tr}-\texttt{MLE} is deterministic.
We note that on a point with at most $\cardM/2$ correct base classifiers the grand coalition has value $0$, and every Shapley value of that game is $0$.
The code of the authors does not return $0$ on such a point, and we use their code unaltered.\footnote{\url{https://github.com/benedekrozemberczki/shapley}}

\texttt{WMG} implements the algorithms of~\cite{Uno2012WMG} 
for computing the Banzhaf and Shapley values under WMGs, 
for which we set the weight of a classifier on a point to 1 if it predicts the label of the point correctly and to 0 otherwise, and the quota to $q=\lfloor\cardM/2\rfloor + 1$.
Each point of $\dataSetVal$ thus defines one WMG.
For each model $\model_i$, its power index Banzhaf (resp.\ Shapley) value that we report under \texttt{WMG} is obtained as the average over the points of $\dataSetVal$.
Note that in a WMG for a fixed point, the power index depends only on the number of correct base classifiers, i.e., $\corrPred$.
We therefore run the dynamic program of~\citet{Uno2012WMG} once per distinct value of $\corrPred$, and reuse its output for all points with the same  values of $\corrPred$.

We have also implemented two simple Monte-Carlo algorithms for our game, 
denoted as \texttt{MC}-\shap and \texttt{MC}-\banz. 
Given that these algorithms require a high runtime to converge, and 
have poor estimation quality (see~\Cref{tab:mc_ERR}) 
we did not consider such baselines in our discussion. On every configuration, we draw $\cardM^2$ samples for each base classifier. We remark that obtaining efficient sampling algorithms for our formulation remains an interesting and challenging open direction, especially for the multiclass case.

\para{Baselines for multi\-class classification.}
All our baselines based on WMGs do not directly extend to the multiclass setting (see~\Cref{sec:RWExtended}).
Possible ways to adapt such baselines for a fixed point $\evz=(\xp,\y)$ are to average the power indices obtained over
\begin{inparaenum}[($i$)]
	\item one binary game for each pair of labels, i.e., $\bigO(\totLabs^2)$ games for each point; or
	\item one binary game between $\y$ and each of the $\totLabs-1$ incorrect labels, i.e., $\bigO(\totLabs)$ games for each point.
\end{inparaenum}
The only method admitting a multiclass formulation, by~\citet{Ykhlef2017GraphShapEnsamble}, has no public implementation, and it already correlates poorly with our indices in the binary case in our implementation (specifically on \ensRF and \ensHET, see \Cref{subsec:rankingComp}).
We therefore consider \texttt{Tr}-\texttt{MLE} and \texttt{Tr}-\texttt{PS} as our multiclass baselines, with a simpler reduction than ($i$) and ($ii$): a single binary game for each point, between $\y$ and all other incorrect labels.
The game $v_{\mathsf{Tr}}$ of \texttt{Troupe} is a simple weighted game where a coalition (i.e., a sub-ensemble) either wins or loses.
Hence, such formulation cannot capture which of the $\totLabs-1$ incorrect classes an incorrect model votes for.
For each point $\evz=(\xp,\y)$ of $\dataSetVal$ we evaluate \texttt{Troupe} as follows: we retain only the correctness of the predictions of model $\model_i$ on the data point $\xp$, i.e., $\ind[\model_i(\xp)=\y]$; and we assign to each base classifier a weight $w_i(\xp)=\frac{1}{\cardM}\ind[\model_i(\xp)=\y]$, as for the binary case.
That is, a coalition $\coal$ wins, i.e., $v_{\mathsf{Tr}}(\evz,\coal)=1$, if it contains more than $\cardM/2$ correct base classifiers (quota $q=\frac12$). 
We then execute \texttt{Tr}-\texttt{MLE} and \texttt{Tr}-\texttt{PS} on such game with the parameters of the binary case, and average the values of each $\model_i, i\in[\cardM]$ over the points of $\dataSetVal$.
The above game $v_{\mathsf{Tr}}$ differs substantially from $\charMC$:
\begin{inparaenum}[($i$)]
	\item each incorrect base classifier has weight $0$, i.e., models voting for different incorrect labels cannot be distinguished;
	\item $v_{\mathsf{Tr}}$ requires more than $\cardM/2$ correct classifiers to evaluate to 1, while $\charMC$ does not require a fixed quota; and
	\item $v_{\mathsf{Tr}}$ does not account for ties (see~\Cref{eq:charMC}), as it requires a strict majority.
\end{inparaenum}
For example, let $\cardM=22$, and consider a point $\evz$ with $8$ correct votes and $2$ votes on each of the remaining $7$ incorrect labels.
According to the plurality voting rule, the ensemble classifies $\xp$ correctly, thus $\charMC(\evz,\modelSet)=1$, while $v_{\mathsf{Tr}}(\evz,\modelSet)=0$, and every Shapley value of $v_{\mathsf{Tr}}$ for the game over $\evz$ is $0$, since $v_{\mathsf{Tr}}(\evz,\coal)=0$ for every $\coal\subseteq\modelSet$.
Data points with less than $\cardM/2$ base correct classifiers are common in our experiments, e.g., on IMAGENET-A (see~\Cref{tab:results_imagenet-a}).
 
\begin{table}[t]
\setlength{\aboverulesep}{0.6pt}
\setlength{\belowrulesep}{0.6pt}
\setlength{\extrarowheight}{1pt}
\centering
\caption{Approximation error of the \texttt{MC} baseline.
For each dataset and power index we report the maximum absolute $(\varepsilon_{\text{abs}})$ and the maximum relative approximation error $(\varepsilon_{\text{rel}})$,
over all models from $\modelSet$.
We further report the Kendall's ($\tau$) and Spearman ($\rho$) ranking coefficients, with the exact scores computed by \suite.  The results are the average and standard deviation over five random splits.}
\label{tab:mc_ERR}
\resizebox{0.8\textwidth}{!}{\begin{tabular}{lll cccc}
\toprule
Dataset & Ensemble & Index & $\varepsilon_{\text{abs}}$ & $\varepsilon_{\text{rel}}$ & Kendall's ($\tau$) & Spearman ($\rho$) \\
\midrule
\multirow{6}{*}{\texttt{BTSC}} & \multirow{2}{*}{\texttt{RF}} & $\shap$ & $0.0045 \pm 0.0007$ & $0.42 \pm 0.15$ & $0.720 \pm 0.036$ & $0.889 \pm 0.027$ \\
 &  & $\banz$ & $0.0003 \pm 0.0000$ & $3.34 \pm 4.51$ & $0.967 \pm 0.004$ & $0.997 \pm 0.001$ \\
\cmidrule{2-7}
 & \multirow{2}{*}{\texttt{AB}} & $\shap$ & $0.0013 \pm 0.0002$ & $0.26 \pm 0.18$ & $0.819 \pm 0.030$ & $0.930 \pm 0.017$ \\
 &  & $\banz$ & $0.0007 \pm 0.0001$ & $0.07 \pm 0.02$ & $0.794 \pm 0.036$ & $0.917 \pm 0.023$ \\
\cmidrule{2-7}
 & \multirow{2}{*}{\texttt{HET}} & $\shap$ & $0.0111 \pm 0.0021$ & $0.51 \pm 0.10$ & $0.515 \pm 0.094$ & $0.676 \pm 0.105$ \\
 &  & $\banz$ & $0.0005 \pm 0.0002$ & $13.72 \pm 25.07$ & $0.886 \pm 0.073$ & $0.964 \pm 0.034$ \\
\midrule
\multirow{6}{*}{\texttt{CR}} & \multirow{2}{*}{\texttt{RF}} & $\shap$ & $0.0043 \pm 0.0005$ & $0.38 \pm 0.06$ & $0.780 \pm 0.026$ & $0.930 \pm 0.015$ \\
 &  & $\banz$ & $0.0004 \pm 0.0001$ & $1.61 \pm 0.96$ & $0.968 \pm 0.005$ & $0.997 \pm 0.001$ \\
\cmidrule{2-7}
 & \multirow{2}{*}{\texttt{AB}} & $\shap$ & $0.0010 \pm 0.0002$ & $0.28 \pm 0.20$ & $0.882 \pm 0.014$ & $0.964 \pm 0.006$ \\
 &  & $\banz$ & $0.0004 \pm 0.0001$ & $0.07 \pm 0.08$ & $0.863 \pm 0.021$ & $0.953 \pm 0.010$ \\
\cmidrule{2-7}
 & \multirow{2}{*}{\texttt{HET}} & $\shap$ & $0.0106 \pm 0.0022$ & $0.46 \pm 0.14$ & $0.558 \pm 0.170$ & $0.720 \pm 0.194$ \\
 &  & $\banz$ & $0.0007 \pm 0.0001$ & $0.96 \pm 0.87$ & $0.931 \pm 0.037$ & $0.986 \pm 0.010$ \\
\midrule
\multirow{6}{*}{\texttt{C9}} & \multirow{2}{*}{\texttt{RF}} & $\shap$ & $0.0063 \pm 0.0009$ & $0.31 \pm 0.04$ & $0.288 \pm 0.089$ & $0.406 \pm 0.112$ \\
 &  & $\banz$ & $0.0001 \pm 0.0000$ & $0.60 \pm 0.35$ & $0.791 \pm 0.043$ & $0.938 \pm 0.021$ \\
\cmidrule{2-7}
 & \multirow{2}{*}{\texttt{AB}} & $\shap$ & $0.0023 \pm 0.0002$ & $0.10 \pm 0.05$ & $0.800 \pm 0.017$ & $0.929 \pm 0.008$ \\
 &  & $\banz$ & $0.0015 \pm 0.0003$ & $0.06 \pm 0.03$ & $0.764 \pm 0.041$ & $0.910 \pm 0.022$ \\
\cmidrule{2-7}
 & \multirow{2}{*}{\texttt{HET}} & $\shap$ & $0.0148 \pm 0.0029$ & $0.42 \pm 0.09$ & $0.457 \pm 0.139$ & $0.609 \pm 0.152$ \\
 &  & $\banz$ & $0.0003 \pm 0.0001$ & $3.42 \pm 2.97$ & $0.816 \pm 0.065$ & $0.935 \pm 0.032$ \\
\midrule
\multirow{6}{*}{\texttt{TT}} & \multirow{2}{*}{\texttt{RF}} & $\shap$ & $0.0056 \pm 0.0009$ & $0.30 \pm 0.03$ & $0.444 \pm 0.070$ & $0.620 \pm 0.082$ \\
 &  & $\banz$ & $0.0001 \pm 0.0000$ & $5.38 \pm 3.58$ & $0.963 \pm 0.008$ & $0.996 \pm 0.001$ \\
\cmidrule{2-7}
 & \multirow{2}{*}{\texttt{AB}} & $\shap$ & $0.0010 \pm 0.0000$ & $0.12 \pm 0.06$ & $0.936 \pm 0.015$ & $0.989 \pm 0.005$ \\
 &  & $\banz$ & $0.0006 \pm 0.0001$ & $0.06 \pm 0.02$ & $0.891 \pm 0.020$ & $0.970 \pm 0.008$ \\
\cmidrule{2-7}
 & \multirow{2}{*}{\texttt{HET}} & $\shap$ & $0.0136 \pm 0.0028$ & $0.51 \pm 0.08$ & $0.706 \pm 0.047$ & $0.872 \pm 0.032$ \\
 &  & $\banz$ & $0.0006 \pm 0.0001$ & $0.86 \pm 0.63$ & $0.913 \pm 0.022$ & $0.980 \pm 0.007$ \\
\midrule
\multirow{6}{*}{\texttt{BM}} & \multirow{2}{*}{\texttt{RF}} & $\shap$ & $0.0056 \pm 0.0007$ & $0.31 \pm 0.04$ & $0.081 \pm 0.075$ & $0.126 \pm 0.117$ \\
 &  & $\banz$ & $0.0000 \pm 0.0000$ & $3.14 \pm 2.79$ & $0.971 \pm 0.006$ & $0.997 \pm 0.001$ \\
\cmidrule{2-7}
 & \multirow{2}{*}{\texttt{AB}} & $\shap$ & $0.0015 \pm 0.0002$ & $0.90 \pm 0.47$ & $0.929 \pm 0.020$ & $0.984 \pm 0.009$ \\
 &  & $\banz$ & $0.0004 \pm 0.0002$ & $0.46 \pm 0.29$ & $0.941 \pm 0.029$ & $0.987 \pm 0.012$ \\
\cmidrule{2-7}
 & \multirow{2}{*}{\texttt{HET}} & $\shap$ & $0.0134 \pm 0.0026$ & $0.43 \pm 0.07$ & $0.187 \pm 0.107$ & $0.259 \pm 0.132$ \\
 &  & $\banz$ & $0.0001 \pm 0.0000$ & $0.51 \pm 0.47$ & $0.970 \pm 0.018$ & $0.996 \pm 0.003$ \\
\midrule
\multirow{6}{*}{\texttt{CT}} & \multirow{2}{*}{\texttt{RF}} & $\shap$ & $0.0056 \pm 0.0007$ & $0.29 \pm 0.04$ & $0.053 \pm 0.072$ & $0.077 \pm 0.097$ \\
 &  & $\banz$ & $0.0000 \pm 0.0000$ & $0.53 \pm 0.54$ & $0.970 \pm 0.008$ & $0.997 \pm 0.001$ \\
\cmidrule{2-7}
 & \multirow{2}{*}{\texttt{AB}} & $\shap$ & $0.0010 \pm 0.0002$ & $1.68 \pm 2.00$ & $0.843 \pm 0.020$ & $0.945 \pm 0.011$ \\
 &  & $\banz$ & $0.0004 \pm 0.0001$ & $6.28 \pm 4.13$ & $0.825 \pm 0.037$ & $0.929 \pm 0.031$ \\
\cmidrule{2-7}
 & \multirow{2}{*}{\texttt{HET}} & $\shap$ & $0.0119 \pm 0.0024$ & $0.40 \pm 0.05$ & $0.666 \pm 0.036$ & $0.848 \pm 0.026$ \\
 &  & $\banz$ & $0.0002 \pm 0.0001$ & $0.28 \pm 0.20$ & $0.980 \pm 0.005$ & $0.998 \pm 0.000$ \\
\bottomrule
\end{tabular}
}
\end{table}

\begin{table}[t]
	\centering
	\caption{Runtime (seconds) of \suite on $\dataSetVal$ in the multiclass setting. The results are the average and standard deviation over five random splits. \suiteShapImp corresponds to the refined algorithm of~\Cref{app:refinedComplexity}.}
	\label{tab:rtMC}
	\begin{tabular}{llrr}
\toprule
Ensemble & Method & \texttt{C9} & \texttt{CT} \\
\midrule
\multirow{3}{*}{\texttt{RF} ($\cardM=50$)} & \suite-$\banz$ & $(3.7 \pm 0.4) \cdot 10^{-3}$ & $0.10 \pm 0.01$ \\
 & \suite-$\shap$ & $0.14 \pm 0.02$ & $2.6 \pm 0.2$ \\
 & \suiteShapImp & $(2.8 \pm 0.4) \cdot 10^{-2}$ & $0.48 \pm 0.04$ \\
\midrule
\multirow{3}{*}{\texttt{AB} ($\cardM=150$)} & \suite-$\banz$ & $(5.0 \pm 0.9) \cdot 10^{-2}$ & $1.9 \pm 0.4$ \\
 & \suite-$\shap$ & $7.8 \pm 1.4$ & $265 \pm 50$ \\
 & \suiteShapImp & $2.7 \pm 0.5$ & $77 \pm 14$ \\
\midrule
\multirow{3}{*}{\texttt{HET} ($\cardM=30$)} & \suite-$\banz$ & $(1.5 \pm 0.2) \cdot 10^{-3}$ & $(4.0 \pm 0.1) \cdot 10^{-2}$ \\
 & \suite-$\shap$ & $(3.3 \pm 0.4) \cdot 10^{-2}$ & $0.42 \pm 0.01$ \\
 & \suiteShapImp & $(7.4 \pm 1.2) \cdot 10^{-3}$ & $(9.0 \pm 0.3) \cdot 10^{-2}$ \\
\bottomrule
\end{tabular}

\end{table}

\begin{figure}[t]
	\centering
	\includegraphics[width=\textwidth]{plots/runtime/rt_leg_credit-g}\\
	\begin{tabular}{ccc}
		\includegraphics[width=0.29\textwidth]{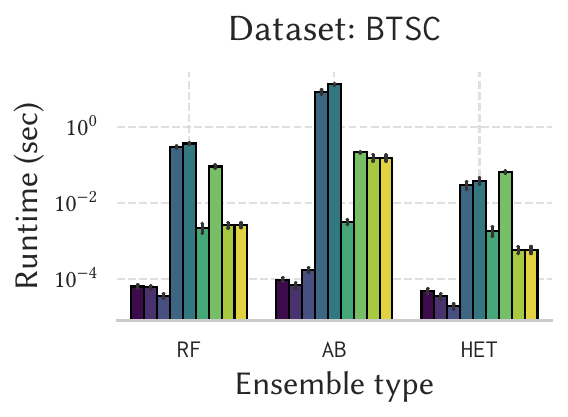}& \includegraphics[width=0.29\textwidth]{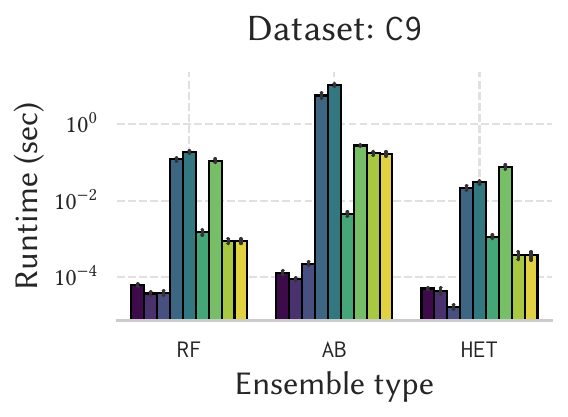} &
		\includegraphics[width=0.29\textwidth]{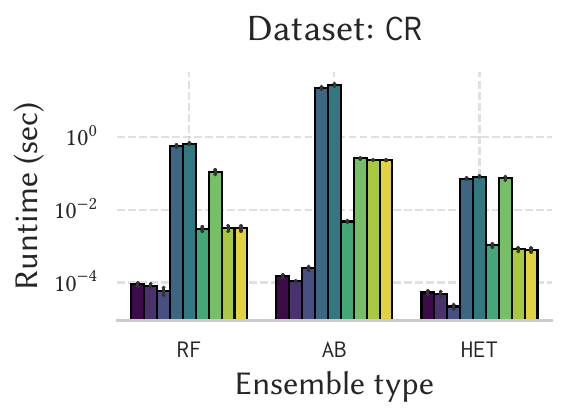}
	\end{tabular}
	\caption{Runtime comparison (mean and standard deviation) on the missing datasets from~\Cref{tab:tabular_datasets} not reported in~\Cref{sec:exps} due to space constraints.}
	\label{fig:rt2}
\end{figure}

\section{Additional experimental results}\label{sec:additionalExps}

\subsection{Monte-Carlo approximation}\label{subsec:MCappxApp}
We show the estimation quality achieved by our simple sampling baseline from~\Cref{alg:mc} in~\Cref{tab:mc_ERR}. We computed, for each model the absolute error $\varepsilon_{\text{abs}} = |\powerIdx_i - \hat{\powerIdx_i}|$ and the relative error $\varepsilon_{\text{rel}} = |\powerIdx_i - \hat{\powerIdx_i}|/|{\powerIdx_i}|$ where $\hat{\powerIdx_i}$ corresponds to the estimate provided by  \texttt{MC}-$\powerIdx$ for $\powerIdx\in\{\banz,\shap\}$. 
Finally we also report the Kendall's $\tau$ and Spearman's correlation metrics.
We recall that the results refer to the binary classification in~\Cref{subsec:binClass}, where the sample size is $s=\cardM^2$; note that the number of possible samples is exponential in $n$.

We observe from~\Cref{tab:mc_ERR}, that overall the performances of the sampling approach are quite remarkable in recovering most of the ranking metrics.
However, we also note that on some configurations (e.g., \texttt{BM} with \texttt{RF} and \texttt{HET}, and \texttt{CT} with \texttt{RF}) the ranking metrics with the exact Shapley values are below $0.3$. 
We also observe a non-negligible relative error in some configurations (e.g., \texttt{BTSC} with \texttt{HET}), likely due to the small values of the values of $\shap_i$ and $\banz_i$ for $i\in[\cardM]$. 
This connects to a well known phenomena for sampling algorithms, which require a very high sample size to provide high quality approximations of extremely small values~\citep{MitzenmacherUpfal2017}.

We remark that we leave as an open direction designing efficient sampling algorithms, especially for the multiclass case, e.g., by reusing samples across different models.

\subsection{Runtime of \suiteMC}
\Cref{tab:rtMC} reports the runtime of our methods on the two datasets of~\Cref{tab:tabular_datasets} with $\totLabs\ge3$ labels, i.e., datasets\texttt{C9} and \texttt{CT}.
\suite-$\banz$ requires at most 2 seconds in all configurations, including \texttt{CT}, our largest dataset.
\suite-$\shap$ requires at most 8 seconds on \texttt{C9}, and at most 3 seconds on \texttt{CT} with \texttt{RF} and \texttt{HET}.
The only exception corresponds to \texttt{CT} with \texttt{AB}, our largest ensemble ($\cardM=150$), where \suite-$\shap$ requires approximately 4 minutes, due to the $\bigO(\cardM^3)$ term in its complexity and to the exact integer arithmetic (see~\Cref{app:practicalImpl}); \suiteShapImp reduces this runtime to about 77 seconds.

\subsection{Accuracy of the base classifiers}\label{subsec:accSpreadApp}
\Cref{tab:accSpread} reports the mean and the standard deviation of the accuracy of the base classifiers on $\dataSetVal$, for each combination of dataset and ensemble.
\begin{table}[h]
	\centering
	\caption{Accuracy on $\dataSetVal$ of the base classifiers of each ensemble, binary classification. Each entry is the mean $\pm$ the standard deviation over the base classifiers of the ensemble, both averaged over five random splits.}
	\label{tab:accSpread}
	\resizebox{\textwidth}{!}{\begin{tabular}{lcccccc}
\toprule
Ensemble & \texttt{BTSC} & \texttt{CR} & \texttt{C9} & \texttt{TT} & \texttt{BM} & \texttt{CT} \\
\midrule
\texttt{RF} & $0.697 \pm 0.033$ & $0.648 \pm 0.035$ & $0.979 \pm 0.014$ & $0.884 \pm 0.032$ & $0.868 \pm 0.003$ & $0.886 \pm 0.005$ \\
\texttt{AB} & $0.510 \pm 0.251$ & $0.505 \pm 0.176$ & $0.534 \pm 0.376$ & $0.535 \pm 0.156$ & $0.545 \pm 0.311$ & $0.519 \pm 0.025$ \\
\texttt{HET} & $0.745 \pm 0.032$ & $0.700 \pm 0.036$ & $0.940 \pm 0.062$ & $0.834 \pm 0.121$ & $0.895 \pm 0.008$ & $0.768 \pm 0.068$ \\
\bottomrule
\end{tabular}
}
\end{table}

\subsection{Ensemble sub-selection}\label{subsec:expsSubApp}
We present additional experimental results for the sub-ensemble selection task in~\Cref{fig:gener2}.
To complement our discussion, we observe that a selection of the top-$k$ models with the rankings provided by \suite yields a sub-ensemble consistently ranking among the ones with best generalization for the accuracy and $F_1$-score metrics. There are, however few exceptions.
On \texttt{CT} with \ensAB, \textsc{Acc} attains higher accuracy and $F_1$-score for intermediate values of $k$.
On \texttt{CR} with \ensHET, \texttt{WMG}-\shap and \texttt{WMG}-\banz attain higher metrics for small values of $k$.
On both configurations the base classifiers have similar accuracy (see~\Cref{tab:accSpread}), thus the rankings of the methods depend on how each formulation weights the classifiers on points with tied votes.
The indices of \suite capture the contribution of each base classifier to the full ensemble with random tie-breaking, and such contributions may differ on a sub-ensemble of size $k< \cardM$. 
We believe that such behavior may also be related to other specific properties of the datasets and of the base models.
Connecting properties of the input dataset with the use of the indices of \suite for sub-ensemble selection is an interesting and important direction for our work.
We note that when performing sub-ensemble selection, a different approach is to recompute the indices of the remaining players after selecting a model, i.e., see \citep{Cohen2007}.
\begin{figure}[t]
	\centering
	\includegraphics[width=\textwidth]{plots/generalization/legend}\\
	\begin{tabular}{cc}
		\includegraphics[width=0.48\textwidth]{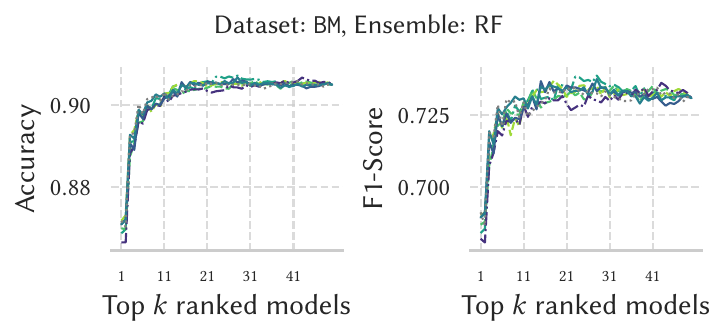}& \includegraphics[width=0.48\textwidth]{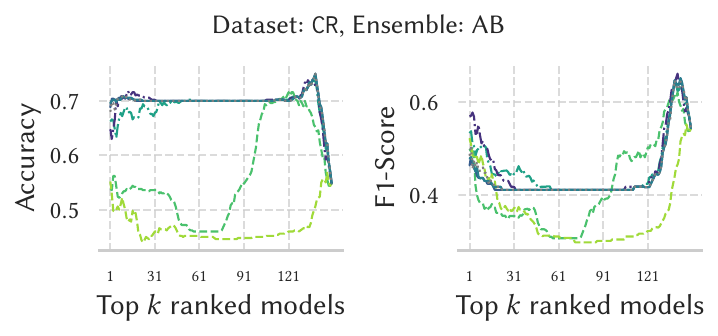} \\
		\includegraphics[width=0.48\textwidth]{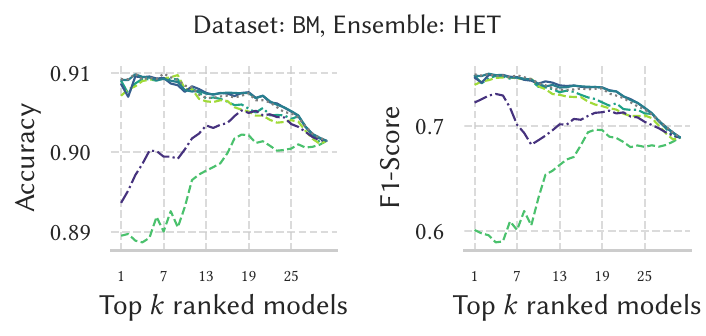} &
		\includegraphics[width=0.48\textwidth]{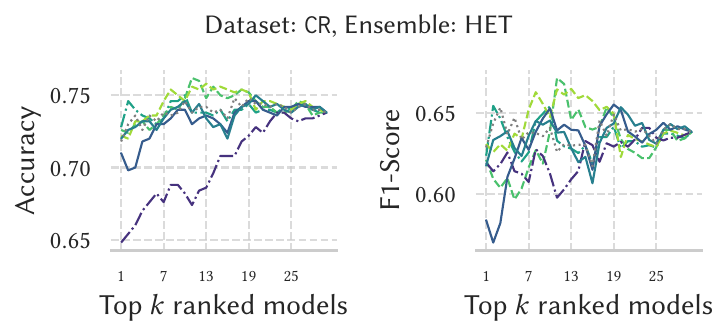} \\
		\includegraphics[width=0.48\textwidth]{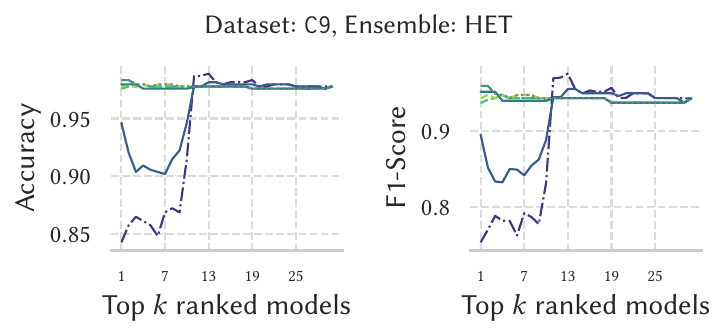} &
		\includegraphics[width=0.48\textwidth]{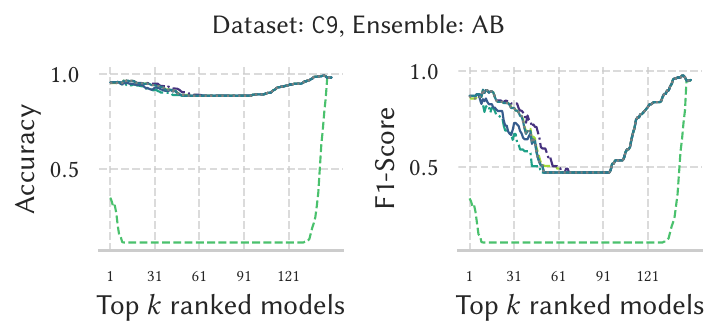}\\
		\includegraphics[width=0.48\textwidth]{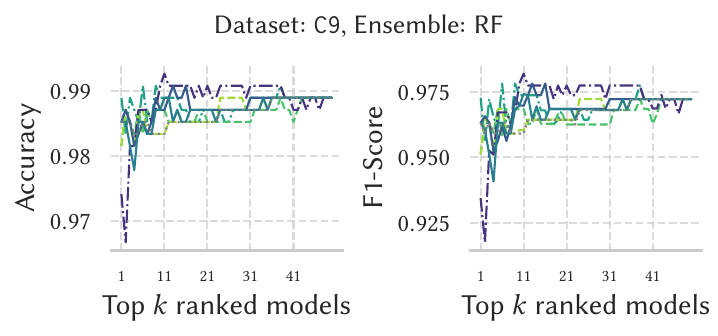} &
		\includegraphics[width=0.48\textwidth]{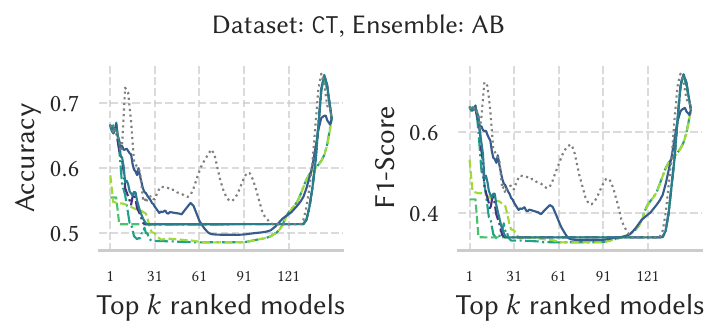}
	\end{tabular}
	\caption{Generalization metrics for sub-ensemble selection, missing results from \Cref{sec:subens}.}
	\label{fig:gener2}
\end{figure}

\subsection{CLIP prompt ensembling}\label{subsec:CLIPpromptApp}

The missing results of~\Cref{subsec:CLIPexp} on the PETS dataset are reported in~\Cref{fig:missingCLIP}.
\Cref{tab:clipSummary} reports, for each ranking, the best accuracy and $F_1$-score over $k$, the value of $k$ attaining it, and the average of each metric over $k\in[\cardM]$.
The best $k\in[\cardM]$ corresponds to the value that maximizes the mean metric over the five $\dataSetTra$ (i.e., chosen independently for the accuracy and $F_1$-score). 
Each entry of the table corresponds to the mean and its the standard deviation over the splits.
As discussed in our main~\Cref{subsec:CLIPexp}, a selection according to the power indices yields metrics comparable to \textsc{Acc} and higher than \texttt{Tr}-\texttt{MLE} and \texttt{Tr}-\texttt{PS}.

\Cref{tab:clipRuntime} reports the runtime of all methods on $\dataSetVal$.
\suite-$\banz$ terminates in less than $0.02$ seconds and \suite-$\shap$ in less than one second on all three datasets, while \texttt{Tr}-\texttt{PS} requires up to $1.3$ seconds.

In addition we also show the variation in the rankings produced by accuracy and our methods on a fixed split in~\Cref{tab:results_cifar100,tab:results_imagenet-a,tab:results_pets}.
An analysis of such tables provides some interesting results about the interaction between different prompts.
We note that, for example, on dataset  CIFAR the prompt ``A cropped photo of a \{\}.'', which is highly ranked by accuracy, has a poor ranking according to power indices. 
That is, such prompt is affecting the overall performances of the ensemble, as it can be seen by its negative Banzhaf value.

Other interesting trends can be observed over~\Cref{tab:results_cifar100,tab:results_imagenet-a,tab:results_pets} for example
\begin{compactitem}
	\item the rankings of prompt ``A macro shot of a \{\}" in~\Cref{tab:results_pets}. That is such prompt that has poor accuracy ranking but high synergy with the overall ensemble as captured by power indices.
	\item the rankings of prompt ``A photo of a \{\}, with a watermark.'' on the IMAGENET-A dataset in~\Cref{tab:results_imagenet-a}. That prompt is ranked high by accuracy, but is has significantly smaller ranks according to power indices.
\end{compactitem}

Overall our experiments for prompt ensembling show that \suite for multiclass classification can provide a valuable addition for LLM prompt engineering, and to enhance the interpretability of LLM predictions under different prompts.

\begin{figure}[t]
		\centering
	\includegraphics[width=0.8\textwidth]{plots/CLIP/legend}\\
	\includegraphics[width=0.7\textwidth]{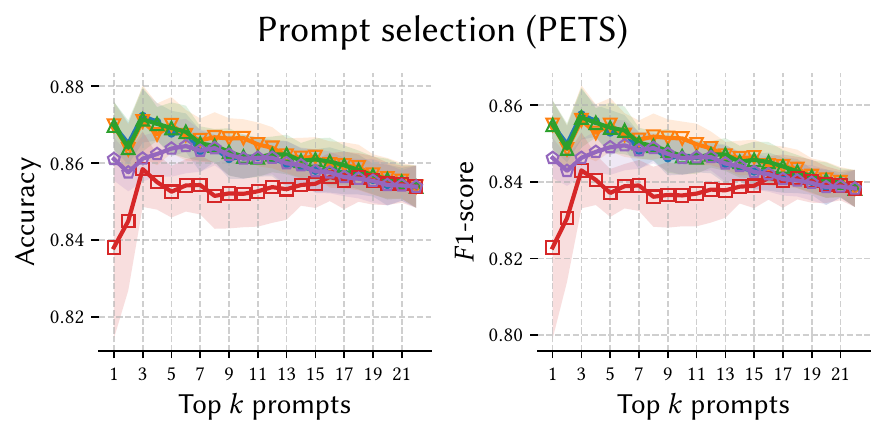}
	\caption{Accuracy and $F_1$-score test metrics for the prompt ensembling of CLIP on $\dataSetTra$ (see~\Cref{subsec:CLIPexp}), as a function of $k$. Missing results from~\Cref{subsec:CLIPexp}.}
	\label{fig:missingCLIP}
\end{figure}

\begin{table}[t]
	\centering
	\caption{Prompt ensembling of CLIP, metrics on $\dataSetTra$ of the sub-ensemble for varying $k$ (see \Cref{subsec:CLIPexp}). Columns report, best: the maximum over $k$ of the average over the five random splits, and the $k$ attaining it. Avg.: the average over $k\in[\cardM]$. Bold marks the best value for each configuration.}
	\label{tab:clipSummary}
	\resizebox{\textwidth}{!}{\begin{tabular}{llcccccc}
\toprule
Dataset & Ranking & best acc. & $k$ & avg.\ acc. & best $F_1$ & $k$ & avg.\ $F_1$ \\
\midrule
CIFAR100 & \textsc{Acc} & $0.648 \pm 0.005$ & 14 & $0.645 \pm 0.004$ & {\boldmath$0.641 \pm 0.004$} & 16 & $0.637 \pm 0.003$ \\
 & \textsc{enpower}-$\beta$ & {\boldmath$0.649 \pm 0.004$} & 10 & {\boldmath$0.646 \pm 0.004$} & {\boldmath$0.641 \pm 0.003$} & 16 & {\boldmath$0.638 \pm 0.003$} \\
 & \textsc{enpower}-$\phi$ & {\boldmath$0.649 \pm 0.004$} & 8 & {\boldmath$0.646 \pm 0.004$} & {\boldmath$0.641 \pm 0.004$} & 17 & {\boldmath$0.638 \pm 0.003$} \\
 & $\mathtt{Tr\text{-}PS}$ & $0.647 \pm 0.005$ & 15 & $0.643 \pm 0.006$ & $0.639 \pm 0.004$ & 15 & $0.635 \pm 0.004$ \\
 & $\mathtt{Tr\text{-}MLE}$ & $0.648 \pm 0.004$ & 14 & $0.643 \pm 0.005$ & $0.640 \pm 0.003$ & 14 & $0.636 \pm 0.003$ \\
\midrule
IMAGENET-A & \textsc{Acc} & {\boldmath$0.308 \pm 0.004$} & 4 & $0.302 \pm 0.004$ & {\boldmath$0.288 \pm 0.005$} & 1 & {\boldmath$0.279 \pm 0.003$} \\
 & \textsc{enpower}-$\beta$ & $0.307 \pm 0.003$ & 5 & {\boldmath$0.303 \pm 0.004$} & $0.285 \pm 0.004$ & 5 & {\boldmath$0.279 \pm 0.003$} \\
 & \textsc{enpower}-$\phi$ & {\boldmath$0.308 \pm 0.003$} & 5 & $0.302 \pm 0.004$ & {\boldmath$0.288 \pm 0.005$} & 1 & {\boldmath$0.279 \pm 0.003$} \\
 & $\mathtt{Tr\text{-}PS}$ & $0.298 \pm 0.006$ & 13 & $0.295 \pm 0.004$ & $0.275 \pm 0.003$ & 10 & $0.272 \pm 0.002$ \\
 & $\mathtt{Tr\text{-}MLE}$ & $0.303 \pm 0.004$ & 8 & $0.300 \pm 0.004$ & $0.279 \pm 0.004$ & 8 & $0.275 \pm 0.003$ \\
\midrule
PETS & \textsc{Acc} & {\boldmath$0.872 \pm 0.007$} & 3 & $0.862 \pm 0.005$ & {\boldmath$0.858 \pm 0.007$} & 3 & $0.846 \pm 0.004$ \\
 & \textsc{enpower}-$\beta$ & $0.871 \pm 0.009$ & 3 & {\boldmath$0.863 \pm 0.006$} & $0.856 \pm 0.009$ & 3 & {\boldmath$0.848 \pm 0.005$} \\
 & \textsc{enpower}-$\phi$ & $0.871 \pm 0.008$ & 3 & $0.862 \pm 0.005$ & $0.857 \pm 0.008$ & 3 & $0.847 \pm 0.005$ \\
 & $\mathtt{Tr\text{-}PS}$ & $0.858 \pm 0.010$ & 3 & $0.853 \pm 0.007$ & $0.843 \pm 0.010$ & 3 & $0.838 \pm 0.006$ \\
 & $\mathtt{Tr\text{-}MLE}$ & $0.864 \pm 0.007$ & 6 & $0.859 \pm 0.005$ & $0.850 \pm 0.007$ & 6 & $0.844 \pm 0.004$ \\
\bottomrule
\end{tabular}
}
\end{table}

\begin{table}[t]
	\centering
	\caption{Runtime (seconds) of the ranking methods on $\dataSetVal$ for the prompt ensembling of CLIP, average $\pm$ standard deviation over five random splits.}
	\label{tab:clipRuntime}
	\begin{tabular}{lccc}
\toprule
Method & CIFAR100 & IMAGENET-A & PETS \\
\midrule
\textsc{enpower}-$\beta$ & $(1.17 \pm 0.01) \cdot 10^{-2}$ & $(1.37 \pm 0.04) \cdot 10^{-2}$ & $(1.64 \pm 0.09) \cdot 10^{-3}$ \\
\textsc{enpower}-$\phi$ & $(3.61 \pm 0.07) \cdot 10^{-1}$ & $(5.14 \pm 0.18) \cdot 10^{-1}$ & $(3.45 \pm 0.20) \cdot 10^{-2}$ \\
$\mathtt{Tr\text{-}PS}$ & $(1.33 \pm 0.13) \cdot 10^{0}$ & $(8.64 \pm 0.44) \cdot 10^{-1}$ & $(3.77 \pm 0.04) \cdot 10^{-1}$ \\
$\mathtt{Tr\text{-}MLE}$ & $(2.37 \pm 0.05) \cdot 10^{-2}$ & $(1.84 \pm 0.08) \cdot 10^{-2}$ & $(1.26 \pm 0.02) \cdot 10^{-2}$ \\
\bottomrule
\end{tabular}

\end{table}

\begin{table}[t]
\setlength{\aboverulesep}{0pt}
\setlength{\belowrulesep}{0pt}
\setlength{\extrarowheight}{0.8pt}
\centering
\caption{Prompt ensemble evaluation for LLM zero-shot multiclass image classification (CIFAR100). Each prompt corresponds to a different model of $\mathcal{C}$. All values and rankings are computed on the validation split $\dataSetVal$. Columns $\Delta_{\banz}$ and $\Delta_{\shap}$ display the difference in the validation rankings induced by the power indices $\banz(G_{\mathsf{MC}})$ and $\shap(G_{\mathsf{MC}})$ with respect to validation accuracy.}
\label{tab:results_cifar100}
\definecolor{CBB}{HTML}{0072B2} %
\definecolor{CBG}{HTML}{009E73} %
\definecolor{CBR}{HTML}{D55E00} %
\definecolor{CBGy}{HTML}{999999} %
\resizebox{\textwidth}{!}{
\begin{tabular}{l rr rr rr cc}
\toprule
\multirow{2}{*}{\textbf{Prompt template}} & \multicolumn{2}{c}{\textbf{Accuracy}} & \multicolumn{2}{c}{\textbf{Banzhaf}} & \multicolumn{2}{c}{\textbf{Shapley}} & \multirow{2}{*}{$\Delta_{\banz}$} & \multirow{2}{*}{$\Delta_{\shap}$} \\
\cmidrule(lr){2-3} \cmidrule(lr){4-5} \cmidrule(lr){6-7}
& Value & Rank & Value & Rank & Value & Rank & & \\
\midrule
A pixelated photo of a \{\}. & \cellcolor{CBB!100} \color{white} 0.6338 & \cellcolor{CBB!100} \color{white} \textbf{1} & \cellcolor{CBB!100} \color{white} 0.00296 & \cellcolor{CBB!100} \color{white} \textbf{1} & \cellcolor{CBB!100} \color{white} 0.03285 & \cellcolor{CBB!100} \color{white} \textbf{1} & \textcolor{CBGy}{$-$} & \textcolor{CBGy}{$-$} \\
A heavily compressed JPEG photo of a \{\}. & \cellcolor{CBB!95} \color{white} 0.6308 & \cellcolor{CBB!95} \color{white} \textbf{2} & \cellcolor{CBB!95} \color{white} 0.00262 & \cellcolor{CBB!95} \color{white} \textbf{2} & \cellcolor{CBB!95} \color{white} 0.03226 & \cellcolor{CBB!95} \color{white} \textbf{2} & \textcolor{CBGy}{$-$} & \textcolor{CBGy}{$-$} \\
A high-resolution, crisp photo of a \{\}. & \cellcolor{CBB!91} \color{white} 0.6248 & \cellcolor{CBB!91} \color{white} \textbf{3} & \cellcolor{CBB!82} \color{white} 0.00113 & \cellcolor{CBB!82} \color{white} 5 & \cellcolor{CBB!87} \color{white} 0.03047 & \cellcolor{CBB!87} \color{white} 4 & \textcolor{CBR}{\textbf{$\downarrow$ 2}} & \textcolor{CBR}{\textbf{$\downarrow$ 1}} \\
A cropped photo of a \{\}. & \cellcolor{CBB!87} \color{white} 0.6242 & \cellcolor{CBB!87} \color{white} 4 & \cellcolor{CBB!40} \color{black} -0.00003 & \cellcolor{CBB!40} \color{black} 15 & \cellcolor{CBB!61} \color{white} 0.02941 & \cellcolor{CBB!61} \color{white} 10 & \textcolor{CBR}{\textbf{$\downarrow$ 11}} & \textcolor{CBR}{\textbf{$\downarrow$ 6}} \\
A photo of a \{\}, cluttered background. & \cellcolor{CBB!82} \color{white} 0.6238 & \cellcolor{CBB!82} \color{white} 5 & \cellcolor{CBB!91} \color{white} 0.00161 & \cellcolor{CBB!91} \color{white} \textbf{3} & \cellcolor{CBB!91} \color{white} 0.03077 & \cellcolor{CBB!91} \color{white} \textbf{3} & \textcolor{CBG}{\textbf{$\uparrow$ 2}} & \textcolor{CBG}{\textbf{$\uparrow$ 2}} \\
An unusual, hard to recognize photo of a \{\}. & \cellcolor{CBB!78} \color{white} 0.6236 & \cellcolor{CBB!78} \color{white} 6 & \cellcolor{CBB!70} \color{white} 0.00095 & \cellcolor{CBB!70} \color{white} 8 & \cellcolor{CBB!78} \color{white} 0.03024 & \cellcolor{CBB!78} \color{white} 6 & \textcolor{CBR}{\textbf{$\downarrow$ 2}} & \textcolor{CBGy}{$-$} \\
A photo of a \{\}, with a watermark. & \cellcolor{CBB!74} \color{white} 0.6232 & \cellcolor{CBB!74} \color{white} 7 & \cellcolor{CBB!65} \color{white} 0.00075 & \cellcolor{CBB!65} \color{white} 9 & \cellcolor{CBB!65} \color{white} 0.02991 & \cellcolor{CBB!65} \color{white} 9 & \textcolor{CBR}{\textbf{$\downarrow$ 2}} & \textcolor{CBR}{\textbf{$\downarrow$ 2}} \\
A hand-drawn illustration of a \{\}. & \cellcolor{CBB!70} \color{white} 0.6212 & \cellcolor{CBB!70} \color{white} 8 & \cellcolor{CBB!61} \color{white} 0.00070 & \cellcolor{CBB!61} \color{white} 10 & \cellcolor{CBB!74} \color{white} 0.02997 & \cellcolor{CBB!74} \color{white} 7 & \textcolor{CBR}{\textbf{$\downarrow$ 2}} & \textcolor{CBG}{\textbf{$\uparrow$ 1}} \\
A black and white photo of a \{\}. & \cellcolor{CBB!65} \color{white} 0.6206 & \cellcolor{CBB!65} \color{white} 9 & \cellcolor{CBB!87} \color{white} 0.00153 & \cellcolor{CBB!87} \color{white} 4 & \cellcolor{CBB!82} \color{white} 0.03046 & \cellcolor{CBB!82} \color{white} 5 & \textcolor{CBG}{\textbf{$\uparrow$ 5}} & \textcolor{CBG}{\textbf{$\uparrow$ 4}} \\
A painting of a \{\}. & \cellcolor{CBB!61} \color{white} 0.6198 & \cellcolor{CBB!61} \color{white} 10 & \cellcolor{CBB!74} \color{white} 0.00102 & \cellcolor{CBB!74} \color{white} 7 & \cellcolor{CBB!70} \color{white} 0.02992 & \cellcolor{CBB!70} \color{white} 8 & \textcolor{CBG}{\textbf{$\uparrow$ 3}} & \textcolor{CBG}{\textbf{$\uparrow$ 2}} \\
A blurry photo of a \{\}. & \cellcolor{CBB!57} \color{black} 0.6196 & \cellcolor{CBB!57} \color{black} 11 & \cellcolor{CBB!48} \color{black} 0.00032 & \cellcolor{CBB!48} \color{black} 13 & \cellcolor{CBB!57} \color{black} 0.02934 & \cellcolor{CBB!57} \color{black} 11 & \textcolor{CBR}{\textbf{$\downarrow$ 2}} & \textcolor{CBGy}{$-$} \\
A photo of a \{\}. & \cellcolor{CBB!52} \color{black} 0.6184 & \cellcolor{CBB!52} \color{black} 12 & \cellcolor{CBB!35} \color{black} -0.00006 & \cellcolor{CBB!35} \color{black} 16 & \cellcolor{CBB!52} \color{black} 0.02884 & \cellcolor{CBB!52} \color{black} 12 & \textcolor{CBR}{\textbf{$\downarrow$ 4}} & \textcolor{CBGy}{$-$} \\
A poor quality, grainy photo of a \{\}. & \cellcolor{CBB!48} \color{black} 0.6112 & \cellcolor{CBB!48} \color{black} 13 & \cellcolor{CBB!52} \color{black} 0.00036 & \cellcolor{CBB!52} \color{black} 12 & \cellcolor{CBB!44} \color{black} 0.02866 & \cellcolor{CBB!44} \color{black} 14 & \textcolor{CBG}{\textbf{$\uparrow$ 1}} & \textcolor{CBR}{\textbf{$\downarrow$ 1}} \\
A photo focusing on the shape of a \{\}. & \cellcolor{CBB!44} \color{black} 0.6100 & \cellcolor{CBB!44} \color{black} 14 & \cellcolor{CBB!27} \color{black} -0.00041 & \cellcolor{CBB!27} \color{black} 18 & \cellcolor{CBB!31} \color{black} 0.02776 & \cellcolor{CBB!31} \color{black} 17 & \textcolor{CBR}{\textbf{$\downarrow$ 4}} & \textcolor{CBR}{\textbf{$\downarrow$ 3}} \\
A photo focusing on the texture of a \{\}. & \cellcolor{CBB!40} \color{black} 0.6098 & \cellcolor{CBB!40} \color{black} 15 & \cellcolor{CBB!31} \color{black} -0.00027 & \cellcolor{CBB!31} \color{black} 17 & \cellcolor{CBB!35} \color{black} 0.02787 & \cellcolor{CBB!35} \color{black} 16 & \textcolor{CBR}{\textbf{$\downarrow$ 2}} & \textcolor{CBR}{\textbf{$\downarrow$ 1}} \\
A photo of a \{\}, with text over it. & \cellcolor{CBB!35} \color{black} 0.6076 & \cellcolor{CBB!35} \color{black} 16 & \cellcolor{CBB!57} \color{black} 0.00041 & \cellcolor{CBB!57} \color{black} 11 & \cellcolor{CBB!40} \color{black} 0.02847 & \cellcolor{CBB!40} \color{black} 15 & \textcolor{CBG}{\textbf{$\uparrow$ 5}} & \textcolor{CBG}{\textbf{$\uparrow$ 1}} \\
A zoomed-out photo of a \{\}. & \cellcolor{CBB!31} \color{black} 0.6046 & \cellcolor{CBB!31} \color{black} 17 & \cellcolor{CBB!18} \color{black} -0.00084 & \cellcolor{CBB!18} \color{black} 20 & \cellcolor{CBB!22} \color{black} 0.02699 & \cellcolor{CBB!22} \color{black} 19 & \textcolor{CBR}{\textbf{$\downarrow$ 3}} & \textcolor{CBR}{\textbf{$\downarrow$ 2}} \\
A photo of a \{\} in its natural environment. & \cellcolor{CBB!27} \color{black} 0.6046 & \cellcolor{CBB!27} \color{black} 18 & \cellcolor{CBB!78} \color{white} 0.00113 & \cellcolor{CBB!78} \color{white} 6 & \cellcolor{CBB!48} \color{black} 0.02871 & \cellcolor{CBB!48} \color{black} 13 & \textcolor{CBG}{\textbf{$\uparrow$ 12}} & \textcolor{CBG}{\textbf{$\uparrow$ 5}} \\
A photo focusing closely on the details of a \{\}. & \cellcolor{CBB!22} \color{black} 0.6032 & \cellcolor{CBB!22} \color{black} 19 & \cellcolor{CBB!22} \color{black} -0.00064 & \cellcolor{CBB!22} \color{black} 19 & \cellcolor{CBB!18} \color{black} 0.02688 & \cellcolor{CBB!18} \color{black} 20 & \textcolor{CBGy}{$-$} & \textcolor{CBR}{\textbf{$\downarrow$ 1}} \\
A close-up photo of a \{\}. & \cellcolor{CBB!18} \color{black} 0.6002 & \cellcolor{CBB!18} \color{black} 20 & \cellcolor{CBB!10} \color{black} -0.00107 & \cellcolor{CBB!10} \color{black} 22 & \cellcolor{CBB!14} \color{black} 0.02633 & \cellcolor{CBB!14} \color{black} 21 & \textcolor{CBR}{\textbf{$\downarrow$ 2}} & \textcolor{CBR}{\textbf{$\downarrow$ 1}} \\
A sketch of a \{\}. & \cellcolor{CBB!14} \color{black} 0.5992 & \cellcolor{CBB!14} \color{black} 21 & \cellcolor{CBB!44} \color{black} 0.00010 & \cellcolor{CBB!44} \color{black} 14 & \cellcolor{CBB!27} \color{black} 0.02741 & \cellcolor{CBB!27} \color{black} 18 & \textcolor{CBG}{\textbf{$\uparrow$ 7}} & \textcolor{CBG}{\textbf{$\uparrow$ 3}} \\
A macro shot of a \{\}. & \cellcolor{CBB!10} \color{black} 0.5872 & \cellcolor{CBB!10} \color{black} 22 & \cellcolor{CBB!14} \color{black} -0.00092 & \cellcolor{CBB!14} \color{black} 21 & \cellcolor{CBB!10} \color{black} 0.02536 & \cellcolor{CBB!10} \color{black} 22 & \textcolor{CBG}{\textbf{$\uparrow$ 1}} & \textcolor{CBGy}{$-$} \\
\bottomrule
\end{tabular}
}
\end{table}

\begin{table}[t]
\setlength{\aboverulesep}{0pt}
\setlength{\belowrulesep}{0pt}
\setlength{\extrarowheight}{0.8pt}
\centering
\caption{Prompt ensemble evaluation for LLM zero-shot multiclass image classification (PETS). Each prompt corresponds to a different model of $\mathcal{C}$. All values and rankings are computed on the validation split $\dataSetVal$. Columns $\Delta_{\banz}$ and $\Delta_{\shap}$ display the difference in the validation rankings induced by the power indices $\banz(G_{\mathsf{MC}})$ and $\shap(G_{\mathsf{MC}})$ with respect to validation accuracy.}
\label{tab:results_pets}
\definecolor{CBB}{HTML}{0072B2} %
\definecolor{CBG}{HTML}{009E73} %
\definecolor{CBR}{HTML}{D55E00} %
\definecolor{CBGy}{HTML}{999999} %
\resizebox{\textwidth}{!}{
\begin{tabular}{l rr rr rr cc}
\toprule
\multirow{2}{*}{\textbf{Prompt template}} & \multicolumn{2}{c}{\textbf{Accuracy}} & \multicolumn{2}{c}{\textbf{Banzhaf}} & \multicolumn{2}{c}{\textbf{Shapley}} & \multirow{2}{*}{$\Delta_{\banz}$} & \multirow{2}{*}{$\Delta_{\shap}$} \\
\cmidrule(lr){2-3} \cmidrule(lr){4-5} \cmidrule(lr){6-7}
& Value & Rank & Value & Rank & Value & Rank & & \\
\midrule
A poor quality, grainy photo of a \{\}. & \cellcolor{CBB!100} \color{white} 0.8648 & \cellcolor{CBB!100} \color{white} \textbf{1} & \cellcolor{CBB!100} \color{white} 0.00434 & \cellcolor{CBB!100} \color{white} \textbf{1} & \cellcolor{CBB!100} \color{white} 0.04529 & \cellcolor{CBB!100} \color{white} \textbf{1} & \textcolor{CBGy}{$-$} & \textcolor{CBGy}{$-$} \\
A blurry photo of a \{\}. & \cellcolor{CBB!95} \color{white} 0.8539 & \cellcolor{CBB!95} \color{white} \textbf{2} & \cellcolor{CBB!95} \color{white} 0.00328 & \cellcolor{CBB!95} \color{white} \textbf{2} & \cellcolor{CBB!95} \color{white} 0.04335 & \cellcolor{CBB!95} \color{white} \textbf{2} & \textcolor{CBGy}{$-$} & \textcolor{CBGy}{$-$} \\
A close-up photo of a \{\}. & \cellcolor{CBB!91} \color{white} 0.8522 & \cellcolor{CBB!91} \color{white} \textbf{3} & \cellcolor{CBB!91} \color{white} 0.00327 & \cellcolor{CBB!91} \color{white} \textbf{3} & \cellcolor{CBB!91} \color{white} 0.04312 & \cellcolor{CBB!91} \color{white} \textbf{3} & \textcolor{CBGy}{$-$} & \textcolor{CBGy}{$-$} \\
A cropped photo of a \{\}. & \cellcolor{CBB!87} \color{white} 0.8506 & \cellcolor{CBB!87} \color{white} 4 & \cellcolor{CBB!87} \color{white} 0.00188 & \cellcolor{CBB!87} \color{white} 4 & \cellcolor{CBB!87} \color{white} 0.04175 & \cellcolor{CBB!87} \color{white} 4 & \textcolor{CBGy}{$-$} & \textcolor{CBGy}{$-$} \\
A heavily compressed JPEG photo of a \{\}. & \cellcolor{CBB!82} \color{white} 0.8462 & \cellcolor{CBB!82} \color{white} 5 & \cellcolor{CBB!65} \color{white} 0.00095 & \cellcolor{CBB!65} \color{white} 9 & \cellcolor{CBB!74} \color{white} 0.04059 & \cellcolor{CBB!74} \color{white} 7 & \textcolor{CBR}{\textbf{$\downarrow$ 4}} & \textcolor{CBR}{\textbf{$\downarrow$ 2}} \\
A photo focusing on the texture of a \{\}. & \cellcolor{CBB!78} \color{white} 0.8451 & \cellcolor{CBB!78} \color{white} 6 & \cellcolor{CBB!78} \color{white} 0.00163 & \cellcolor{CBB!78} \color{white} 6 & \cellcolor{CBB!78} \color{white} 0.04088 & \cellcolor{CBB!78} \color{white} 6 & \textcolor{CBGy}{$-$} & \textcolor{CBGy}{$-$} \\
A photo of a \{\}. & \cellcolor{CBB!74} \color{white} 0.8446 & \cellcolor{CBB!74} \color{white} 7 & \cellcolor{CBB!61} \color{white} 0.00019 & \cellcolor{CBB!61} \color{white} 10 & \cellcolor{CBB!70} \color{white} 0.03962 & \cellcolor{CBB!70} \color{white} 8 & \textcolor{CBR}{\textbf{$\downarrow$ 3}} & \textcolor{CBR}{\textbf{$\downarrow$ 1}} \\
A pixelated photo of a \{\}. & \cellcolor{CBB!70} \color{white} 0.8435 & \cellcolor{CBB!70} \color{white} 8 & \cellcolor{CBB!82} \color{white} 0.00180 & \cellcolor{CBB!82} \color{white} 5 & \cellcolor{CBB!82} \color{white} 0.04099 & \cellcolor{CBB!82} \color{white} 5 & \textcolor{CBG}{\textbf{$\uparrow$ 3}} & \textcolor{CBG}{\textbf{$\uparrow$ 3}} \\
A high-resolution, crisp photo of a \{\}. & \cellcolor{CBB!65} \color{white} 0.8375 & \cellcolor{CBB!65} \color{white} 9 & \cellcolor{CBB!48} \color{black} -0.00002 & \cellcolor{CBB!48} \color{black} 13 & \cellcolor{CBB!65} \color{white} 0.03878 & \cellcolor{CBB!65} \color{white} 9 & \textcolor{CBR}{\textbf{$\downarrow$ 4}} & \textcolor{CBGy}{$-$} \\
A photo of a \{\}, with text over it. & \cellcolor{CBB!61} \color{white} 0.8337 & \cellcolor{CBB!61} \color{white} 10 & \cellcolor{CBB!40} \color{black} -0.00045 & \cellcolor{CBB!40} \color{black} 15 & \cellcolor{CBB!52} \color{black} 0.03811 & \cellcolor{CBB!52} \color{black} 12 & \textcolor{CBR}{\textbf{$\downarrow$ 5}} & \textcolor{CBR}{\textbf{$\downarrow$ 2}} \\
A zoomed-out photo of a \{\}. & \cellcolor{CBB!57} \color{black} 0.8310 & \cellcolor{CBB!57} \color{black} 11 & \cellcolor{CBB!57} \color{black} 0.00009 & \cellcolor{CBB!57} \color{black} 11 & \cellcolor{CBB!57} \color{black} 0.03839 & \cellcolor{CBB!57} \color{black} 11 & \textcolor{CBGy}{$-$} & \textcolor{CBGy}{$-$} \\
A photo of a \{\} in its natural environment. & \cellcolor{CBB!52} \color{black} 0.8261 & \cellcolor{CBB!52} \color{black} 12 & \cellcolor{CBB!44} \color{black} -0.00033 & \cellcolor{CBB!44} \color{black} 14 & \cellcolor{CBB!44} \color{black} 0.03759 & \cellcolor{CBB!44} \color{black} 14 & \textcolor{CBR}{\textbf{$\downarrow$ 2}} & \textcolor{CBR}{\textbf{$\downarrow$ 2}} \\
A photo of a \{\}, with a watermark. & \cellcolor{CBB!48} \color{black} 0.8250 & \cellcolor{CBB!48} \color{black} 13 & \cellcolor{CBB!31} \color{black} -0.00142 & \cellcolor{CBB!31} \color{black} 17 & \cellcolor{CBB!35} \color{black} 0.03657 & \cellcolor{CBB!35} \color{black} 16 & \textcolor{CBR}{\textbf{$\downarrow$ 4}} & \textcolor{CBR}{\textbf{$\downarrow$ 3}} \\
A sketch of a \{\}. & \cellcolor{CBB!44} \color{black} 0.8233 & \cellcolor{CBB!44} \color{black} 14 & \cellcolor{CBB!52} \color{black} -0.00001 & \cellcolor{CBB!52} \color{black} 12 & \cellcolor{CBB!48} \color{black} 0.03768 & \cellcolor{CBB!48} \color{black} 13 & \textcolor{CBG}{\textbf{$\uparrow$ 2}} & \textcolor{CBG}{\textbf{$\uparrow$ 1}} \\
A photo focusing closely on the details of a \{\}. & \cellcolor{CBB!40} \color{black} 0.8201 & \cellcolor{CBB!40} \color{black} 15 & \cellcolor{CBB!35} \color{black} -0.00092 & \cellcolor{CBB!35} \color{black} 16 & \cellcolor{CBB!31} \color{black} 0.03642 & \cellcolor{CBB!31} \color{black} 17 & \textcolor{CBR}{\textbf{$\downarrow$ 1}} & \textcolor{CBR}{\textbf{$\downarrow$ 2}} \\
A macro shot of a \{\}. & \cellcolor{CBB!35} \color{black} 0.8195 & \cellcolor{CBB!35} \color{black} 16 & \cellcolor{CBB!74} \color{white} 0.00126 & \cellcolor{CBB!74} \color{white} 7 & \cellcolor{CBB!61} \color{white} 0.03869 & \cellcolor{CBB!61} \color{white} 10 & \textcolor{CBG}{\textbf{$\uparrow$ 9}} & \textcolor{CBG}{\textbf{$\uparrow$ 6}} \\
A photo focusing on the shape of a \{\}. & \cellcolor{CBB!31} \color{black} 0.8168 & \cellcolor{CBB!31} \color{black} 17 & \cellcolor{CBB!22} \color{black} -0.00225 & \cellcolor{CBB!22} \color{black} 19 & \cellcolor{CBB!27} \color{black} 0.03494 & \cellcolor{CBB!27} \color{black} 18 & \textcolor{CBR}{\textbf{$\downarrow$ 2}} & \textcolor{CBR}{\textbf{$\downarrow$ 1}} \\
A painting of a \{\}. & \cellcolor{CBB!27} \color{black} 0.8141 & \cellcolor{CBB!27} \color{black} 18 & \cellcolor{CBB!14} \color{black} -0.00247 & \cellcolor{CBB!14} \color{black} 21 & \cellcolor{CBB!22} \color{black} 0.03464 & \cellcolor{CBB!22} \color{black} 19 & \textcolor{CBR}{\textbf{$\downarrow$ 3}} & \textcolor{CBR}{\textbf{$\downarrow$ 1}} \\
A photo of a \{\}, cluttered background. & \cellcolor{CBB!22} \color{black} 0.8108 & \cellcolor{CBB!22} \color{black} 19 & \cellcolor{CBB!18} \color{black} -0.00236 & \cellcolor{CBB!18} \color{black} 20 & \cellcolor{CBB!18} \color{black} 0.03443 & \cellcolor{CBB!18} \color{black} 20 & \textcolor{CBR}{\textbf{$\downarrow$ 1}} & \textcolor{CBR}{\textbf{$\downarrow$ 1}} \\
An unusual, hard to recognize photo of a \{\}. & \cellcolor{CBB!18} \color{black} 0.8086 & \cellcolor{CBB!18} \color{black} 20 & \cellcolor{CBB!10} \color{black} -0.00344 & \cellcolor{CBB!10} \color{black} 22 & \cellcolor{CBB!14} \color{black} 0.03313 & \cellcolor{CBB!14} \color{black} 21 & \textcolor{CBR}{\textbf{$\downarrow$ 2}} & \textcolor{CBR}{\textbf{$\downarrow$ 1}} \\
A black and white photo of a \{\}. & \cellcolor{CBB!14} \color{black} 0.7999 & \cellcolor{CBB!14} \color{black} 21 & \cellcolor{CBB!70} \color{white} 0.00097 & \cellcolor{CBB!70} \color{white} 8 & \cellcolor{CBB!40} \color{black} 0.03676 & \cellcolor{CBB!40} \color{black} 15 & \textcolor{CBG}{\textbf{$\uparrow$ 13}} & \textcolor{CBG}{\textbf{$\uparrow$ 6}} \\
A hand-drawn illustration of a \{\}. & \cellcolor{CBB!10} \color{black} 0.7863 & \cellcolor{CBB!10} \color{black} 22 & \cellcolor{CBB!27} \color{black} -0.00198 & \cellcolor{CBB!27} \color{black} 18 & \cellcolor{CBB!10} \color{black} 0.03286 & \cellcolor{CBB!10} \color{black} 22 & \textcolor{CBG}{\textbf{$\uparrow$ 4}} & \textcolor{CBGy}{$-$} \\
\bottomrule
\end{tabular}
}
\end{table}

\begin{table}[t]
\setlength{\aboverulesep}{0pt}
\setlength{\belowrulesep}{0pt}
\setlength{\extrarowheight}{0.8pt}
\centering
\caption{Prompt ensemble evaluation for LLM zero-shot multiclass image classification (IMAGENET-A). Each prompt corresponds to a different model of $\mathcal{C}$. All values and rankings are computed on the validation split $\dataSetVal$. Columns $\Delta_{\banz}$ and $\Delta_{\shap}$ display the difference in the validation rankings induced by the power indices $\banz(G_{\mathsf{MC}})$ and $\shap(G_{\mathsf{MC}})$ with respect to validation accuracy.}
\label{tab:results_imagenet-a}
\definecolor{CBB}{HTML}{0072B2} %
\definecolor{CBG}{HTML}{009E73} %
\definecolor{CBR}{HTML}{D55E00} %
\definecolor{CBGy}{HTML}{999999} %
\resizebox{\textwidth}{!}{
\begin{tabular}{l rr rr rr cc}
\toprule
\multirow{2}{*}{\textbf{Prompt template}} & \multicolumn{2}{c}{\textbf{Accuracy}} & \multicolumn{2}{c}{\textbf{Banzhaf}} & \multicolumn{2}{c}{\textbf{Shapley}} & \multirow{2}{*}{$\Delta_{\banz}$} & \multirow{2}{*}{$\Delta_{\shap}$} \\
\cmidrule(lr){2-3} \cmidrule(lr){4-5} \cmidrule(lr){6-7}
& Value & Rank & Value & Rank & Value & Rank & & \\
\midrule
A poor quality, grainy photo of a \{\}. & \cellcolor{CBB!100} \color{white} 0.3136 & \cellcolor{CBB!100} \color{white} \textbf{1} & \cellcolor{CBB!100} \color{white} 0.00359 & \cellcolor{CBB!100} \color{white} \textbf{1} & \cellcolor{CBB!100} \color{white} 0.01936 & \cellcolor{CBB!100} \color{white} \textbf{1} & \textcolor{CBGy}{$-$} & \textcolor{CBGy}{$-$} \\
A pixelated photo of a \{\}. & \cellcolor{CBB!95} \color{white} 0.3029 & \cellcolor{CBB!95} \color{white} \textbf{2} & \cellcolor{CBB!91} \color{white} 0.00242 & \cellcolor{CBB!91} \color{white} \textbf{3} & \cellcolor{CBB!91} \color{white} 0.01738 & \cellcolor{CBB!91} \color{white} \textbf{3} & \textcolor{CBR}{\textbf{$\downarrow$ 1}} & \textcolor{CBR}{\textbf{$\downarrow$ 1}} \\
A photo of a \{\}, with a watermark. & \cellcolor{CBB!91} \color{white} 0.2989 & \cellcolor{CBB!91} \color{white} \textbf{3} & \cellcolor{CBB!70} \color{white} 0.00096 & \cellcolor{CBB!70} \color{white} 8 & \cellcolor{CBB!78} \color{white} 0.01581 & \cellcolor{CBB!78} \color{white} 6 & \textcolor{CBR}{\textbf{$\downarrow$ 5}} & \textcolor{CBR}{\textbf{$\downarrow$ 3}} \\
A blurry photo of a \{\}. & \cellcolor{CBB!87} \color{white} 0.2989 & \cellcolor{CBB!87} \color{white} 4 & \cellcolor{CBB!87} \color{white} 0.00184 & \cellcolor{CBB!87} \color{white} 4 & \cellcolor{CBB!87} \color{white} 0.01656 & \cellcolor{CBB!87} \color{white} 4 & \textcolor{CBGy}{$-$} & \textcolor{CBGy}{$-$} \\
An unusual, hard to recognize photo of a \{\}. & \cellcolor{CBB!82} \color{white} 0.2981 & \cellcolor{CBB!82} \color{white} 5 & \cellcolor{CBB!95} \color{white} 0.00295 & \cellcolor{CBB!95} \color{white} \textbf{2} & \cellcolor{CBB!95} \color{white} 0.01750 & \cellcolor{CBB!95} \color{white} \textbf{2} & \textcolor{CBG}{\textbf{$\uparrow$ 3}} & \textcolor{CBG}{\textbf{$\uparrow$ 3}} \\
A heavily compressed JPEG photo of a \{\}. & \cellcolor{CBB!78} \color{white} 0.2949 & \cellcolor{CBB!78} \color{white} 6 & \cellcolor{CBB!82} \color{white} 0.00154 & \cellcolor{CBB!82} \color{white} 5 & \cellcolor{CBB!82} \color{white} 0.01583 & \cellcolor{CBB!82} \color{white} 5 & \textcolor{CBG}{\textbf{$\uparrow$ 1}} & \textcolor{CBG}{\textbf{$\uparrow$ 1}} \\
A zoomed-out photo of a \{\}. & \cellcolor{CBB!74} \color{white} 0.2893 & \cellcolor{CBB!74} \color{white} 7 & \cellcolor{CBB!61} \color{white} 0.00066 & \cellcolor{CBB!61} \color{white} 10 & \cellcolor{CBB!70} \color{white} 0.01468 & \cellcolor{CBB!70} \color{white} 8 & \textcolor{CBR}{\textbf{$\downarrow$ 3}} & \textcolor{CBR}{\textbf{$\downarrow$ 1}} \\
A cropped photo of a \{\}. & \cellcolor{CBB!70} \color{white} 0.2888 & \cellcolor{CBB!70} \color{white} 8 & \cellcolor{CBB!65} \color{white} 0.00071 & \cellcolor{CBB!65} \color{white} 9 & \cellcolor{CBB!65} \color{white} 0.01463 & \cellcolor{CBB!65} \color{white} 9 & \textcolor{CBR}{\textbf{$\downarrow$ 1}} & \textcolor{CBR}{\textbf{$\downarrow$ 1}} \\
A hand-drawn illustration of a \{\}. & \cellcolor{CBB!65} \color{white} 0.2872 & \cellcolor{CBB!65} \color{white} 9 & \cellcolor{CBB!78} \color{white} 0.00137 & \cellcolor{CBB!78} \color{white} 6 & \cellcolor{CBB!74} \color{white} 0.01501 & \cellcolor{CBB!74} \color{white} 7 & \textcolor{CBG}{\textbf{$\uparrow$ 3}} & \textcolor{CBG}{\textbf{$\uparrow$ 2}} \\
A photo of a \{\}, with text over it. & \cellcolor{CBB!61} \color{white} 0.2869 & \cellcolor{CBB!61} \color{white} 10 & \cellcolor{CBB!57} \color{black} 0.00049 & \cellcolor{CBB!57} \color{black} 11 & \cellcolor{CBB!61} \color{white} 0.01428 & \cellcolor{CBB!61} \color{white} 10 & \textcolor{CBR}{\textbf{$\downarrow$ 1}} & \textcolor{CBGy}{$-$} \\
A photo of a \{\}. & \cellcolor{CBB!57} \color{black} 0.2853 & \cellcolor{CBB!57} \color{black} 11 & \cellcolor{CBB!40} \color{black} -0.00027 & \cellcolor{CBB!40} \color{black} 15 & \cellcolor{CBB!52} \color{black} 0.01353 & \cellcolor{CBB!52} \color{black} 12 & \textcolor{CBR}{\textbf{$\downarrow$ 4}} & \textcolor{CBR}{\textbf{$\downarrow$ 1}} \\
A photo of a \{\}, cluttered background. & \cellcolor{CBB!52} \color{black} 0.2840 & \cellcolor{CBB!52} \color{black} 12 & \cellcolor{CBB!35} \color{black} -0.00035 & \cellcolor{CBB!35} \color{black} 16 & \cellcolor{CBB!48} \color{black} 0.01337 & \cellcolor{CBB!48} \color{black} 13 & \textcolor{CBR}{\textbf{$\downarrow$ 4}} & \textcolor{CBR}{\textbf{$\downarrow$ 1}} \\
A photo of a \{\} in its natural environment. & \cellcolor{CBB!48} \color{black} 0.2797 & \cellcolor{CBB!48} \color{black} 13 & \cellcolor{CBB!74} \color{white} 0.00116 & \cellcolor{CBB!74} \color{white} 7 & \cellcolor{CBB!57} \color{black} 0.01420 & \cellcolor{CBB!57} \color{black} 11 & \textcolor{CBG}{\textbf{$\uparrow$ 6}} & \textcolor{CBG}{\textbf{$\uparrow$ 2}} \\
A painting of a \{\}. & \cellcolor{CBB!44} \color{black} 0.2779 & \cellcolor{CBB!44} \color{black} 14 & \cellcolor{CBB!48} \color{black} -0.00000 & \cellcolor{CBB!48} \color{black} 13 & \cellcolor{CBB!44} \color{black} 0.01309 & \cellcolor{CBB!44} \color{black} 14 & \textcolor{CBG}{\textbf{$\uparrow$ 1}} & \textcolor{CBGy}{$-$} \\
A photo focusing on the shape of a \{\}. & \cellcolor{CBB!40} \color{black} 0.2776 & \cellcolor{CBB!40} \color{black} 15 & \cellcolor{CBB!31} \color{black} -0.00043 & \cellcolor{CBB!31} \color{black} 17 & \cellcolor{CBB!35} \color{black} 0.01261 & \cellcolor{CBB!35} \color{black} 16 & \textcolor{CBR}{\textbf{$\downarrow$ 2}} & \textcolor{CBR}{\textbf{$\downarrow$ 1}} \\
A high-resolution, crisp photo of a \{\}. & \cellcolor{CBB!35} \color{black} 0.2765 & \cellcolor{CBB!35} \color{black} 16 & \cellcolor{CBB!18} \color{black} -0.00136 & \cellcolor{CBB!18} \color{black} 20 & \cellcolor{CBB!27} \color{black} 0.01162 & \cellcolor{CBB!27} \color{black} 18 & \textcolor{CBR}{\textbf{$\downarrow$ 4}} & \textcolor{CBR}{\textbf{$\downarrow$ 2}} \\
A black and white photo of a \{\}. & \cellcolor{CBB!31} \color{black} 0.2739 & \cellcolor{CBB!31} \color{black} 17 & \cellcolor{CBB!52} \color{black} 0.00024 & \cellcolor{CBB!52} \color{black} 12 & \cellcolor{CBB!40} \color{black} 0.01289 & \cellcolor{CBB!40} \color{black} 15 & \textcolor{CBG}{\textbf{$\uparrow$ 5}} & \textcolor{CBG}{\textbf{$\uparrow$ 2}} \\
A sketch of a \{\}. & \cellcolor{CBB!27} \color{black} 0.2704 & \cellcolor{CBB!27} \color{black} 18 & \cellcolor{CBB!44} \color{black} -0.00004 & \cellcolor{CBB!44} \color{black} 14 & \cellcolor{CBB!31} \color{black} 0.01240 & \cellcolor{CBB!31} \color{black} 17 & \textcolor{CBG}{\textbf{$\uparrow$ 4}} & \textcolor{CBG}{\textbf{$\uparrow$ 1}} \\
A photo focusing on the texture of a \{\}. & \cellcolor{CBB!22} \color{black} 0.2579 & \cellcolor{CBB!22} \color{black} 19 & \cellcolor{CBB!27} \color{black} -0.00080 & \cellcolor{CBB!27} \color{black} 18 & \cellcolor{CBB!22} \color{black} 0.01048 & \cellcolor{CBB!22} \color{black} 19 & \textcolor{CBG}{\textbf{$\uparrow$ 1}} & \textcolor{CBGy}{$-$} \\
A photo focusing closely on the details of a \{\}. & \cellcolor{CBB!18} \color{black} 0.2544 & \cellcolor{CBB!18} \color{black} 20 & \cellcolor{CBB!14} \color{black} -0.00180 & \cellcolor{CBB!14} \color{black} 21 & \cellcolor{CBB!18} \color{black} 0.00934 & \cellcolor{CBB!18} \color{black} 20 & \textcolor{CBR}{\textbf{$\downarrow$ 1}} & \textcolor{CBGy}{$-$} \\
A close-up photo of a \{\}. & \cellcolor{CBB!14} \color{black} 0.2365 & \cellcolor{CBB!14} \color{black} 21 & \cellcolor{CBB!10} \color{black} -0.00218 & \cellcolor{CBB!10} \color{black} 22 & \cellcolor{CBB!10} \color{black} 0.00727 & \cellcolor{CBB!10} \color{black} 22 & \textcolor{CBR}{\textbf{$\downarrow$ 1}} & \textcolor{CBR}{\textbf{$\downarrow$ 1}} \\
A macro shot of a \{\}. & \cellcolor{CBB!10} \color{black} 0.2349 & \cellcolor{CBB!10} \color{black} 22 & \cellcolor{CBB!22} \color{black} -0.00134 & \cellcolor{CBB!22} \color{black} 19 & \cellcolor{CBB!14} \color{black} 0.00808 & \cellcolor{CBB!14} \color{black} 21 & \textcolor{CBG}{\textbf{$\uparrow$ 3}} & \textcolor{CBG}{\textbf{$\uparrow$ 1}} \\
\bottomrule
\end{tabular}
}
\end{table}

\subsection{Synthetic experiments}\label{appsubsec:synth}
Real data does not allow us to control the parameters that determine the runtime of \suite. We therefore generate synthetic instances in which we vary the number $|\unique|$ of unique vote profiles, the ensemble size $\cardM$ and the number of classes $\totLabs$. Recall that the runtime of our algorithms depends on $|\unique|$ and not on $|\dataSet|$. Finally, for the Shapley values we will also consider the runtime our refined algorithm \suiteShapImp from \Cref{app:refinedComplexity}.

\textbf{Data generation.} A vote profile is determined by two integers: the number $\cardM_1$ of classifiers predicting the correct class, and the number $r$ of incorrect labels receiving at least one vote.
We allocate the $\cardM-\cardM_1$ incorrect votes evenly over $r$ labels, i.e., $(\cardM-\cardM_1)\bmod r$ labels receive $\lceil(\cardM-\cardM_1)/r\rceil$ votes and the other ones receive $\lfloor(\cardM-\cardM_1)/r\rfloor$ votes.
There are $\totLabs-1$ incorrect labels and $\cardM-\cardM_1$ incorrect votes,
thus $r\le r_{\max}(\cardM_1)\doteq\min\{\totLabs-1,\cardM-\cardM_1\}$.
We consider $1\le r\le r_{\max}(\cardM_1)$ if $\cardM_1<\cardM$, and $r=0$ if $\cardM_1=\cardM$.
Each distinct pair $(\cardM_1,r)$ yields a distinct profile.
Given $p\in(0,1)$, we build an instance of $|\unique|$ profiles as follows.
\begin{compactenum}
	\item Let $\cardM_1^\star\doteq\lfloor p\cardM+1/2\rfloor$, i.e., $p\cardM$ rounded to the nearest integer.
	\item Order the values $\cardM_1\in\{0,\dots,\cardM\}$ as the sequence $\cardM_1^\star,\ \cardM_1^\star-1,\ \cardM_1^\star+1,\ \cardM_1^\star-2,\ \cardM_1^\star+2,\dots$, i.e., in increasing order of $|\cardM_1-\cardM_1^\star|$ and, breaking ties in favor of the smallest value.
	\item Build a list $\mathcal{U}$ of profiles by replacing each $\cardM_1$ in this sequence with the profiles $(\cardM_1,r)$ in decreasing order of $r$, i.e., $(\cardM_1,r_{\max}(\cardM_1)),(\cardM_1,r_{\max}(\cardM_1)-1),\dots,(\cardM_1,1)$. The value $\cardM_1=\cardM$ is replaced with the single profile $(\cardM,0)$.
	\item An instance consists of the first $|\unique|$ profiles of such a list (i.e., of $\mathcal{U}$).
\end{compactenum}
Hence the fraction of correct models $\cardM_1/\cardM$ is close to the parameter $p$. 
Furthermore, for each $\cardM_1$ the profiles that assign the incorrect votes over most of the incorrect labels are generated first.
Note that when the incorrect votes are distributed evenly, many coalitions result in a tie, increasing the runtime of our algorithms for $\charMC(\cdot)$.
For the binary case it holds $r_{\max}(\cardM_1)\le 1$, i.e., a profile is determined by $\cardM_1$ uniquely.
We vary $p\in\{0.1,0.5,0.9\}$, i.e., representing ensembles of base classifiers with low, medium and high accuracy.
In the binary case we vary $\cardM\in\{100, 300, 1\,000, 3\,000, 10\,000\}$ and $|\unique|\in\{10,20,40,80,160\}$.
In the multiclass case we vary $\cardM\in\{10,25,50,100\}$, $\totLabs\in\{2,5,10,20,50,100, 1000\}$ and $|\unique|\in\{10,10^2,10^3\}$.
We do not consider configurations with less than the fixed number $|\unique|$ of profiles.
We run each instance five times and divide its median runtime by $|\unique|$.
For each $(p,\cardM,\totLabs)$ we report the median of these values over $|\unique|$, which we denote as the runtime per unique profile.
The time complexity of our methods is linear in $|\unique|$, thus dividing by $|\unique|$ isolates the influence of all other parameters on the runtime.

\textbf{Results.} \Cref{fig:synth} reports the runtime per unique profile in log-log scale, where each  row corresponds to a different value of $p$.
In each panel the horizontal axis corresponds to the bound on the time complexity of the corresponding theorem (i.e., according to the setting and power index considered).
That is
\begin{compactitem}
	\item the left column reports $\suite$-$\banz$ and $\suite$-$\shap$ for the binary classification setting ($\totLabs=2$) as a function of $\cardM$ (\Cref{theo:runtimeBin});
	\item the middle column reports $\suite$-$\banz$ in the multiclass classification setting as a function of $\cardM\totLabs^2$ (\Cref{theo:runtimeMCBZ}); and
	\item the right column reports $\suite$-$\shap$ (circles, solid lines) and \suiteShapImp (triangles, dashed lines) as a function of $\cardM^3\totLabs$ (\Cref{theo:runtimeMCSHAP}).
\end{compactitem}
Regarding the figures concerning the multiclass case each line connects the values of $\cardM$ for a fixed $\totLabs$, colored by the value of $\totLabs$.
The dotted line in each panel corresponds to theoretical bound.
In the binary classification setting, the runtime per unique profile of \suite grows linearly with  $\cardM$ (with a constant slightly smaller than 1), as predicted by~\Cref{theo:runtimeBin}.
In the multiclass setting for $\suite$-$\banz$ for a fixed value of $\totLabs$ the runtime scales with $\cardM$ almost linearly as from the bound in \Cref{theo:runtimeMCBZ}.
The only setting where the runtime appears to be slightly larger than linear in $\cardM$, is for $p=0.1$ and a high value of $\totLabs$, where the number of incorrect labels approaches $\cardM$.
Interestingly, we observe that increasing $\totLabs$ for a fixed $\cardM$ does not result in the $\bigO(\totLabs^2)$ of our bound, that is our bounds can be likely improved.
Similarly, the theoretical runtime of $\suite$-$\shap$ and \suiteShapImp for large values of $\cardM^3\totLabs$ is likely pessimistic given that in practice our methods are much more efficient.
Finally, our method \suiteShapImp is always more efficient than $\suite$-$\shap$, as captured by our analysis (see \Cref{app:refinedComplexity}).

\begin{figure}[t]
	\centering
	\includegraphics[width=0.2\textwidth]{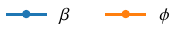}
	\includegraphics[width=0.5\textwidth]{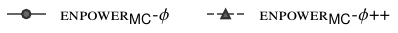}\\
	\includegraphics[width=0.9\textwidth]{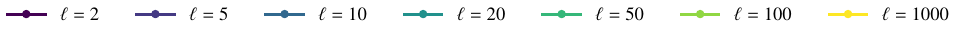}
	\\
	\includegraphics[width=\textwidth]{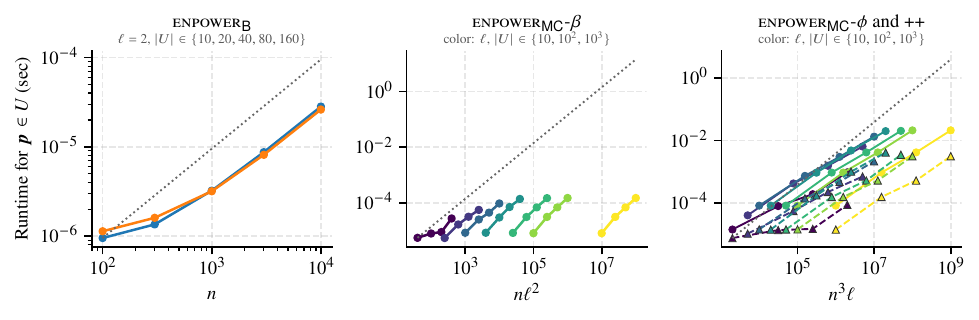}\\
	\includegraphics[width=\textwidth]{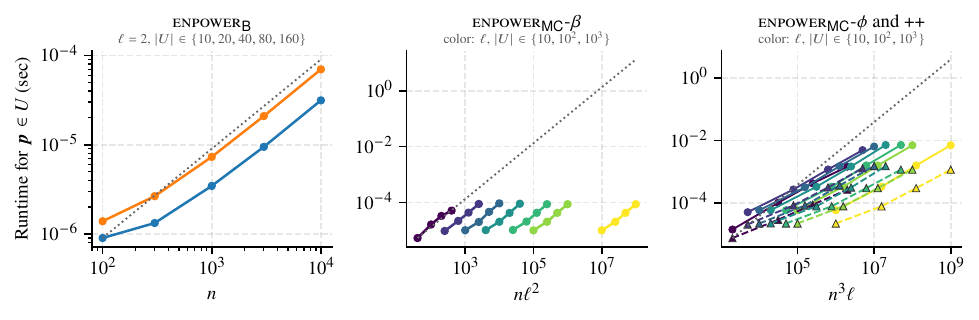}\\
	\includegraphics[width=\textwidth]{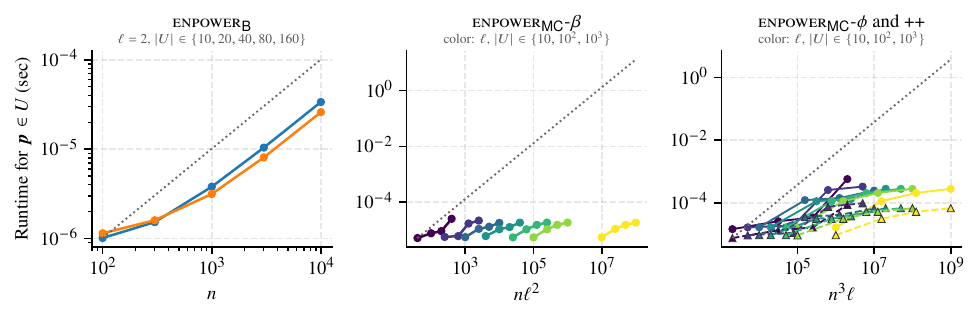}
	\caption{Runtime to process a unique vote profile on synthetic instances, where $p=0.1$ (top), $p=0.5$ (middle) and $p=0.9$ (bottom). The y-axis reports the measured time to compute the power indices under a vote profile $\VotVec\in\unique$ with $\cardM$ players and $\totLabs$ labels. See \Cref{appsubsec:synth}  for a detailed description of each panel.}
	\label{fig:synth}
\end{figure}

\clearpage
\section{A refined multiclass Shapley value computation}\label{app:refinedComplexity}

\Cref{alg:ShapMC} can be implemented efficiently, improving over the time complexity bound of \Cref{theo:runtimeMCSHAP}.
In this section we detail such a procedure.
The key technical idea is that, for any fixed $h\ge 1$, the matrix $P_h$ holding the coefficients of $P_h(\gfvar,u)$ (i.e., the matrix whose entry $P_h[k,m]$ corresponds to the coefficient $[\gfvar^ku^m]P_h(\gfvar,u)$), does not necessarily require to be of size of $\totLabs$ rows and $\cardM-\cardM_1+1$ columns as in \Cref{alg:ShapMC}.
The definition of $\coalset(k,h,m)$ yields $|\coalset(k,h,m)|=0$ for most of the triplets $k,h,m$, and by allocating a matrix that retains every non-zero coefficient of $P_h(\gfvar,u)$,
we will improve the time complexity of the resulting algorithm.
More precisely, fix $h\ge1$ and recall from~\Cref{lemma:bivariateGenFunction} that $P_h(\gfvar,u)=\prod_{j=2}^{\totLabs}(A_j(h;u)+B_j(h;u)\gfvar)$, where $B_j(h;u)=\binom{\cardM_j}{h}u^h$.
In each factor, the variable $\gfvar$ multiplies the single monomial $B_j(h;u)$, and each such monomial contains \emph{exactly} variable $u^h$.
A monomial $\gfvar^ku^m$ of $P_h(\gfvar,u)$ therefore collects a factor $u^h$ from each of its $k$ factors that have term $B_j(h;u)$ contributing to $\gfvar^k$, thus $[\gfvar^ku^m]P_h(\gfvar,u)=0$ when $m<kh$.
In other words, each monomial $\gfvar^ku^m$  of $P_h(\gfvar,u)$ with a non-zero coefficient, must have $m\ge kh$.
We use this insight to reduce the computation performed by $\suite_{\mathsf{MC}}$-$\shap$, as it guarantees that the entries $P_h[k,m]$ with $m<kh$ are never visited by the algorithm.
We formalize our argument in~\Cref{lemma:support}\ifarxiv, and in~\Cref{appsubsec:fft} we further exploit such insights to remove factor $u^h$ from $B_j(h;u)$, recovering the coefficients of $P_h(\gfvar,u)$ by their degree in $\gfvar$\fi.

We express our bounds as functions of the vote profile of a fixed data point from $\dataSet$.
Throughout we fix a test data point $\evz=(\xp,\y)\in\dataSet$ and we let $\y=\y_1$, as by our convention in~\Cref{subsubsec:MCSHAP}.
Furthermore, let $N\doteq\cardM-\cardM_1=|\modelSet'|$ be the number of incorrect classifiers from $\modelSet$, and we let
\[
\cardM_{(1)}\ge \cardM_{(2)}\ge \dots \ge \cardM_{(\totLabs-1)}
\]
denote the values $\{\cardM_j\}_{j\in[2,\totLabs]}$ sorted in non-increasing order.

\begin{definition}\label{def:profileFunctions}
	For $h\in \mathbb{Z}$ we let
	\[
	K(h)\doteq\big|\{j\in[2,\totLabs] : \cardM_j\ge h\}\big|,\ h\ge 1,
	\qquad\text{and}\qquad
	\Phi(h)\doteq\sum_{j=2}^{\totLabs}\min\{\cardM_j,h\},\ h\ge 0 \enspace.
	\]
	In particular, $\Phi(0)=0$.
	Furthermore, for each $j\in[2,\totLabs]$ and $h\ge1$ we let
	\[
	K_{-j}(h)\doteq K(h)-\ind[\cardM_j\ge h] \enspace .
	\]
\end{definition}

The value $K(h)$ is the number of incorrect labels receiving at least $h\ge1$ votes, while $\Phi(h)$ corresponds to the number of classifiers obtained by retaining, for each incorrect label, at most $h$ of the classifiers predicting it.
The value $K_{-j}(h)$ is the number of incorrect labels other than $\y_j$ receiving at least $h$ votes.
We have $K(h)\le \totLabs-1$ and $\Phi(h)\le N$.
For a fixed $h\ge1$ we compute the matrix $P_h$ with $K(h)+1$ rows and $\Phi(h)+1$ columns, in place of the $\totLabs$ rows and $\cardM-\cardM_1+1$ columns of~\Cref{alg:ShapMC}.
We will show that all the remaining entries are zero, thus the refined algorithm retains the correctness of \Cref{alg:ShapMC}.
Such a matrix will still contain the entries $P_h[k,m]$ with $m<kh$, which are zero as observed above, for simplicity of the analysis, simplifying our main result presented in \Cref{theo:refinedRuntimeMCSHAP}.\footnote{Our implementation avoids building the full matrix. That is we compute a matrix $P_h[k,m]$ by retaining, for the row $k$ the columns with indices $m\ge kh$. 
\ifarxiv We exploit such map of the indices in  \Cref{appsubsec:fft}, to reduce the size of the matrix $P_h$.\fi}

We let $H^\star\doteq\min\{\cardM_1,\cardM_{(1)}\}$ denote the upper bound on the range of $h$.
The following value will be fundamental for our analysis, recall that $\min\{\cardM_j, \cardM_{(1)}\} = \cardM_j$ for any $j\ge 2$, thus we let
\[
N^\star \doteq \Phi(H^\star)  
= \sum_{j=2}^{\totLabs}\min\{\cardM_j,\min\{\cardM_1, \cardM_{(1)}\}\}
=\sum_{j=2}^{\totLabs}\min\{\cardM_j,\cardM_1\}
\le \min \big\{\cardM-\cardM_1,(\totLabs-1)\cardM_1\big\} \enspace.
\]

We now show how to implement \Cref{alg:ShapMC} refining the time complexity of \Cref{theo:runtimeMCSHAP} to
\[
\bigO\big(|\unique|\big((N^\star)^3+\totLabs (N^\star)^2\big)+|\dataSet|\cardM\big) \enspace ,
\]
a result that we prove in \Cref{theo:refinedRuntimeMCSHAP} below.
The refined bound is never worse than the bound of \Cref{theo:runtimeMCSHAP}, for every vote profile.

Consider the following two settings:
\begin{inparaenum}[($a$)]
	\item when all but a constant number of classifiers are correct, i.e., $\cardM-\cardM_1 = \bigO(1)$, we have $N^\star \le N =\bigO(1)$, thus the time complexity reduces to $\bigO\big(|\unique| \totLabs+|\dataSet|\cardM\big)$;
	\item when the correct label receives a constant number of votes, i.e., $\cardM_1 = \bigO(1)$, we have $N^\star = \bigO(\min\{\cardM, \totLabs\})$, thus the time complexity reduces to $\bigO\big(|\unique| \totLabs \min\{\cardM, \totLabs\}^2+|\dataSet|\cardM\big)$, as $\min\{\cardM, \totLabs\}^3 \le \totLabs \min\{\cardM, \totLabs\}^2$,
	that is we do not incur a factor $\bigO(\cardM^3 \totLabs)$ as in \Cref{theo:runtimeMCSHAP}.
\end{inparaenum}

We now introduce the modifications to \Cref{alg:ShapMC}, restricting the size of $P_h$ and enabling our improved time complexity.

\subsection{\suiteShapImp: an improved algorithm}\label{appsubsec:modifiedAlg}

We now list the modifications to~\Cref{alg:ShapMC} to obtain our improved algorithm named \suiteShapImp.

First we need some definitions, for $k\in\{1,\dots,\totLabs-1\}$ and $h\ge 1$, we let
\[
\overline{m}(k,h) \doteq \min\big\{N, k+\Phi(h-1)\big\} \enspace.
\]
Similarly, for each $j\in\{2,\dots,\totLabs\}$, $k\in\{0,\dots, \totLabs-2\}$ and $h\ge1$ we let
\[
\overline{m}_{-j}(k,h) \doteq \min\{N-\cardM_j,k+\Phi(h-1)-\min\{\cardM_j,h-1\}\}\enspace,
\]
where $N-\cardM_j = \cardM-(\cardM_1+\cardM_j)$ is the number of classifiers predicting a label in $\labelSet-\{\y_1, \y_j\}$.

We are now ready to introduce \suiteShapImp.
The algorithm exploits \Cref{def:profileFunctions} to allocate a smaller matrix $P_h$ and to adjust the index boundaries of the loops performed.
The body of all methods, including \texttt{RemoveClass} and \texttt{AddClass}, remains unaltered, except for the modifications that we now list, each followed by the statement certifying its correctness.
\begin{compactitem}
	\item After~\Cref{alglineSHMC:9}, after computing the counts $\cardM_j$ for $j\in[2,\totLabs]$, compute $\cardM_{(1)}=\max_{j\in[2,\totLabs]}\cardM_j$, the value $H^\star=\min\{\cardM_1,\cardM_{(1)}\}$, and the values $K(h)$ and $\Phi(h)$ for $h\in[1,H^\star]$ from \Cref{def:profileFunctions}.
	The counts $\cardM_j$ are integers and the algorithm uses only $h\le H^\star$. Thus, we can use bucket sort to compute the array $c[v]\doteq|\{j\in[2,\totLabs] : \cardM_j = v\}|$ for $v\in[0,H^\star-1]$, and the count $c_{\ge}\doteq|\{j\in[2,\totLabs] : \cardM_j\ge H^\star\}|$, in $\bigO(\totLabs+H^\star)$ time.
	Compute $K(h)=c_{\ge}+\sum_{v=h}^{H^\star-1}c[v]$  for every $h\in[1,H^\star]$ by a single scan of $c$ from $v=H^\star-1$ down to $v=1$ accumulating the suffix sums. 
	Furthermore, compute $\Phi(h)=\Phi(h-1)+K(h)$ from $K(h)$ and $\Phi(0)=0$ by~\Cref{lemma:profileIdentities}~($i$).
	Since $H^\star\le N^\star$, the overall time complexity is dominated by $\bigO(\totLabs+N^\star)$ per vote profile.

	\item In~\Cref{alglineSHMC:12}, replace the upper limit $\cardM_1$ of the loop over $h$ with $H^\star=\min\{\cardM_1,\cardM_{(1)}\}$, by~\Cref{cor:restrictedRanges}~($i$).

	\item In~\Cref{alglineSHMC:13}, allocate $P_h$ with $K(h)+1$ rows and $\Phi(h)+1$ columns, instead of $\totLabs$ rows and $\cardM-\cardM_1+1$ columns, by~\Cref{lemma:profileIdentities}~($iii$) and~($iv$).

	\item In \Cref{alglineSHMC:14}, replace the upper limit $\totLabs-1$ of the loop over $k$ with $K(h)$, by \Cref{lemma:profileIdentities}~($iii$).

	\item In \Cref{alglineSHMC:15}, replace the upper limit $\cardM-\cardM_1$ of the loop over $m$ with $\min\{\overline{m}(k,h),\Phi(h)\}$, by \Cref{cor:restrictedRanges}~($ii$) and \Cref{lemma:profileIdentities}~($iv$).

	\item In \Cref{alglineSHMC:17}, do not perform the iteration for a label $\y_j$ with $\cardM_j<h$, that is, process $\y_j$ only for $h\le\min\{H^\star,\cardM_j\}$, by \Cref{cor:restrictedRanges}~($i$).
	By simply processing the keys of $J_{\text{\emph{unique}}}$ in non-increasing order of $\cardM_j$, we stop the loop at the first key with $\cardM_j<h$.
	Such an order is obtained in $\bigO(\totLabs+H^\star)$ time from the array $c$ computed above.

	\item In \Cref{alglineSHMC:20}, replace the upper limit $\totLabs-2$ of the loop over $k$ with $K_{-j}(h)$, by \Cref{lemma:profileIdentities}~($v$).

	\item In \Cref{alglineSHMC:21}, replace the upper limit $\cardM-(\cardM_1+\cardM_j)$ of the loop over $m$ with $\min\{\overline{m}_{-j}(k,h),\Phi(h)\}$, by \Cref{cor:restrictedRanges}~($iii$) and \Cref{lemma:profileIdentities}~($v$).
\end{compactitem}
All above modifications, avoid the iteration over entries of $P_h[k,m]$ (as used in~\Cref{alg:ShapMC}) that do not contribute towards the output of~\Cref{alg:ShapMC}, as we now show.

\begin{lemma}\label{lemma:support}
	Let $h\ge 1$ and let $P_h(\gfvar,u)$ be the generating function of~\Cref{lemma:bivariateGenFunction}.
	\begin{enumerate}[($i$)]
		\item For every $k\in\{1,\dots,\totLabs-1\}$ and every $m\ge0$,
		\[
		[\gfvar^ku^m]P_h(\gfvar,u) = 0
		\quad\text{whenever}\quad
		m < kh
		\ \text{ or }\
		m > \overline{m}(k,h) \enspace.
		\]
		\item For every $j\in[2,\totLabs]$, every $k\in\{0,\dots,\totLabs-2\}$ and every $m\ge0$,
		\[
		[\gfvar^ku^m]\frac{P_h(\gfvar,u)}{A_j(h;u)+B_j(h;u)\gfvar} = 0
		\quad\text{whenever}\quad
		m < kh
		\ \text{ or }\
		m > \overline{m}_{-j}(k,h) \enspace .
		\]
	\end{enumerate}
\end{lemma}

\begin{proof}
	($i$) Let $k\in\{1,\dots, \totLabs-1\}, h\ge 1$, by \Cref{lemma:bivariateGenFunction} it holds that $[\gfvar^ku^m]P_h(\gfvar,u)=|\coalset(k,h,m)|$, therefore it is sufficient to show that $\coalset(k,h,m)=\emptyset$ for $m< kh$ and $m>\overline{m}(k,h)$.

	Let $\coal\in\coalset(k,h,m)$, by definition we have $\coal\subseteq\modelSet'$ hence $\y_1\notin\tieSet(\coal)$, and $k =|\tieSet(\coal)|$.
	Each label of $\tieSet(\coal)$ receives exactly $h$ votes from $\coal$, thus $m\ge kh$ by the pigeonhole principle.
	Furthermore, each label $\y_j\in\tieSet(\coal)$ satisfies $\cardM_j\ge h$, hence $\min\{\cardM_j,h-1\}=h-1$, while each incorrect label $\y_j\notin\tieSet(\coal)$ receives at most $h-1$ votes and thus can have at most $\min\{\cardM_j,h-1\}$ classifiers in $\coal$, therefore
	\[
	m\le kh+\sum_{\y_j\notin\tieSet(\coal)}\min\{\cardM_j,h-1\} = kh+\Phi(h-1)-k(h-1)=k+\Phi(h-1) \enspace.
	\]
	Since $m=|\coal|\le|\modelSet'|=N$, any $\coal\in\coalset(k,h,m)$ must have $m\le \overline{m}(k,h)$.
	We have shown that $\coalset(k,h,m)=\emptyset$ for every $m$ outside $[kh,\overline{m}(k,h)]$, proving the first part of the claim.

	($ii$) Let $k\in\{0,\dots, \totLabs-2\}, h\ge 1$, and $j\in[2,\totLabs]$, the proof is analogous to the above, by replacing $\coalset(k,h,m)$ with the set $\coalset_{-j}(k,h,m)$ of~\Cref{cor:bivariateGenFunctionMinusJ}.
	Note in fact, that under such scenario $\coal\in\coalset_{-j}(k,h,m)$ can be formed by models voting for the $\totLabs-2$ incorrect labels other than $\y_j$, and thus the sum bounding $m$ above should be modified to not account for the term of label $\y_j$.
	Hence any $\coal\in\coalset_{-j}(k,h,m)$ must have $m\le N-\cardM_j$ and $m\le k+\Phi(h-1)-\min\{\cardM_j,h-1\}$, that is $m\le\overline{m}_{-j}(k,h)$ by the definition of $\overline{m}_{-j}(k,h)$, and $m\ge kh$ by the same pigeonhole argument, proving the statement.
\end{proof}

The above bounds are tight. 
Consider a vote profile with $\cardM_j\ge h$ for every $j\in[2,\totLabs]$, a coalition $\coal$ with exactly $h$ classifiers from each of the $k$ incorrect labels attains $m=kh$, now include $h-1$ classifiers voting for each remaining incorrect label, it holds $m=kh+(\totLabs-1-k)(h-1)$ and $\tieSet(\coal)$ remains unaltered, where $\Phi(h-1)=(\totLabs-1)(h-1)$ on such a profile, hence $m=k+\Phi(h-1)=\overline{m}(k,h)$.

When the incorrect predictions are focused on few classes, most coefficients $[\gfvar^ku^m]P_h(\gfvar,u)$ equal 0. 
This is captured formally in the next lemma.

\begin{lemma}\label{lemma:vanishingH}
	Let $k\in\{1,\dots,\totLabs-1\}$ and $h\ge1$.
	If $h>\cardM_{(k)}$, then $[\gfvar^ku^m]P_h(\gfvar,u)=0$ for every $m\ge0$.
\end{lemma}
\begin{proof}
	Fix $k\in\{1,\dots,\totLabs-1\}$ and $h\ge1$. 
	By \Cref{lemma:bivariateGenFunction}, $[\gfvar^ku^m]P_h(\gfvar,u)=|\coalset(k,h,m)|$, and every $\coal\in\coalset(k,h,m)$ has $k$ incorrect labels receiving exactly $h$ votes from $\coal\subseteq\modelSet'$, thus $\cardM_j\ge h$ for each such label $\y_j$.
	If $h>\cardM_{(k)}$ then at most $k-1$ labels $j\in[2,\totLabs]$ satisfy $\cardM_j\ge h$, so no such $\coal$ exists and $\coalset(k,h,m)=\emptyset$.
\end{proof}

We now show that using the ranges of \Cref{lemma:support} and \Cref{lemma:vanishingH} our \Cref{theo:ShapleyMCCorrect} retains its correctness.

\begin{corollary}\label{cor:restrictedRanges}
	Let $H^\star=\min\{\cardM_1,\cardM_{(1)}\}$.
	The statement of \Cref{theo:ShapleyMCCorrect} holds with the following modifications.
	\begin{compactenum}[($i$)]
		\item Replace the upper limit of the summation over $h$, which is $\cardM_1$ in the first identity and $\cardM_j$ in the second, by $H^\star$ and by $\min\{H^\star,\cardM_j\}=\min\{\cardM_1,\cardM_j\}$, respectively.
		\item For a correct classifier replace the upper limit $\cardM-\cardM_1$ over $m$ with $\overline{m}(k,h)$.
		\item For an incorrect classifier predicting $\y_j$ replace the upper limit $\cardM-(\cardM_1+\cardM_j)$ over $m$ with $\overline{m}_{-j}(k,h)$.
	\end{compactenum}
\end{corollary}
\begin{proof}
	($i$) 
	Consider the two sums defining the Shapley value from \Cref{theo:ShapleyMCCorrect}. 
	The summation in $h$ for the Shapley value of a correct classifier has upper limit $\cardM_1$, and all terms for $h>\cardM_{(1)}$ are zero since $k\ge1$ and $\cardM_{(k)}\le\cardM_{(1)}$ thus by~\Cref{lemma:vanishingH} it holds $[\gfvar^ku^m]P_h(\gfvar,u)=0$.
	
	Similarly, the summation in $h$ for the Shapley value of an incorrect classifier has upper limit $\cardM_j\le\cardM_{(1)}$, and its terms with $h>\cardM_1$ are zero, since $f_j(k,h,m)$ evaluates to zero since $\binom{\cardM_1}{h}=0$ by the definition of $f_j$ in \Cref{subsubsec:MCSHAP}.
	Therefore the two summations are restricted to $h\le H^\star$ and to $h\le\min\{\cardM_1,\cardM_j\}$ respectively, both of which retain the correctness of \Cref{theo:ShapleyMCCorrect}, proving the first part of the claim.

	($ii$) and ($iii$) Both substitutions truncate the two sums, note $\overline{m}(k,h)\le\cardM-\cardM_1$ and $\overline{m}_{-j}(k,h)\le\cardM-(\cardM_1+\cardM_j)$ by definition.
	By \Cref{lemma:support}, every truncated term contributes a factor $[\gfvar^ku^m]P_h(\gfvar,u)$, respectively $[\gfvar^ku^m]P_h(\gfvar,u)/(A_j(h;u)+B_j(h;u)\gfvar)$, equal to zero. This concludes the proof.
\end{proof}

The modification of~\Cref{alglineSHMC:17} uses the bound $\min\{H^\star,\cardM_j\}$ of \Cref{cor:restrictedRanges}~($i$).
For simplicity, our analysis of \Cref{theo:refinedRuntimeMCSHAP} uses the weaker bound $H^\star$, which is common to all the incorrect labels and can be used as  common bound for the iterations of the loop of \Cref{alglineSHMC:12}.

As mentioned, \suiteShapImp builds at each iteration $h\in [1, H^\star]$ a matrix of size $(K(h)+1)\times(\Phi(h)+1)$ to retain the non-zero coefficients of $P_h(\gfvar, u)$, we now show the key combinatorial identities allowing such construction, and connect $K(h)$ and $\Phi(h)$. 

\begin{lemma}\label{lemma:profileIdentities}
	Let $H=\cardM_{(1)}$, $H^\star=\min\{\cardM_1,\cardM_{(1)}\}$ and fix $h\ge1$. The following hold.
	\begin{compactenum}[($i$)]
		\item $K(h)=\Phi(h)-\Phi(h-1)$.
		\item $\Phi(H)=N$ and $\Phi(h')\le N$ for every $h'\ge0$.
		\item $[\gfvar^ku^m]P_h(\gfvar,u)=0$ for every $k>K(h)$ and every $m\ge0$.
		\item $[\gfvar^ku^m]P_h(\gfvar,u)=0$ for every $m>\Phi(h)$ and every $k\ge0$.
		\item For every $j\in[2,\totLabs]$, with $K_{-j}(h)$ as in~\Cref{def:profileFunctions},
		\[
		[\gfvar^ku^m]\frac{P_h(\gfvar,u)}{A_j(h;u)+B_j(h;u)\gfvar}=0
		\]
		for every $k>K_{-j}(h)$ and every $m\ge0$, and for every $m>\Phi(h)-\min\{\cardM_j,h\}$ and every $k\ge0$.
	\end{compactenum}
\end{lemma}

\begin{proof}
	We prove ($i$).
	Let $a\ge0$ and $h\ge1$.
	If $a\ge h$ then $\min\{a,h\}-\min\{a,h-1\}=h-(h-1)=1$.
	If $a\le h-1$ then $\min\{a,h\}-\min\{a,h-1\}=a-a=0$.
	Therefore, $\min\{a,h\}-\min\{a,h-1\}=\ind[a\ge h]$.
	Hence,
	\[
	\Phi(h)-\Phi(h-1) = \sum_{j=2}^{\totLabs}\ind[\cardM_j\ge h] = K(h) \enspace ,
	\]
	where the first equality uses the above identity, and the second equality is from the definition of $K(h)$, proving ($i$).

	We prove ($ii$).
	By the definition of $H$ we have $\cardM_j\le H$ for every $j\in[2,\totLabs]$, hence $\min\{\cardM_j,H\}=\cardM_j$ and $\Phi(H)=\sum_{j=2}^{\totLabs}\cardM_j=N$.
	Furthermore, the function $\Phi(h)$ is non-decreasing in $h$, thus $\Phi(h)\le\Phi(H)=N$ for $h\le H$, and $\Phi(h)=N$ for $h>H$, proving ($ii$).

	We prove ($iii$).
	By the definition of $K(h)$, at most $K(h)$ labels $j\in[2,\totLabs]$ satisfy $\cardM_j\ge h$, thus $\cardM_{(k)}<h$ for every $k>K(h)$.
	The claim then follows from~\Cref{lemma:vanishingH}, proving ($iii$).

	We prove ($iv$).
	By~\Cref{lemma:bivariateGenFunction}, $[\gfvar^ku^m]P_h(\gfvar,u)=|\coalset(k,h,m)|$.
	Fix $k\ge 0, h\ge 1, m\ge0$ and $\coal\in\coalset(k,h,m)$, each label of $\tieSet(\coal)$ receives exactly $h$ votes from $\coal$ and each incorrect label not in $\tieSet(\coal)$ receives at most $h-1$ votes. Thus $\coal$ contains at most $\min\{\cardM_j,h\}$ classifiers predicting $\y_j$, for every $j\in[2,\totLabs]$.
	Summing over all incorrect labels yields $m=|\coal|\le\Phi(h)$, hence $\coalset(k,h,m)=\emptyset$ for every $m>\Phi(h)$, proving ($iv$).

	We prove ($v$).
	Fix $j\in[2,\totLabs]$, by~\Cref{cor:bivariateGenFunctionMinusJ} it holds $[\gfvar^ku^m]P_h(\gfvar,u)/(A_j(h;u)+B_j(h;u)\gfvar)=|\coalset_{-j}(k,h,m)|$, where a coalition $\coal\in\coalset_{-j}(k,h,m)$ contains no classifier predicting $\y_j$.
	Each label of $\tieSet(\coal)$ receives exactly $h$ votes from $\coal$, hence it corresponds to a label $\y_{j'}$ with $j'\ne j$ and $\cardM_{j'}\ge h$, for which there are $K_{-j}(h)$ such labels by~\Cref{def:profileFunctions}, thus $\coalset_{-j}(k,h,m)=\emptyset$ for every $k>K_{-j}(h)$.
	Furthermore, $\coal$ contains at most $\min\{\cardM_{j'},h\}$ classifiers predicting $\y_{j'}$ for every $j'\ne j$ by the argument of ($iv$), thus $m=|\coal|\le\sum_{j'\neq j}\min\{\cardM_{j'},h\}=\Phi(h)-\min\{\cardM_j,h\}$, proving ($v$) and concluding the proof.
\end{proof}

We now need a final technical lemma used to prove our main bound on the time complexity of \suiteShapImp.

\begin{lemma}\label{lemma:telescoping}
	Fix $H^\star=\min\{\cardM_1,\cardM_{(1)}\}$ and let $N^\star=\Phi(H^\star)$. Then
	\[
	\frac{(N^\star)^3}{3}\le \sum_{h=1}^{H^\star}K(h) \Phi(h)^2\le \frac{(N^\star)^3}{3}+2\totLabs (N^\star)^2 \enspace .
	\]
\end{lemma}

\begin{proof}
	First note that the function $\Phi(h)$ is non-decreasing in $h$ and $\Phi(0)=0$, furthermore $\Phi(H^\star)=N^\star$ by the definition of $N^\star$.
	Now consider the sets
	\[
	X_h \doteq \{x\in\mathbb{Z} : \Phi(h-1)<x\le\Phi(h)\} \enspace,
	\]
	for $h\in[1,H^\star]$.
	Sets $X_h$ form a partition of the set $Z=\{1,\dots,N^\star\}$.
	For any $x\in Z$ we let $\mathrm{col}(x)\in[1,H^\star]$ be the unique index such that $x\in X_{\mathrm{col}(x)}$.
	The sets $X_h$ contains exactly $\Phi(h)-\Phi(h-1)$ integers, and therefore by \Cref{lemma:profileIdentities}~($i$) $|X_h|=K(h)$, 
	thus $\mathrm{col}(x)$ takes the value $h$ for exactly $K(h)$ indices $x\in Z$.
	Hence
	\begin{equation}\label{eq:telescopeRewrite}
		\sum_{h=1}^{H^\star}K(h)\,\Phi(h)^2
		= \sum_{h=1}^{H^\star}\sum_{x\in X_h}\Phi(h)^2
		= \sum_{x=1}^{N^\star}\Phi(\mathrm{col}(x))^2 \enspace,
	\end{equation}
	where the first equality follows by $|X_h|=K(h)$ and in the last equality we grouped the terms of the summation by their value of $\mathrm{col}(x)$, and used the fact that $X_h, h \ge 1$ forms a partition of $Z$.

	We now prove the lower bound.
	By the definition of $\mathrm{col}(x)$ we have $x\le\Phi(\mathrm{col}(x))$, hence $\Phi(\mathrm{col}(x))^2\ge x^2$.
	Thus
	\[
	\sum_{h=1}^{H^\star}K(h)\,\Phi(h)^2 \ge \sum_{x=1}^{N^\star}x^2 = \frac{(N^\star)^3}{3}+\frac{(N^\star)^2}{2}+\frac{N^\star}{6} \ge \frac{(N^\star)^3}{3} \enspace,
	\]
	where the first inequality follows by \Cref{eq:telescopeRewrite}, and the equality follows by the well known sum $\sum_{h=1}^a x^2 = a(a+1)(2a+1)/6, a\ge 1$.

	We conclude proving the upper bound.
	By the definition of $\mathrm{col}(x)$ we have $x>\Phi(\mathrm{col}(x)-1)$, and $\Phi(\mathrm{col}(x)-1)=\Phi(\mathrm{col}(x))-K(\mathrm{col}(x))$ by~\Cref{lemma:profileIdentities}~($i$).
	By the above
	\[
	K(\mathrm{col}(x)) = \Phi(\mathrm{col}(x)) -\Phi(\mathrm{col}(x)-1) > \Phi(\mathrm{col}(x)) - x  \enspace,
	\]
	and thus $\Phi(\mathrm{col}(x))-x\le K(\mathrm{col}(x))-1$ since $x$, $\Phi(\mathrm{col}(x))$ and $K(\mathrm{col}(x)) $ are integers.
	
	Now note
	\[
	\Phi(\mathrm{col}(x))^2-x^2 = \big(\Phi(\mathrm{col}(x))-x\big)\big(\Phi(\mathrm{col}(x))+x\big) \le 2N^\star\big(K(\mathrm{col}(x))-1\big) \enspace,
	\]
	where we used $\Phi(\mathrm{col}(x))\le\Phi(H^\star)=N^\star$ and $x\le N^\star$.
	Consider the following sum
	\begin{align*}
	\sum_{x=1}^{N^\star}\Big(\Phi(\mathrm{col}(x))^2-x^2\Big) &\le 2N^\star\sum_{h=1}^{H^\star}K(h)\big(K(h)-1\big) \\
	&\le 2N^\star(\totLabs-2)\sum_{h=1}^{H^\star}K(h) = 2(N^\star)^2(\totLabs-2) \enspace ,
	\end{align*}
	the first inequality follows from the above bound and by grouping the terms by their value of $\mathrm{col}(x)\in [1,H^\star]$, where for fixed $h$ there are exactly $K(h)$ such terms (i.e., $|X_h| = K(h)$).
	The second inequality follows since $K(h)\le\totLabs-1$ and the final equality uses $\sum_{h=1}^{H^\star}K(h)=\Phi(H^\star)-\Phi(0)=N^\star$.
	From \Cref{eq:telescopeRewrite} we have that $\sum_{h=1}^{H^\star}K(h)\,\Phi(h)^2=\sum_{x=1}^{N^\star}x^2+\sum_{x=1}^{N^\star}\big(\Phi(\mathrm{col}(x))^2-x^2\big)$, thus
	\[
	\sum_{h=1}^{H^\star}K(h)\,\Phi(h)^2 \le \frac{(N^\star)^3}{3}+(N^\star)^2+2(N^\star)^2(\totLabs-2) \le \frac{(N^\star)^3}{3}+2\totLabs (N^\star)^2 \enspace,
	\]
	where we used the well known bound $\sum_{x=1}^{N^\star}x^2\le (N^\star)^3/3+(N^\star)^2$, which holds for every $N^\star\ge1$,
	which concludes the proof.
\end{proof}

The above identity will be used to bound the total number of operations performed by \suiteShapImp at each of its iterations $h\in[1,H^\star]$.
To obtain a final bound for the processing of the full dataset $\dataSet$ we need the following definition, let $\widetilde{N} = \max_{(\xp,\y) \in \dataSet} N^\star_\xp$, where $N^\star_\xp$ depends on the vote profile for $\xp$ (recall that we fixed a single data point $(\xp,\y)$ up to now).
We can now prove the time complexity of \suiteShapImp.

\begin{theorem}\label{theo:refinedRuntimeMCSHAP}
	Consider \suiteShapImp (\Cref{alg:ShapMC} modified as listed in \Cref{appsubsec:modifiedAlg}),
	its output coincides with the output of \Cref{alg:ShapMC}, and its time complexity to compute the Shapley values for all classifiers in $\modelSet$ in the game $\gameMC_\dataSet$ for a dataset $\dataSet$ is bounded by
	\[
	\bigO\!\left(|\unique|\big((\widetilde{N})^3+\totLabs (\widetilde{N})^2\big) + |\dataSet|\cardM\right) \enspace,
	\]
	and its memory complexity is bounded by $\bigO(|\unique|\totLabs+\totLabs \widetilde{N})$.
\end{theorem}

\begin{proof}
	The output of \suiteShapImp corresponds to the output of \Cref{alg:ShapMC} by \Cref{cor:restrictedRanges}.

	We now bound the time complexity.
	As in the proof of~\Cref{theo:runtimeMCSHAP}, by symmetry it is sufficient to compute the values $\shap_i, i\in[\cardM]$ for the data points in $\dataSet$ with unique vote profile $\VotVec\in\unique$, and evaluating the predictions of all classifiers over $\dataSet$ has time complexity $\bigO(|\dataSet|\cardM)$.
	Fix one data point $(\xp,\y)\in \dataSet$ with its vote profile, and let $H^\star= H^\star_\xp$ and $N^\star=N^\star_\xp$ for ease of notation.

	Fix an iteration $h\in[1,H^\star]$ of \suiteShapImp, then the matrix retaining the coefficients of $P_h(\gfvar, u)$ has size  $(K(h)+1) \times (\Phi(h)+1)$.
	For each $j\in[2,\totLabs]$ the polynomial $A_j(h;u)$ has at most $\min\{h,\cardM_j+1\}$ non-zero coefficients and $B_j(h;u)$ has at most one, thus multiplying the matrix by $A_j(h;u)+B_j(h;u)\gfvar$ requires $\bigO(K(h)\Phi(h)\min\{\cardM_j,h\})$ operations.
	Summing over the labels and using $\sum_{j=2}^{\totLabs}\min\{\cardM_j,h\}=\Phi(h)$, the matrix is computed in $\bigO(K(h)\Phi(h)^2)$ operations.
	The terms contributing to the Shapley value of a correct classifier require then $\bigO(K(h)\Phi(h))$ operations, by a simple linear scan of the matrix.
	For an incorrect label $j\in[2,\totLabs]$ the division and the multiplication with $A_j(h;u)+B_j(h;u)\gfvar$ are dominated by the construction of the matrix, hence the time complexity of processing all incorrect labels is bounded by $\bigO(\totLabs K(h)\Phi(h)+ K(h)\Phi(h)^2)$.

	Now observe that $H^\star \le N^\star$, since the term of $N^\star$ indexed by a label $j$ attaining $\cardM_j=\cardM_{(1)}$ equals $\min\{\cardM_{(1)},\cardM_1\}=H^\star$, and every term of $N^\star$ is non-negative.
	Iterating over the labels for any fixed $h$ has $\bigO(\totLabs)$ time complexity contributing $\bigO(\totLabs H^\star)=\bigO(\totLabs N^\star)$ over the iterations of \suiteShapImp.
	Furthermore, note that $\totLabs+H^\star=\bigO(\totLabs N^\star)$, thus the complexity of bucket sort is dominated by the above.
	Summing the costs over the iterations we have
	\[
	\sum_{h=1}^{H^\star}\bigO\big(K(h)\Phi(h)^2+\totLabs K(h)\Phi(h)\big) + \bigO(\totLabs N^\star) = \bigO\big((N^\star)^3+\totLabs (N^\star)^2\big) \enspace ,
	\]
	where we used~\Cref{lemma:telescoping} and \[
	\sum_{h=1}^{H^\star}K(h)\,\Phi(h)\le N^\star\sum_{h=1}^{H^\star}K(h)=(N^\star)^2 \enspace,
	\]
	 which holds since $\Phi(h)\le\Phi(H^\star)=N^\star$ and  $\sum_{h=1}^{H^\star}K(h) = \Phi(H^\star)-\Phi(0)$ by \Cref{lemma:profileIdentities}~($i$).
	
	Now consider the processing of the entire dataset $\dataSet$, for each point $(\xp, \y) \in \dataSet$ we have $N^\star_\xp \le \widetilde{N}$ by definition.
	Thus, summing over the vote profiles of $\unique$ and considering the evaluation of the votes of all classifiers, the time complexity to compute $\shap_i, i\in[\cardM]$ is bounded by
	\[
	\bigO\!\left(|\unique|\big((\widetilde{N})^3+\totLabs (\widetilde{N})^2\big) + |\dataSet|\cardM\right) \enspace,
	\]
	proving our main result.
	
	Finally, note that the memory complexity is dominated by storing the largest matrix across all data points in $\dataSet$ thus $\bigO(K(h)\Phi(h))=\bigO(\totLabs \widetilde{N})$, and the $\bigO(|\unique|\totLabs)$ values of the unique vote profile, and the output, as in the proof of~\Cref{theo:runtimeMCSHAP}, which concludes the proof.
\end{proof}

We now show that the time complexity of \suiteShapImp improves over that of \Cref{alg:ShapMC}.

\begin{corollary}\label{cor:refinedComparison}
	Recall that $N^\star = \sum_{j=2}^{\totLabs}\min\{\cardM_j,\cardM_1\}$.
	For the processing of a unique vote profile of a fixed data point $(\xp, \y) \in \dataSet$ with $N^\star\ge\totLabs$, the time complexity of \suiteShapImp improves on the bound of~\Cref{theo:runtimeMCSHAP} by a factor
	\[
	\Theta\!\left(\totLabs\left(\frac{\cardM}{N^\star}\right)^{3}\right) \enspace,
	\]
	that is, the improvement is
	\[
	\Omega\!\left(\max\left\{\totLabs\left(\frac{\cardM}{\cardM-\cardM_1}\right)^{3}, \frac{\cardM^{3}}{\totLabs^{2}\cardM_1^{3}}\right\}\right)\enspace.
	\]
\end{corollary}

\begin{proof}
	By the hypothesis we have $N^\star\ge\totLabs$.
	Thus, $\totLabs (N^\star)^2\le (N^\star)^3$, and the bound on the time complexity of \suiteShapImp from \Cref{theo:refinedRuntimeMCSHAP} becomes $\bigO((N^\star)^3)$, for a single data point.
	
	Now recall that by \Cref{theo:runtimeMCSHAP}, the complexity of \Cref{alg:ShapMC} to process a single data point (i.e., unique vote profile) is bounded by $\bigO(\totLabs\cardM^3)$.
	Thus the improvement, evaluated by the ratio of the bounds is 
	$
	\tfrac{\totLabs\cardM^3}{(N^\star)^3},
	$ yielding the first claim.

	The second claim follows by upper bounding $N^\star$ with $\min\{\cardM-\cardM_1,(\totLabs-1)\cardM_1\}$, as already shown.
	Substituting such a bound in $\tfrac{\totLabs\cardM^3}{(N^\star)^3}$ yields
	\[
	\frac{\totLabs\cardM^3}{\min\{\cardM-\cardM_1,(\totLabs-1)\cardM_1\}^3}
	= \max\left\{\totLabs\left(\frac{\cardM}{\cardM-\cardM_1}\right)^{3},\ \frac{\totLabs\cardM^{3}}{(\totLabs-1)^{3}\cardM_1^{3}}\right\} \enspace,
	\]
	where we used that for positive numbers $A,B,C$ it holds $A/\min\{B,C\}^3 = \max\{A/B^3, A/C^3\}$, since $x\mapsto x^3$ is increasing.
	The claim follows by $\totLabs/(\totLabs-1)^3=\Theta(\totLabs^{-2})$, concluding the proof.
\end{proof}

First note, that under the hypothesis $N^\star\ge\totLabs$, the lower bound of \Cref{lemma:telescoping} yields $\sum_{h\le H^\star}K(h)\Phi(h)^2=\Theta((N^\star)^3)$.
That is, the bound on the complexity of \suiteShapImp is tight, thus the above improvement is not an artifact of our analysis.

Furthermore, the bound from \Cref{theo:refinedRuntimeMCSHAP} is never worse than the bound of \Cref{theo:runtimeMCSHAP}, for any vote profile, since $N^\star\le\cardM-\cardM_1\le\cardM$ and thus $(N^\star)^3+\totLabs (N^\star)^2\le 2\totLabs\cardM^3$.
The improvement is substantial on vote profiles where most classifiers are correct.
Consider a data point with $N^\star\ge\totLabs$ on which a fraction $\alpha\in(0,1)$ of the classifiers predicts $\y_1$ we thus have $N^\star\le(1-\alpha)\cardM$, and by \Cref{cor:refinedComparison} the improvement is $\Omega(\totLabs/(1-\alpha)^3)$, i.e., a factor $10^3\totLabs$ for $\alpha=9/10$.

\ifarxiv
\subsection{A fast polynomial multiplication-based algorithm}\label{appsubsec:fft}

We observe that for any fixed iteration $h\in[1,H^\star]$, \suiteShapImp incurs a  total cost of $\bigO(K(h)\Phi(h)^2)$ for building the matrix retaining the coefficients of  $P_h(\gfvar, u)$, due to our simple way of performing polynomial multiplication.
That is, we process one label at a time and \texttt{AddClass} and \texttt{RemoveClass} respectively multiply and divide the matrix by the bivariate polynomials $A_j(h;u)+B_j(h;u)\gfvar$ for any $j\in[2,\totLabs]$.
We briefly describe an algorithm attaining $\bigO(K(h)(\Phi(h-1)+1)\log(K(h)(\Phi(h-1)+1))\log\totLabs)$ cost for processing the matrix containing the coefficients of $P_h(\gfvar, u)$, relying on fast polynomial multiplication.
We however note that we do not implement such technique, as fast polynomial multiplication requires to perform the products modulo a set of primes, followed by the use of the Chinese remainder theorem, while we focus on exact arithmetic, similar concerns are known for weighted majority games (see \citet{kakimura2026computing}).
For our analysis we use the notation $\mathsf{M}(s)$ to denote the number of operations used to multiply two univariate polynomials of degree less than $s$, following the convention of \citet{vzgathen2013modern}.
We recall that it holds $\mathsf{M}(s)=\bigO(s\log s)$ over a ring supporting the FFT, or $\mathsf{M}(s)=\bigO(s\log s\log\log s)$ over an arbitrary commutative ring \cite[Table 8.6]{vzgathen2013modern}.

\textbf{The algorithm.} Fix $h\ge1$ and recall from \Cref{lemma:bivariateGenFunction} that $P_h(\gfvar,u)=\prod_{j=2}^{\totLabs}(A_j(h;u)+B_j(h;u)\gfvar)$, where $B_j(h;u)=\binom{\cardM_j}{h}u^h$.

The variable $\gfvar$ marks a label receiving \emph{exactly} $h$ votes, hence in each factor $\gfvar$ it appears only in the monomial $B_j(h;u)\gfvar$, always multiplied by $u^h$.
A monomial $\gfvar^ku^m$ of $P_h(\gfvar,u)$ therefore collects one factor $u^h$ from each of the $k$ factors contributing $\gfvar$, that is, the exponent $m$ contains a term $kh$ which is determined by the exponent $k$ of $\gfvar$.
We can thus remove such a term, and for $j\in[2,\totLabs]$ we let
\[
F_j(\gfvar,u)\doteq A_j(h;u)+\binom{\cardM_j}{h}\gfvar \enspace .
\]
That is $P_h(\gfvar,u)=\prod_{j=2}^{\totLabs}F_j(u^h\gfvar,u)$, or equivalently
\[
[\gfvar^ku^m]P_h(\gfvar,u)=[\gfvar^ku^{m-kh}]\prod_{j=2}^{\totLabs}F_j(\gfvar,u) \enspace , \qquad\text{for every } k\ge0 \text{ and } m\ge0 \enspace ,
\]
i.e., by simply substituting $u^h\gfvar$ for $\gfvar$ recovers the removed factors.
The algorithm we now describe computes the coefficients of $\prod_{j=2}^{\totLabs}F_j(\gfvar,u)$.
The coefficients of $P_h(\gfvar,u)$ are then recovered by the identity above, that is by the map of the indices moving the coefficient of indices $(k,m)$ to the indices $(k,m+kh)$.
Using the factors $F_j(\gfvar,u)$ allows us to reduce the degree of the generating function in $u$, yielding a faster polynomial multiplication.
To simplify our notation, in the remainder of this subsection we let $P_h(\gfvar,u)$ denote $\prod_{j=2}^{\totLabs}F_j(\gfvar,u)$, leaving thus the above substitution implicit in the notation.

\emph{Bottom-up pass.} We can compute the matrix with the coefficients of $P_h(\gfvar,u)$ by divide and conquer as follows
\begin{compactenum}[($a$)]
	\item discard every label $\y_j, j\ge 2$ with $\cardM_j=0$, for which $A_j(h;u)=1$ and $\binom{\cardM_j}{h}=0$, hence $F_j(\gfvar,u)=1$, and let $\totLabs_h\le\totLabs-1$ be the number of remaining labels;
	\item build a binary tree whose nodes hold sets of labels with $\cardM_j\ge1, j\ge 2$: the root holds the $\totLabs_h$ labels with non-zero votes. 
	Any (internal) node holds a set $J$ of at least two labels, with its two children holding a balanced partition of $J$ of sizes $\lceil|J|/2\rceil$ and $\lfloor|J|/2\rfloor$. 
	A leaf holds a single label, therefore the tree has $\totLabs_h$ leaves and depth $\lceil\log_2\totLabs_h\rceil$, since it is a complete binary tree.
	\item assign to each leaf the factor $F_j(\gfvar,u)$ of the single label $\y_j$ it holds, and to each internal node the product of the two polynomials of its children.
	Thus a node holding a set $J_v$ holds $\prod_{j\in J_v}F_j(\gfvar,u)$, and the root node holds $P_h(\gfvar,u)$;
	\item the polynomial of each node $v$ in the tree has degree $K_v\doteq|\{j\in J_v : \cardM_j\ge h\}|$ in $\gfvar$ and degree $\Phi_v\doteq\sum_{j\in J_v}\min\{\cardM_j,h-1\}$ in $u$, since $\binom{\cardM_j}{h}$ is zero for $\cardM_j<h$, while the degree in $u$ of $F_j(\gfvar,u)$ is the degree of $A_j(h;u)$, that is $\min\{\cardM_j,h-1\}$, as the coefficient of $\gfvar$ in $F_j(\gfvar,u)$ does not depend on $u$.
	Therefore, the matrix of coefficients representing $\prod_{j\in J_v}F_j(\gfvar,u)$ of node $v$ has $(K_v+1)\times(\Phi_v+1)$ entries.
	Note that the root of the tree holds a matrix with $(K(h)+1)\times(\Phi(h-1)+1)$ entries, in place of the $(K(h)+1)\times(\Phi(h)+1)$ entries required to store the coefficients of $P_h(\gfvar,u)$ without our substitution;
	\item for any internal node $v$ with children $w$ and $r$ let $P_w$ and $P_r$ represent the two polynomials of the two nodes respectively.
	To compute $\prod_{j\in J_v}F_j(\gfvar,u)$ from $P_w, P_r$ apply the \emph{Kronecker substitution} $\gfvar=u^{\Phi_v+1}$ on both $P_w, P_r$, which maps each polynomial to a \emph{univariate} polynomial in $u$.
	The image of $P_w$ has degree $K_w(\Phi_v+1)+\Phi_w\le(K_v+1)(\Phi_v+1)-1$, since $K_w\le K_v$ and $\Phi_w\le\Phi_v$, and analogously for $P_r$.
	The two images can be multiplied with $\bigO(\mathsf{M}((K_v+1)(\Phi_v+1)))$ operations~\cite[Corollary 8.28]{vzgathen2013modern}.
	To recover the coefficients of $\prod_{j\in J_v}F_j(\gfvar,u)$ it is sufficient to split the resulting product of the two images into consecutive blocks of $\Phi_v+1$ coefficients, the block of index $k\in[0,K_v]$ corresponds to the coefficients of the term $\gfvar^k$.
	Such a splitting inverts the substitution since $J_v$ is the disjoint union of $J_w$ and $J_r$, hence $\Phi_w+\Phi_r=\Phi_v$ (see ($d$) above) and the product $\prod_{j\in J_v}F_j(\gfvar,u)$ has degree at most $\Phi_v$ in the variable $u$.
	The monomials $\gfvar^ku^m$ with $k\in[0,K_v]$ and $m\in[0,\Phi_v]$ are therefore mapped to the distinct powers $u^{k(\Phi_v+1)+m}$, avoiding overlap.
\end{compactenum}
It is clear that processing the tree in a bottom-up manner, we recover at the root node the matrix retaining the coefficients of $P_h(\gfvar, u)$. 
In turn, this allows us to compute the contribution of index $h$ to the Shapley value of a  correct classifier, by the first equation of \Cref{theo:ShapleyMCCorrect}.

To bound the total time complexity of an iteration $h\in[1,H^\star]$ of the algorithm we sum the cost of computing the polynomials of the nodes of the tree, one level at a time.
Let $V_L$ be the set of the nodes at level $L\in[0,\lceil\log_2\totLabs_h\rceil]$ of the tree, and recall from ($e$) that computing the polynomial of a node $v$ requires $\bigO(\mathsf{M}((K_v+1)(\Phi_v+1)))$ operations.
The label sets of the nodes of $V_L$ are pairwise disjoint subsets of the $\totLabs_h$ labels, thus $\sum_{v\in V_L}K_v\le K(h)$, $\sum_{v\in V_L}\Phi_v\le\Phi(h-1)$, and $|V_L|\le\totLabs_h$.
Therefore
\begin{equation}\label{eq:levelSize}
\sum_{v\in V_L} (K_v+1)(\Phi_v+1) \le K(h)\Phi(h-1)+K(h)+\Phi(h-1)+\totLabs_h = \bigO\big(K(h)(\Phi(h-1)+1)\big) \enspace ,
\end{equation}
where the first inequality expands the product and uses $\sum_{v\in V_L}K_v\Phi_v\le(\sum_{v\in V_L}K_v)(\sum_{v\in V_L}\Phi_v)$, which holds since all the terms are non-negative.
The last step groups the first two terms as $K(h)\Phi(h-1)+K(h)=K(h)(\Phi(h-1)+1)$, and bounds each of the two remaining terms with $K(h)(\Phi(h-1)+1)$.
The term $\Phi(h-1)$ is bounded since $K(h)\ge1$ for every $h\le H^\star$, while for the term $\totLabs_h$ we use $\totLabs_h\le K(h)(\Phi(h-1)+1)$.
In fact, for $h=1$ we have $\Phi(0)=0$, hence $K(h)(\Phi(h-1)+1)=K(1)=\totLabs_h$, and the bound holds with equality.
For $h\ge2$ each of the $\totLabs_h$ labels has $\min\{\cardM_j,h-1\}\ge1$, hence $\Phi(h-1)\ge\totLabs_h$, and the bound follows by $K(h)\ge1$.

The cost of performing the polynomial multiplications on a fixed level is therefore $\sum_{v\in V_L}\bigO(\mathsf{M}((K_v+1)(\Phi_v+1)))=\bigO(\mathsf{M}(K(h)(\Phi(h-1)+1)))$, by the superlinearity $\mathsf{M}(s)+\mathsf{M}(s')\le \mathsf{M}(s+s')$ and by $\mathsf{M}(cs)=\bigO(\mathsf{M}(s))$ for any constant $c$, see~\citet[Section 8.3]{vzgathen2013modern}.

The tree has $\bigO(\log\totLabs)$ levels, thus summing the cost of the levels yields a total of $\bigO(\mathsf{M}(K(h)(\Phi(h-1)+1))\log\totLabs)$ number of operations.
The bound $\mathsf{M}(s)=\bigO(s\log s)$ of the FFT to perform the multiplications over the tree yields a time complexity bounded by $\bigO(K(h)(\Phi(h-1)+1)\log (K(h)(\Phi(h-1)+1))\log\totLabs)$.

\emph{Top-down pass.} We now show that the divisions required \texttt{RemoveClass} can be performed without incurring in an additional complexity cost.
In particular, for $j\in[2,\totLabs]$ let
\[
Q_j\doteq\prod_{\substack{j'=2 \\ j' \neq j}}^\totLabs (A_{j'}(h;u) + B_{j'}(h;u)\gfvar) \enspace,
\]
that is, $Q_j$ is the generating function of~\Cref{cor:bivariateGenFunctionMinusJ}, obtained by removing the factor of $\y_j$ from the polynomial of~\Cref{lemma:bivariateGenFunction}.
Note that $Q_j$ is a product of the factors of~\Cref{lemma:bivariateGenFunction}, and not of the factors $F_j(\gfvar,u)$ used by the tree.
We will show that we do not require to compute explicitly the matrices $Q_j$, and we can recover their contribution from our tree constructed above.
For a fixed $h\in[1,H^\star]$ we can express
\begin{align*}
f_j(k,h,m) &=\binom{\cardM_1}{h}\left[\binom{\cardM_j-1}{h}\,w_1(k,m)+\binom{\cardM_j-1}{h-1}\,w_2(k,m)\right] \enspace , \text{ where}\\
w_1(k,m) &\doteq\frac{\eta(m+2h+1)}{k+2} \enspace , \quad \text{ and } \quad
w_2(k,m) \doteq\frac{\eta(m+2h)}{(k+1)(k+2)} \enspace .
\end{align*}
Note that the dependency of $f_j$ on the label $\y_j$ is limited to the two binomial coefficients, while $w_1$ and $w_2$ are common to all the labels $j\in[2,\totLabs]$.
We extend $\eta(\cdot)$ by $\eta(a)\doteq0$ for $a>\cardM$, thus $w_1(k,m)$ and $w_2(k,m)$ can be evaluated for every $k\ge0$ and $m\ge0$. 
Fix $i\in\{1,2\}$, for a polynomial $G(\gfvar,u)$ we let
\[
\Lambda(G,w_i)\doteq\sum_{k\ge0}\sum_{m\ge0} w_i(k,m)\,[\gfvar^ku^m]G \enspace .
\]
The two sums have finitely many non-zero terms, since $G$ has finitely many non-zero coefficients, and $w_i(k,m)$ is defined for every $k\ge0$ and $m\ge0$ by our extension of $\eta(\cdot)$, hence $\Lambda(\cdot,w_i)$ is a linear functional on the polynomials in $\gfvar$ and $u$.
Furthermore, for $j\in[2,\totLabs]$ and $h\ge1$ we let
\[
S_h(j)\doteq\sum_{k=0}^{\totLabs-2}\ \sum_{m=hk}^{\cardM-(\cardM_1+\cardM_j)} f_j(k,h,m)\,[\gfvar^ku^m]Q_j \enspace ,
\]
that is $S_h(j)$ is the summand of index $h$ of the second equation of \Cref{theo:ShapleyMCCorrect}, i.e., a classifier predicting the incorrect label $\y_j$ has Shapley value $-\sum_{h=1}^{\cardM_j}S_h(j)$ over $\gameMC_{\evz}$.
The next lemma shows that $S_h(j)$ is determined by the two values $\Lambda(Q_j,w_1)$ and $\Lambda(Q_j,w_2)$, that is by two linear functionals of $Q_j$.
Hence, to evaluate $S_h(j)$ we do not need to process the $(K_{-j}(h)+1)(\Phi(h-1)-\min\{\cardM_j,h-1\}+1)$ entries of the matrix representing $Q_j$ after the map of the indices.

\begin{lemma}\label{lemma:functionalContribution}
	Fix $j\in[2,\totLabs]$ and $h\ge1$.
	The two sums defining $\Lambda(Q_j,w_i)$, $i\in\{1,2\}$, have finitely many non-zero terms, and
	\[
	S_h(j)=\binom{\cardM_1}{h}\left[\binom{\cardM_j-1}{h}\Lambda(Q_j,w_1)+\binom{\cardM_j-1}{h-1}\Lambda(Q_j,w_2)\right] \enspace .
	\]
\end{lemma}
\begin{proof}
	By \Cref{lemma:profileIdentities}~($v$) and \Cref{lemma:support}~($ii$), a pair of indices $(k,m)$ for which $[\gfvar^ku^m]Q_j\ne0$ must satisfy
	\begin{equation}\label{eq:supportQj}
		k\le K_{-j}(h)\le\totLabs-2
		\qquad\text{and}\qquad
		hk\le m\le\Phi(h)-\min\{\cardM_j,h\}\le\cardM-(\cardM_1+\cardM_j) \enspace ,
	\end{equation}
	where the bound on $k$ and the bound $m\le\Phi(h)-\min\{\cardM_j,h\}$ follow by \Cref{lemma:profileIdentities}~($v$).
	Furthermore, $K_{-j}(h)\le\totLabs-2$ since $K_{-j}(h)$ counts the incorrect labels other than $\y_j$ receiving at least $h$ votes, which are at most $\totLabs-2$, thus by the fact that $k\le \totLabs-2$ and from \Cref{lemma:support}~($ii$) it must be that $m\ge hk$.
	The last inequality holds since 
	\[
	\Phi(h)-\min\{\cardM_j,h\}=\sum_{j'\in[2,\totLabs],j'\ne j}\min\{\cardM_{j'},h\}\le\sum_{j'\in[2,\totLabs],j'\ne j}\cardM_{j'}=N-\cardM_j \enspace.
	\]
	Hence $Q_j$ has finitely many non-zero coefficients, and this holds also for the two sums defining $\Lambda(Q_j,w_i)$, which sum over the non-zero coefficients of $Q_j$.
	Note that our extension of $\eta$ does not modify any non-zero term.
	That is, consider a pair $(k,m)$ of indices with $\binom{\cardM_1}{h}\binom{\cardM_j-1}{h}[\gfvar^ku^m]Q_j\ne0$,
	it must be that $h\le\cardM_1$, $h\le\cardM_j-1$ and $m\le\cardM-(\cardM_1+\cardM_j)$ by \Cref{eq:supportQj}, hence $m+2h+1\le\cardM$, and analogously $m+2h\le\cardM$ when $\binom{\cardM_1}{h}\binom{\cardM_j-1}{h-1}[\gfvar^ku^m]Q_j\ne0$.

	By the definition of $f_j$ and by linearity,
	\[
	S_h(j)=\binom{\cardM_1}{h}\left[\binom{\cardM_j-1}{h}T_1+\binom{\cardM_j-1}{h-1}T_2\right] \enspace ,\qquad
	T_i\doteq\sum_{k=0}^{\totLabs-2}\ \sum_{m=hk}^{\cardM-(\cardM_1+\cardM_j)}w_i(k,m)\,[\gfvar^ku^m]Q_j \enspace .
	\]
	In the definition of $\Lambda(Q_j,w_i)$, each pair $(k,m)$ of indices outside the ranges defining $T_i$ yields $[\gfvar^ku^m]Q_j=0$ by \Cref{eq:supportQj}, thus contributing a zero term to $\Lambda(Q_j,w_i)$.
	Therefore it holds $T_i=\Lambda(Q_j,w_i)$ for $i\in\{1,2\}$, which concludes the proof.
\end{proof}

By \Cref{lemma:functionalContribution} and \Cref{theo:ShapleyMCCorrect}, at a fixed iteration $h\in[1,H^\star]$ the contribution $S_h(j)$ of every classifier predicting an incorrect label is determined by the $2(\totLabs-1)$ values $\Lambda(Q_j,w_i)$, for $j\in[2,\totLabs]$ and $i\in\{1,2\}$.\footnote{We will drop the dependency on $h$ to simplify our notation.}
We now briefly discuss how to obtain the values $\Lambda(Q_j,w_i)$, by a top-down pass on the product tree built for computing the coefficients of $P_h(\gfvar, u)$, avoiding the explicit computation of $Q_j, j\in[2,\totLabs]$, which would defeat our improvement in the time complexity.

For a node $v$ of the tree, with labels $J_v$, let $P_v\doteq\prod_{j\in J_v}F_j(\gfvar,u)$ be its polynomial, and let $R_v\doteq\prod_{j\notin J_v}F_j(\gfvar,u)$ be the product of the factors outside its subtree.
Thus, $R_v(u^h\gfvar,u)=Q_j$ for any leaf node $v$ holding label $\y_j$, by the map of the indices of the factors $F_j(\gfvar,u)$.
In addition, it is simple to observe that $R_w=R_vP_r$ for the two children $w$ and $r$ of an internal node $v$.
We now fix $i\in\{1,2\}$ for the rest of our discussion, our argument will apply for each such value independently.
To account for the map of the indices we let $\widehat w_i(k,m)\doteq w_i(k,m+kh)$, which is clearly defined for every $k\ge0$ and $m\ge0$ by our extension of $\eta(\cdot)$.
Substituting $m=m'+kh$ in the definition of $\Lambda(\cdot,w_i)$ yields $\Lambda(R_v,\widehat w_i)=\Lambda(Q_j,w_i)$ for any leaf node $v$ holding $\y_j$.
For each node $v$ consider the linear functional $\Lambda_v(X)\doteq\Lambda(XR_v,\widehat w_i)$, where the argument $X$ corresponds to a polynomial of degree at most $K_v$ in $\gfvar$ and at most $\Phi_v$ in $u$, that is $X$ ranges over the polynomials with the degrees of $P_v$.
Note that $R_v$ has degree $K(h)-K_v$ in $\gfvar$ and $\Phi(h-1)-\Phi_v$ in $u$, since each factor $F_j(\gfvar,u), j\notin J_v$ defining $R_v$ has non-negative coefficients, thus their product sums the degrees of all factors, summing up to $K(h)-K_v$ and $\Phi(h-1)-\Phi_v$ by ($d$) and \Cref{def:profileFunctions}.
Hence the product $XR_v$ has degree at most $K(h)$ in $\gfvar$ and at most $\Phi(h-1)$ in $u$, that is the degrees of $P_h(\gfvar,u)$.
In particular,
since $P_vR_v=P_h(\gfvar,u)$ for any node $v$ in the tree we have $\Lambda_v(P_v)=\Lambda(P_h,\widehat w_i)$.
Any polynomial $X$ of the domain of $\Lambda_v$ is the sum of its monomials $\gfvar^ku^m$ weighted by the coefficients $X[k,m]$, hence by linearity
\[
\Lambda_v(X)=\sum_{k=0}^{K_v}\sum_{m=0}^{\Phi_v}X[k,m]\,\Lambda_v(\gfvar^ku^m) \enspace ,
\]
that is $\Lambda_v$ is determined by its $(K_v+1)(\Phi_v+1)$ values on the monomials.
We represent such values with a matrix of dimension $(K_v+1)\times(\Phi_v+1)$, and we use $\Lambda_v[k,m]\doteq\Lambda_v(\gfvar^ku^m)$ to denote the entries retained in such a matrix.
The next lemma shows two important properties exploited by our algorithm, used to evaluate the linear functionals $\Lambda_v$.

\begin{lemma}\label{lemma:shiftFunctional}
	Fix $i\in\{1,2\}$, $h\in[1,H^\star]$ and a node $v$ of the tree, and let $R_v[k',m']\doteq[\gfvar^{k'}u^{m'}]R_v$, with the convention $R_v[k',m']=0$ whenever $k'<0$ or $m'<0$.
	For every $k\in[0,K_v]$ and $m\in[0,\Phi_v]$ we have
	\begin{compactenum}[($i$)]
		\item
		\[
		\Lambda_v[k,m]=\sum_{k'\ge0}\sum_{m'\ge0}\widehat w_i(k+k',m+m')\,R_v[k',m'] \enspace ,
		\]
		where the sums have finitely many non-zero terms.
		\item For every polynomial $G(\gfvar,u)$ of degree at most $K_v-k$ in $\gfvar$ and at most $\Phi_v-m$ in $u$ it holds
		\[
		\Lambda_v(\gfvar^ku^mG)=\sum_{k'\ge0}\sum_{m'\ge0}G[k',m']\,\Lambda_v[k+k',m+m'] \enspace .
		\]
	\end{compactenum}
\end{lemma}
\begin{proof}
	($i$) The multiplication by the monomial $\gfvar^ku^m$ shifts the coefficients of $R_v$ by $k$ and $m$, that is $[\gfvar^au^b](\gfvar^ku^mR_v)=R_v[a-k,b-m]$ for every $a\ge0$ and $b\ge0$.
	By the definitions of $\Lambda_v$ and of $\Lambda$ we thus have
	\[
	\Lambda_v[k,m]=\Lambda(\gfvar^ku^mR_v,\widehat w_i)=\sum_{a\ge0}\sum_{b\ge0}\widehat w_i(a,b)\,R_v[a-k,b-m] \enspace ,
	\]
	and the claim follows by the substitutions $a=k+k'$ and $b=m+m'$, noting that the terms with $a<k$ or $b<m$ are zero by definition.
	It is simple to note that the sum has finitely many non-zero terms since $R_v$ is a polynomial, and $\widehat w_i(a,b)$ is defined for every $a\ge0$ and $b\ge0$ by our extension of $\eta(\cdot)$.

	($ii$) Expanding $G$ over its monomials and multiplying by $\gfvar^ku^m$ yields $\gfvar^ku^mG=\sum_{k'\ge0}\sum_{m'\ge0}G[k',m']\gfvar^{k+k'}u^{m+m'}$.
	Each of its monomials lies in the domain of $\Lambda_v$, since $k+k'\le K_v$ and $m+m'\le\Phi_v$ by the assumption on the degrees of $G$, and on the values of $k$ and $m$, thus $\gfvar^ku^mG$ is in the domain of $\Lambda_v$.
	The claim follows by the linearity of $\Lambda_v$ and by $\Lambda_v[k+k',m+m']=\Lambda_v(\gfvar^{k+k'}u^{m+m'})$.
\end{proof}

We are now ready to detail our top-down algorithm, fix $i\in\{1,2\}$ and proceed as follows
\begin{compactenum}[($i$)]
	\item the root $v$ has $J_v=[2,\totLabs]$ and $R_v=1$, hence $\Lambda_v(X)=\Lambda(X,\widehat w_i)$ and $\Lambda_v[k,m]=\widehat w_i(k,m)$ for $k\in[0,K(h)]$ and $m\in[0,\Phi(h-1)]$;
	\item at an internal node $v$ with children $w$ and $r$ we have $R_w=R_vP_r$. 
	Thus $\Lambda_w(X)=\Lambda(XP_rR_v,\widehat w_i)=\Lambda_v(XP_r)$, by evaluating $X=\gfvar^ku^m$ it holds
	\begin{equation}\label{eq:childFunctional}
	\Lambda_w[k,m]=\Lambda_v(\gfvar^ku^mP_r)=\sum_{k'=0}^{K_r}\sum_{m'=0}^{\Phi_r}P_r[k',m']\,\Lambda_v[k+k',m+m'] \enspace ,
	\end{equation}
	for $k\in[0,K_w]$ and $m\in[0,\Phi_w]$, where we used \Cref{lemma:shiftFunctional} ($ii$) with $G=P_r$, note that the degrees satisfy the bounds of the lemma, since $K_w+K_r=K_v$ and $\Phi_w+\Phi_r=\Phi_v$, i.e., $K_r\le K_v-k$ and $\Phi_r\le\Phi_v-m$.
	That is, we can obtain $\Lambda_w$ from the entries of $P_r$ (computed in the bottom-up processing of the tree) and from the entries of $\Lambda_v$ (already available since we are processing the tree top-down).
	The matrix $\Lambda_r$ is obtained analogously from $\Lambda_v$ and $P_w$.
	Note that we do not have to compute \Cref{eq:childFunctional} by evaluating each entry independently.
	Such an evaluation would require $(K_w+1)(\Phi_w+1)(K_r+1)(\Phi_r+1)$ operations, which is up to quadratic in the size of the matrix of $P_v$, defeating our purpose.
	We show in \Cref{lemma:middleProduct} how to obtain the entries of the matrix of $\Lambda_w$ by a \emph{single polynomial multiplication} of two bivariate polynomials.
	The first polynomial has the entries of the matrix of $\Lambda_v$ as its coefficients, and the second one is $P_r$ with its indices reversed.
	Computing both $\Lambda_w$ and $\Lambda_r$ is thus reduced to $\bigO(\mathsf{M}((K_v+1)(\Phi_v+1)))$ total operations, that is the cost incurred at the node $v$ to build $P_v$ in the bottom-up phase;
	\item read $\Lambda_v[0,0]$ for any leaf $v$, that is, for a leaf $v$ holding the label $\y_j$ it holds $R_v(u^h\gfvar,u)=Q_j$, hence $\Lambda_v[0,0]=\Lambda_v(1)=\Lambda(R_v,\widehat w_i)=\Lambda(Q_j,w_i)$. Finally, use such values to compute $S_h(j)$ by \Cref{lemma:functionalContribution}.
\end{compactenum}

The next lemma shows that the matrix $\Lambda_w$ (obtained in step ($ii$) above) can be computed from matrices $\Lambda_v$ and from the matrix of $P_r$ with its two indices reversed.\footnote{Such a matrix of coefficients is known as the \emph{middle product} of two polynomials, and it can be computed more efficiently than by the full multiplication we perform, with the techniques of \citet{Bostan2003}.}

\begin{lemma}\label{lemma:middleProduct}
	Fix $i\in\{1,2\}$, $h\in[1,H^\star]$, and an internal node $v$ of the tree with children $w$ and $r$.
	Let $C(\gfvar,u)$ be the product of the two bivariate polynomials
	\[
	\sum_{k=0}^{K_r}\sum_{m=0}^{\Phi_r}P_r[K_r-k,\Phi_r-m]\gfvar^ku^m \enspace ,\qquad
	\sum_{k=0}^{K_v}\sum_{m=0}^{\Phi_v}\Lambda_v[k,m]\gfvar^ku^m \enspace ,
	\]
	that is, the first polynomial has at its coefficients the entries of $P_r$ reversed,\footnote{In matrix form, the entry $P_r[a,b]$ is moved to the position $[K_r-a,\Phi_r-b]$, that is, both indices are reversed.} and the second polynomial holds the matrix of $\Lambda_v$ as its coefficients.
	Then the following hold.
	\begin{compactenum}[($i$)]
		\item $\Lambda_w[k,m]=[\gfvar^{k+K_r}u^{m+\Phi_r}]C(\gfvar,u)$ for every $k\in[0,K_w]$ and $m\in[0,\Phi_w]$; and
		\item the matrix of $\Lambda_w$ is obtained from the matrices of $\Lambda_v$ and of $P_r$ with $\bigO(\mathsf{M}((K_v+1)(\Phi_v+1)))$ total operations.
	\end{compactenum}
\end{lemma}
\begin{proof}
	($i$) We adopt the convention $\Lambda_v[k,m]=0$ whenever $k<0$, $m<0$, $k>K_v$ or $m>\Phi_v$.
	For every $a\ge0$ and $b\ge0$ it holds
	\[
	[\gfvar^au^b]C(\gfvar,u)=\sum_{k'=0}^{K_r}\sum_{m'=0}^{\Phi_r}P_r[K_r-k',\Phi_r-m']\,\Lambda_v[a-k',b-m'] \enspace ,
	\]
	by the definition of polynomial multiplication and by our convention on $\Lambda_v$, since the first factor of $C(\gfvar,u)$ has degree $K_r$ in $\gfvar$ and $\Phi_r$ in $u$.
	Thus
	\[
	[\gfvar^au^b]C(\gfvar, u)=\sum_{k''=0}^{K_r}\sum_{m''=0}^{\Phi_r}P_r[k'',m'']\,\Lambda_v[a-K_r+k'',b-\Phi_r+m''] \enspace ,
	\]
	by substituting $k''=K_r-k'$ and $m''=\Phi_r-m'$.
	Fix $k\in[0,K_w]$ and $m\in[0,\Phi_w]$, and consider the identity above for $a=k+K_r$ and $b=m+\Phi_r$.
	Each entry of $\Lambda_v$ in the sum has therefore indices $k+k''$ and $m+m''$, with $0\le k+k''\le K_w+K_r=K_v$ and $0\le m+m''\le\Phi_w+\Phi_r=\Phi_v$, thus the indices do not range over the zero entries set by our convention. 
	The resulting sum coincides with right-hand side of \Cref{eq:childFunctional}, which equals $\Lambda_w[k,m]$ and proves the first part of the claim.

	($ii$) 
	We multiply the two bivariate polynomials defining $C(\gfvar, u)$ similarly to how we obtain $P_v$ at an internal node $v$ of the tree in the bottom-up phase, that is, by the Kronecker substitution and the FFT.
	First we apply the Kronecker substitution $\gfvar=u^{\Phi_r+\Phi_v+1}$ to the two factors of $C(\gfvar,u)$, as in step ($e$) of the bottom-up algorithm.
	Note that the polynomial $C(\gfvar,u)$ has degree at most $K_r+K_v$ in $\gfvar$ and at most $\Phi_r+\Phi_v$ in $u$, hence the substitution maps the monomials $\gfvar^au^b$ with $a\in[0,K_r+K_v]$ and $b\in[0,\Phi_r+\Phi_v]$ to distinct powers of $u$.
	To recover the coefficients of $C(\gfvar,u)$ we split the coefficients of the product of the two images (i.e., FFT's output) into consecutive blocks of $\Phi_r+\Phi_v+1$ coefficients.\footnote{The coefficient of $u^k$ in the product of the two images is the coefficient of $\gfvar^au^b$ in $C(\gfvar,u)$, for $a=\floor{k/(\Phi_r+\Phi_v+1)}$ and $b=k\bmod(\Phi_r+\Phi_v+1)$.}
	The image of the second factor of $C(\gfvar, u)$ has degree $K_v(\Phi_r+\Phi_v+1)+\Phi_v\le(K_v+1)(2\Phi_v+1)-1$, since $\Phi_r\le\Phi_v$, and the image of the first factor has a smaller degree, since $K_r\le K_v$ and $\Phi_r\le\Phi_v$.
	The two images are therefore multiplied with $\bigO(\mathsf{M}((K_v+1)(2\Phi_v+1)))=\bigO(\mathsf{M}((K_v+1)(\Phi_v+1)))$ operations by~\citet[Corollary 8.28]{vzgathen2013modern}, since $(K_v+1)(2\Phi_v+1)\le2(K_v+1)(\Phi_v+1)$ and $\mathsf{M}(cs)=\bigO(\mathsf{M}(s))$ for any constant $c$, as in the analysis of our bottom-up algorithm.
	We then recover the entries of the matrix of $\Lambda_w$ by reading $(K_w+1)(\Phi_w+1)\le(K_v+1)(\Phi_v+1)$ coefficients of $C(\gfvar, u)$, that is at most $\bigO((K_v+1)(\Phi_v+1))$ additional operations, which concludes the proof.
\end{proof}

The matrix $\Lambda_r$ is obtained by \Cref{lemma:middleProduct} simply exchanging the roles of $w$ and $r$, hence processing a node $v$ in the top-down pass requires $\bigO(\mathsf{M}((K_v+1)(\Phi_v+1)))$ operations.
By performing two top-down passes, one for each value of $i\in\{1,2\}$, we thus obtain the contributions $S_h(j)$ of the iteration $h$ for every $j\in[2,\totLabs]$, and summing over $h$ yields the Shapley value of the incorrect classifiers as from \Cref{theo:ShapleyMCCorrect}.
The polynomials $P_r$ used in step ($ii$) above, are those computed by the bottom-up phase, and the matrix of $\Lambda_v$ has the size of the matrix of $P_v$. 
Therefore, the cost of processing a level of the tree in top-down is bounded as in the bottom-up pass.
That is, each of the two top-down passes over $i\in\{1,2\}$ for an iteration $h\in[1,H^\star]$ of the algorithm requires $\bigO(\mathsf{M}(K(h)(\Phi(h-1)+1))\log\totLabs)$ operations.

We conclude with the resulting time complexity.
Recall from the proof of \Cref{lemma:telescoping} that $\Phi(\mathrm{col}(x)-1)<x$ for every $x\in[1,N^\star]$, and that $\mathrm{col}(x)$ takes the value $h$ for exactly $K(h)$ indices $x$, thus
\[
\sum_{h=1}^{H^\star}K(h)\big(\Phi(h-1)+1\big)=\sum_{x=1}^{N^\star}\big(\Phi(\mathrm{col}(x)-1)+1\big)\le\sum_{x=1}^{N^\star}x\le (N^\star)^2 \enspace .
\]
Now note that the time complexity of performing an iteration $h$ is
\[
\bigO\big(\mathsf{M}(K(h)(\Phi(h-1)+1))\log\totLabs\big)
=\bigO\big(K(h)(\Phi(h-1)+1)\log(\totLabs \widetilde{N})\log\totLabs\big) \enspace ,
\]
by using the bounds $\mathsf{M}(s)=\bigO(s\log s)$, $K(h)\le\totLabs-1$ and $\Phi(h-1)\le\widetilde{N}$, which hold for every $h$ and every vote profile.
Thus combining the above
\begin{align*}
\bigO\left(\log(\totLabs \widetilde{N})\log\totLabs\sum_{h=1}^{H^\star}K(h)\big(\Phi(h-1)+1\big)\right)
&=\bigO\big((N^\star)^2\log(\totLabs \widetilde{N})\log\totLabs\big) \\
&=\bigO\big((\widetilde{N})^2\log(\totLabs \widetilde{N})\log\totLabs\big) \enspace .
\end{align*}
Finally, accounting for the processing of each unique vote profile in $\unique$, we obtain a time complexity bounded by
\[
\bigO(|\unique|(\widetilde{N})^2\log(\totLabs \widetilde{N})\log\totLabs + |\dataSet|\cardM) \enspace .
\]

Consider the bound of \Cref{theo:refinedRuntimeMCSHAP}, both bounds contain the term $\bigO(|\dataSet|\cardM)$ for the evaluation of the classifiers over $\dataSet$.
On the remaining term, the above bound is smaller by a factor $\Theta((\widetilde{N}+\totLabs)/(\log(\totLabs \widetilde{N})\log\totLabs))$, showing the improvement in the asymptotic time complexity.

We conclude by noting that such an improvement comes at the price of larger memory complexity.
In fact, for a fixed iteration $h\in[1,H^\star]$, the entire tree has memory complexity bounded by $\bigO(K(h)(\Phi(h-1)+1)\log\totLabs)$.
That is, we need to retain all matrices at all $\bigO(\log\totLabs)$ levels required for the top-down pass. 
Note that at each level such matrices have at most $\bigO(K(h)(\Phi(h-1)+1))$ entries (see \Cref{eq:levelSize}).
The matrices of the functionals $\Lambda_v$ do not add any asymptotic memory complexity.
The resulting memory complexity is therefore bounded by $\bigO(|\unique|\totLabs+\totLabs \widetilde{N}\log\totLabs)$.
\fi

\end{document}